\documentclass[twoside,11pt]{article}

\usepackage[nohyperref]{jmlr2e}

\usepackage{comment}

\usepackage{float}					
\usepackage{tikz}                   
\usepackage{tikzscale}              
\usetikzlibrary{calc}               
\usepackage[caption=false,font=footnotesize]{subfig}	
\usepackage{pgfplots}               
\pgfplotsset{compat=newest}         
\pgfplotsset{plot coordinates/math parser=false}    
\newlength\fwidth                   
\definecolor{mycolor1}{rgb}{0.00000,0.44700,0.74100}    
\definecolor{mycolor2}{rgb}{0.85000,0.32500,0.09800}    
\definecolor{mycolor3}{rgb}{0.92900,0.69400,0.12500}    
\definecolor{mycolor4}{rgb}{0.49400,0.18400,0.55600}    

\usepackage{amsmath,amsfonts}
\usepackage{mathtools}				
\DeclareMathOperator{\diag}{diag}	
\DeclareMathOperator{\relu}{ReLU}	
\DeclareMathOperator{\rank}{rank}	
\DeclareMathOperator{\tr}{tr}	    

\DeclareMathOperator*{\argmax}{arg\,max}	
\DeclareMathOperator*{\argmin}{arg\,min}	

\newtheorem{assumption}[theorem]{Assumption}

\usepackage{booktabs}
\usepackage{multirow}

\usepackage{siunitx}

\usepackage{algorithm}
\usepackage{algorithmic}

\usepackage{hyperref}		
\hypersetup{hidelinks}      
\usepackage[all]{hypcap}	

\newcommand\hl[1]{{\color{black}{#1}}}	

\usepackage{lastpage}
\jmlrheading{26}{2025}{1-\pageref{LastPage}}{1/21; Revised 9/24}{5/25}{21-0068}{Brendon G.\ Anderson, Ziye Ma, Jingqi Li, and Somayeh Sojoudi}

\ShortHeadings{Towards Optimal Branching for Neural Network Robustness Certification}{Anderson, Ma, Li, and Sojoudi}
\firstpageno{1}

\begin{document}

\title{Towards Optimal Branching of Linear and Semidefinite Relaxations for Neural Network Robustness Certification}

\author{\name Brendon G.\ Anderson\footnotemark[1] \email bganderson@berkeley.edu\\
       \name Ziye Ma\footnotemark[2] \email ziyema@berkeley.edu\\
       \name Jingqi Li\footnotemark[2] \email jingqili@berkeley.edu\\
       \name Somayeh Sojoudi\footnotemark[2] \email sojoudi@berkeley.edu\\
       \addr \footnotemark[1]~Department of Mechanical Engineering, University of California, Berkeley\\
       \addr \footnotemark[2]~Department of Electrical Engineering and Computer Sciences, University of California, Berkeley}

\editor{Miguel Carreira-Perpinan}

\maketitle

\begin{abstract}
In this paper, we study certifying the robustness of ReLU neural networks against adversarial input perturbations. To diminish the relaxation error suffered by the popular linear programming (LP) and semidefinite programming (SDP) certification methods, we take a branch-and-bound approach to propose partitioning the input uncertainty set and solving the relaxations on each part separately. We show that this approach reduces relaxation error, and that the error is eliminated entirely upon performing an LP relaxation with a partition intelligently designed to exploit the nature of the ReLU activations. To scale this approach to large networks, we consider using a coarser partition whereby the number of parts in the partition is reduced. We prove that computing such a coarse partition that directly minimizes the LP relaxation error is NP-hard. By instead minimizing the worst-case LP relaxation error, we develop a closed-form branching scheme \hl{in the single-hidden layer case}. We extend the analysis to the SDP, where the feasible set geometry is exploited to design a branching scheme that minimizes the worst-case SDP relaxation error. Experiments on MNIST, CIFAR-10, and Wisconsin breast cancer diagnosis classifiers demonstrate significant increases in the percentages of test samples certified. By independently increasing the input size and the number of layers, we empirically illustrate under which regimes the branched LP and branched SDP are best applied. \hl{Finally, we extend our LP branching method into a multi-layer branching heuristic, which attains comparable performance to prior state-of-the-art heuristics on large-scale, deep neural network certification benchmarks.}
\end{abstract}




\section{Introduction}

It is evident that the data used in real-world engineering systems has uncertainty. This uncertainty takes many forms, including corruptions, random measurement noise, and adversarial attacks \citep{franceschi2018robustness,balunovic2019certifying,jin2019boundary}. Recently, researchers have shown that the performance of neural networks can be highly sensitive to these uncertainties in the input data \citep{szegedy2013intriguing,fawzi2016robustness,su2019one}. Clearly, safety-critical systems, such as autonomous vehicles \citep{bojarski2016end,wu2017squeezedet} and the power grid \citep{kong2017short,muralitharan2018neural,pan2019deepopf}, must be robust against fluctuations and uncertainty in the inputs to their decision-making algorithms. This fact has motivated a massive influx of research studying methods to certify the robustness of neural networks against uncertainty in their input data \citep{wong2018provable,raghunathan2018semidefinite,xiang2018reachability,weng2018towards,zhang2018efficient,royo2019fast,fazlyab2020safety,anderson2020tightened,anderson2022data,ma2020strengthened,jin2021power}.

The two primary settings taken in the robustness certification literature consider either random input uncertainty or adversarial uncertainty. In the former, the neural network input is assumed to be random and follow a known probability distribution. For instance, the works \citet{weng2019proven,anderson2022data} derive high-probability guarantees that this randomness causes no misclassification or unsafe output. In the adversarial setting, the input is assumed to be unknown but contained in a prescribed input uncertainty set. The goal here is to certify that all possible inputs from the uncertainty set are mapped to outputs that the network operator deems as safe \citep{wong2018provable,royo2019fast}. In this paper, we take the latter, worst-case perspective. We remark that this approach is more generally applicable than the techniques for random inputs. Indeed, certifying that a network is robust against adversarial inputs immediately also certifies it is robust against random inputs distributed over the same uncertainty set; if the worst-case input cannot cause a failure, then neither will randomly selected ones.

The problem of adversarial robustness certification amounts to proving that all possible outputs in the output set, i.e., the image of the input uncertainty set under the mapping of the network, are contained in the safe set. However, this output set is generally nonconvex, even when the input set is convex. Consequently, the certification decision problem has been shown to be NP-complete, and the optimization-based formulation for solving it is an NP-hard, nonconvex optimization problem \citep{katz2017reluplex,weng2018towards}. To make the problem more tractable, researchers have proposed various ways to over-approximate the nonconvex output set with a convex one. Performing the certification over the convex surrogate reduces the problem to an easy-to-solve convex relaxation, and if the relaxation issues a robustness certificate for the outer approximation, it also issues a certificate for the true set of outputs. Such convex relaxation-based certification algorithms are sound, but generally incomplete; the network may be robust even if the algorithm is unable to verify it. Thus, a line of works has been developed in order to increase the percentage of test inputs that convex relaxation-based methods are able to certify \citep{wong2018provable,raghunathan2018semidefinite,fazlyab2020safety,tjandraatmadja2020convex,muller2021prima,batten2021efficient,wang2021beta}.

Perhaps the simplest and most popular outer approximation technique is based on a linear programming (LP) relaxation of the $\relu$ activation function \citep{ehlers2017formal,wong2018provable}. However, this method has been shown to yield relatively loose outer approximations, making it possible for the approximating set to contain unsafe outputs, even if the true output set is entirely safe. If this occurs, the convex relaxation fails to issue a certificate of robustness, despite the fact the network is indeed robust. A semidefinite programming (SDP) relaxation was proposed in \citet{raghunathan2018semidefinite} and shown to yield tighter outer approximations when compared to the LP method. Other methods, such as the quadratically-constrained semidefinite program \citep{fazlyab2020safety}, use sophisticated relaxations to tighten the approximation, but these SDP-based methods inevitably gain accuracy at the expense of enlarging the computational cost, and are still susceptible to relaxation error. The work \citet{tjandraatmadja2020convex} introduces a joint $\relu$ LP relaxation that is tighter than the per-neuron approach, but may require an exponential number of linear inequalities and is weaker overall in comparison to the per-neuron relaxation when embedded into a branch-and-bound scheme \citep{wang2021beta}. We follow the certification literature's emphasis on $\relu$ networks and focus our attention in this paper on such models, which, aside from their certifiable structure, are extremely popular for their non-vanishing gradient property and their quick training times \citep{xiang2018reachability}. \hl{We remark that branch-and-bound methods that are derived for $\relu$ models may still be applied to models with non-$\relu$ activations, either by restricting the branching scheme to neurons with $\relu$ activations, or by naively applying the branching scheme to the non-$\relu$ neurons as if they had $\relu$ activations.}

\subsection{Branch-and-Bound Certification}

As eluded to above, instead of resorting to exorbitantly costly relaxation techniques to lower the approximation error, e.g., using heuristic variants of Lasserre's hierarchy \citep{chen2020semialgebraic}, state-of-the-art convex relaxation-based verifiers have turned to branch-and-bound methods, where the nonconvex optimization domain is recursively partitioned and outer-approximated by smaller convex sets. In fact, the verifier that won the 2021, 2022, and 2023 VNN certification competitions is $\alpha,\beta$-CROWN \citep{bak2021second,muller2022third,brix2023fourth,wang2021beta}, which is a branch-and-bound method that uses linear programming relaxations in the bounding steps. $\alpha,\beta$-CROWN uses $\alpha$-CROWN to generate linear upper and lower bounds on the neural network output, together with the heuristic per-neuron branching methods BaBSR \citep{bunel2020branch} and filtered smart branching (FSB) \citep{de2021improved}.

The success of $\alpha,\beta$-CROWN is interesting, since the BaBSR and FSB branching strategies are designed in the context of Ehler's per-neuron ``triangle'' convex relaxations \citep{ehlers2017formal}, which is a different and, when activation bounds are adequately chosen, tighter bounding approach than $\alpha$-CROWN's linear bounding of the output, albeit computationally more expensive. FSB remains the state-of-the-art branching heuristic, outperforming BaBSR \citep{wang2021beta,de2021improved}, which in turn beat a state-of-the-art mixed-integer approach and all prior $\relu$ neuron splitting methods: Reluplex, Planet, and Neurify \citep{bunel2020branch,katz2017reluplex,ehlers2017formal,wang2018efficient}. FSB works by assigning scores to each neuron that are computed by quickly estimating the improvement in the optimization objective upon splitting that neuron. In this paper, we propose two branch-and-bound certification methods, one utilizing the triangle relaxation of \cite{ehlers2017formal} as a bounding method and the other utilizing the SDP relaxation of \citet{raghunathan2018semidefinite}, and for these bounding methods we derive novel branching schemes that minimize the worst-case relaxation error.

\hl{It is worth briefly mentioning here that branch-and-bound certification techniques can make use of parallel computing when solving independent subproblems arising from the feasible set partitioning \citep{lu2019neural,wu2020parallelization,xu2020fast}. This improves the scalability and efficiency of neural network certification. Our proposed branch-and-bound methods are naturally able to leverage such parallelizations.}

\subsection{Contributions}

In this paper, we fully exploit the piecewise linear structure of ReLU networks in order to tighten convex relaxations for robustness certification. A condensed summary of our main contributions is as follows:
\begin{enumerate}
    \item For the LP relaxation, we show that a methodically designed finite partition of the input set allows one to attain zero relaxation error. We then use the structure of this motivating partition to derive a closed-form branching scheme that minimizes worst-case relaxation error \hl{for single-hidden layer networks}. This novel branching scheme makes theoretically principled strides towards optimal LP branching. \hl{We show how to extend our branching scheme into a multi-layer branching heuristic, applicable to deep neural networks.}
    \item We prove that computing the branching neuron that minimizes the actual LP relaxation error is NP-hard, in turn theoretically justifying our approach of minimizing the worst-case relaxation error.
    \item For the SDP relaxation, we develop a geometrically interpretable measure for how far the solution is from being rank-1. We then show \hl{in the single-hidden layer case} that, when branching along a given coordinate, this rank-1 gap is minimized by a uniform grid. Finally, we derive the branching coordinate that minimizes the worst-case relaxation error.
    \item We embed our branching schemes into a branch-and-bound framework, and empirically validate the effectiveness of our approaches on the MNIST, CIFAR-10, and Wisconsin breast cancer diagnosis benchmarks. We experimentally show on these datasets that our LP branching scheme outperforms the state-of-the-art filtered smart branching, the branching heuristic used in $\alpha,\beta$-CROWN to win the 2021, 2022, and 2023 VNN certification competitions. Our SDP branching scheme is found to yield even higher certified percentages. \hl{We also run experiments on large-scale deep neural network certification problems involving multiple layers with complex structures, and we find that our method yields performance comparable to that of these state-of-the-art methods. Overall, our experiments suggest that our theoretically principled single-hidden layer methods outperform the current methods on moderately-sized problems, and our heuristic extensions to deep networks achieve comparable performance on large-scale problems.}
\end{enumerate}
This paper is a major extension of the conference paper \citet{anderson2020tightened}. Notably, the NP-hardness result, all of the SDP results, and all of the experiments of this paper (the second, third, and fourth major contributions listed above) are new additions. The contributions of this paper culminate into two theoretically principled branching-based convex relaxations for ReLU robustness certification. Our experiments demonstrate that the proposed approaches are effective and efficient, and from these results we develop general guidance for which network size and structure regimes the LP, branched LP, SDP, and branched SDP are best applied.

\subsection{Outline}
This paper is organized as follows. Section \ref{sec: problem_statement} introduces the ReLU robustness certification problem as well as its basic LP and SDP relaxations. In Section \ref{sec: partitioned_lp_relaxation}, we study the effect of partitioning the input uncertainty set for the LP relaxation. After developing formal guarantees for its effectiveness, we develop a branching strategy that optimally reduces the worst-case relaxation error for networks with a single hidden layer. \hl{We then extend our theoretically principled LP branching scheme to a multi-layer branching heuristic.} Similarly, in Section \ref{sec: partitioned_sdp_relaxation} we study the partitioned SDP relaxation, \hl{bound its error, and propose a coordinate-wise branching scheme to minimize the bound}. Section \ref{sec: simulation_results} demonstrates the developed methods on numerical examples and studies the effectiveness of branching on the LP and SDP as the network grows in width and depth. Finally, we conclude in Section \ref{sec: conclusions}.

\subsection{Notations}
\label{sec: notations}
\hl{
For the reader's convenience, we list our commonly used symbols in Table~\ref{tab: symbols}, some of which will be more rigorously defined in the following sections.

For an index set $\mathcal{I}\subseteq\{1,2,\dots,n\}$, the symbol $\mathbf{1}_\mathcal{I}$ is used to denote the $n$-vector with $(\mathbf{1}_\mathcal{I})_i=1$ for $i\in\mathcal{I}$ and $(\mathbf{1}_\mathcal{I})_i=0$ for $i\in\mathcal{I}^c=\{1,2,\dots,n\}\setminus\mathcal{I}$.
For a function $f\colon\mathbb{R}^n\to\mathbb{R}^m$ and the point $x\in\mathbb{R}^n$, we denote the $i^{\text{th}}$ element of the $m$-vector $f(x)$ by $f_i(x)$.
For a matrix $X\in\mathbb{R}^{m\times n}$, we use two indices to denote an element of $X$ and one index to denote a column of $X$, unless otherwise stated. In particular, the $(i,j)$ element and the $i^\text{th}$ column of $X$ are denoted by $X_{ij}$ and $X_i$ respectively, or, when necessary for clarity, by $(X)_{ij}$ and $(X)_i$ respectively.
Furthermore, for a function $f\colon \mathbb{R}\to\mathbb{R}$, we define $f(X)$ to be an $m\times n$ matrix whose $(i,j)$ element is equal to $f(X_{ij})$ for all $i\in\{1,2,\dots,m\}$ and all $j\in\{1,2,\dots,n\}$. If $Z$ is a square $n\times n$ matrix, we use the notation $\diag(Z)$ to mean the $n$-vector $(Z_{11},Z_{22},\dots,Z_{nn})$ and $\tr(Z)$ to mean the trace of $Z$. When $Z\in\mathbb{S}^n$, we write $Z\succeq 0$ to mean that $Z$ is positive semidefinite.

We use $\mathbb{I}$ to denote the indicator function, i.e., $\mathbb{I}(A)=1$ if event $A$ holds and $\mathbb{I}(A)=0$ if event $A$ does not hold.
For a set $\mathcal{T}\subseteq \mathbb{R}$, we respectively denote its infimum and supremum by $\inf\mathcal{T}$ and $\sup\mathcal{T}$. For a function $f\colon\mathbb{R}^n \to \mathbb{R}$ and a set $\mathcal{X}\subseteq\mathbb{R}^n$, we write $\inf_{x\in\mathcal{X}}f(x)$ to mean $\inf \{f(x) : x\in\mathcal{X}\}$ and similarly for suprema.
Finally, we assume that all infima and suprema throughout the paper are attained.
}

\section{Problem Statement}
\label{sec: problem_statement}

\subsection{Description of the Network and Uncertainty}

In this paper, we consider a pre-trained $K$-layer ReLU neural network defined by
\begin{equation}
\begin{aligned}
x^{[0]} ={}& x, \\
\hat{z}^{[k]} ={}& W^{[k-1]}x^{[k-1]}+b^{[k-1]}, \\
x^{[k]} ={}& \relu(\hat{z}^{[k]}), \\
z ={}& x^{[K]},
\end{aligned} \label{eq: network_description}
\end{equation}
for all $k\in\{1,2,\dots,K\}$. Here, the neural network input is $x\in\mathbb{R}^{n_x}$, the output is $z\in\mathbb{R}^{n_z}$, and the $k^\text{th}$ layer's preactivation is $\hat{z}^{[k]}\in\mathbb{R}^{n_k}$. The parameters $W^{[k]}\in\mathbb{R}^{n_{k+1}\times n_k}$ and $b^{[k]}\in\mathbb{R}^{n_{k+1}}$ are the weight matrix and bias vector applied to the $k^{\text{th}}$ layer's activation $x^{[k]}\in\mathbb{R}^{n_k}$, respectively. Without loss of generality,\footnote{Note that $Wx+b = \begin{bmatrix} W & b \end{bmatrix} \begin{bmatrix} x \\ 1 \end{bmatrix} \eqqcolon \tilde{W}\tilde{x}$, so the bias term $b$ can be eliminated by appending a fixed value of $1$ at the end of the activation $x$. This parameterization can be used throughout this paper by using matching lower and upper activation bounds of $1$ in the last coordinate of each layer.} we assume that the bias terms are accounted for in the activations $x^{[k]}$, thereby setting $b^{[k]}=0$ for all layers $k$. We remark that, since we make no assumptions on the linear transformations defined by $W^{[k]}$, our work also applies to networks with linear convolutional layers. We let the function $f\colon \mathbb{R}^{n_x}\to\mathbb{R}^{n_z}$ denote the map $x\mapsto z$ defined by \eqref{eq: network_description}. In the case that $f$ is a classification network, the output dimension $n_z$ equals the number of classes. The problem at hand is to certify the robustness in the neural network output $z$ when the input $x$ is subject to uncertainty.

\begin{table}[tbh]
    \hl{
	\centering
	\caption{List of commonly used symbols.}
    \vspace*{0.5\baselineskip}
	\begin{tabular}{l l}
	\toprule
	Symbol & Meaning \\
	\midrule %
	$\mathbb{R}^n$ & Set of $n$-vectors with real elements \\
    $\mathbb{S}^n$ & Set of $n\times n$ matrices with real elements \\
    $e_i$ & $i^\textup{th}$ standard basis vector of $\mathbb{R}^n$ \\
    $\mathbf{1}_n$ & $n$-vector of all ones \\
    $X \le Y$ & The array $X$ is element-wise less than or equal to the array $Y$ \\
    $X \odot Y$ & Element-wise (Hadamard) product between arrays $X$ and $Y$ \\
    $X \oslash Y$ & Element-wise (Hadamard) division of array $X$ by array $Y$ \\
    $|\mathcal{S}|$ & Cardinality of the set $\mathcal{S}$ \\
    $\|\cdot\|_*$ & Dual norm of the norm $\|\cdot\|$ \\
    $d_{\|\cdot\|}(\mathcal{X})$ & Diameter of $\mathcal{X}\subseteq \mathbb{R}^n$ with respect to norm $\|\cdot\|$; $d_{\|\cdot\|}(\mathcal{X}) = \sup_{x,x'\in \mathcal{X}} \|x - x'\|$ \\
    $f$ & Neural network \\
    $\relu$ & Rectified linear unit; $\relu(\cdot) = \max\{0,\cdot\}$ \\
    $x$, $z$ & Neural network input and output, respectively \\
    $x^{[k]}$, $\hat{z}^{[k]}$ & Neural network activations and preactivations at layer $k$, respectively \\
    $W^{[k]}$ & Neural network weights at layer $k$ \\
    $l^{[k]}$, $u^{[k]}$ & Neural network preactivation lower and upper bounds at layer $k$, respectively \\
    $\mathcal{X}$ & Neural network input uncertainty set \\
    $f(\mathcal{X})$ & Neural network output uncertainty set given the input uncertainty set $\mathcal{X}$ \\
    $\phi(\mathcal{X})$ & Optimal value of certification problem over the input uncertainty set $\mathcal{X}$ \\
	\bottomrule
	\end{tabular}
    }
	\label{tab: symbols}
\end{table}

To model the input uncertainty, we assume that the network inputs are unknown but contained in a compact set $\mathcal{X}\subseteq\mathbb{R}^{n_x}$, called the \emph{input uncertainty set}. We assume that the set $\mathcal{X}$ is a convex polytope, so that the condition $x\in\mathcal{X}$ can be written as a finite number of affine inequalities and equalities. In the adversarial robustness literature, the input uncertainty set is commonly modeled as $\mathcal{X} = \{x\in\mathbb{R}^{n_x} : \|x - \bar{x}\|_\infty \le \epsilon\}$, where $\bar{x}\in\mathbb{R}^{n_x}$ is a nominal input to the network and $\epsilon>0$ is an uncertainty radius \citep{wong2018provable,raghunathan2018semidefinite}. We remark that our generalized polytopic model for $\mathcal{X}$ includes this common case. The primary theme of this paper revolves around partitioning the input uncertainty set in order to strengthen convex robustness certification methods. Let us briefly recall the definition of a partition.

\begin{definition}[Partition]
	\label{def: partition}
	The collection $\{\mathcal{X}^{(j)} \subseteq \mathcal{X} : j\in\{1,2,\dots,p\}\}$ is said to be a \emph{partition} of the input uncertainty set $\mathcal{X}$ if $\mathcal{X}=\cup_{j=1}^p\mathcal{X}^{(j)}$ and $\mathcal{X}^{(j)}\cap \mathcal{X}^{(k)} = \emptyset$ for all $j \ne k$. The set $\mathcal{X}^{(j)}$ is called the $j^{\text{th}}$ \emph{input part}.
\end{definition}

We generally use the terminology \textit{branching} to refer to partitioning with two parts, which is typically applied in a recursive fashion.

Now, in order to describe the robustness of the network \eqref{eq: network_description}, we need a notion of permissible outputs. For this, we consider a \emph{safe set}, denoted $\mathcal{S}\subseteq\mathbb{R}^{n_z}$. If an output $z\in\mathbb{R}^{n_z}$ is an element of $\mathcal{S}$, we say that $z$ is \emph{safe}. It is common in the robustness certification literature to consider (possibly unbounded) polyhedral safe sets. We take this same perspective, and assume that $\mathcal{S}$ is defined by
\begin{equation*}
    \mathcal{S} = \{z\in\mathbb{R}^{n_z} : C z \le d\},
\end{equation*}
where $C\in\mathbb{R}^{n_\mathcal{S}\times n_z}$ and $d\in\mathbb{R}^{n_\mathcal{S}}$ are given. Note that, again, our model naturally includes the classification setting. In particular, suppose that $i^* \in \argmax_{i\in\{1,2,\dots,n_z\}} f_i(\bar{x})$ is the true class of the nominal input $\bar{x}$. Then, define the $i^{\text{th}}$ row of $C\in\mathbb{R}^{n_z\times n_z}$ to be
\begin{equation*}
    c_i^\top = e_i^\top - e_{i^*}^\top
\end{equation*}
and $d$ to be the zero vector. Then it is immediately clear that an input $x$ (which can be thought of as a perturbed version of $\bar{x}$) is safe if and only if $f_i(x) \le f_{i^*}(x)$ for all $i$, i.e., the network classifies $x$ into class $i^*$. From this classification perspective, the safe set represents the set of outputs without any misclassification (with respect to the nominal input $\bar{x}$). From here on, we consider $f$, $\mathcal{X}$, and $\mathcal{S}$ in their general forms---we do not restrict ourselves to the classification setting.

\subsection{The Robustness Certification Problem}

The fundamental goal of the robustness certification problem is to prove that all inputs in the input uncertainty set map to safe outputs, i.e., that $f(x)\in\mathcal{S}$ for all $x\in\mathcal{X}$. If this certificate holds, the network is said to be \emph{certifiably robust}, which of course is a property that holds with respect to a particular input uncertainty set. The robustness condition can also be written as $f(\mathcal{X})\subseteq\mathcal{S}$, or equivalently
\begin{equation*}
\sup_{x\in\mathcal{X}}\left(c_i^\top f(x) - d_i \right) \le 0 ~ \text{for all $i\in\{1,2,\dots,n_{\mathcal{S}}\}$},
\end{equation*}
where $c_i^\top$ is the $i^{\text{th}}$ row of $C$. Under this formulation, the certification procedure amounts to solving $n_\mathcal{S}$ optimization problems. The methods we develop in this paper can be applied to each of these optimizations individually, and therefore in the sequel we focus on a single optimization problem by assuming that $n_\mathcal{S}=1$, namely $\sup_{x\in\mathcal{X}} c^\top f(x)$. We also assume without loss of generality that $d=0$. If $d$ were nonzero, one may absorb $d$ into the cost vector $c$ and modify the network model by appending a fixed value of $1$ at the end of the output vector $f(x)$. Under these formulations, we write the robustness certification problem as
\begin{equation}
\begin{aligned}
\hl{\phi^*}(\mathcal{X}) ={}& \sup\{c^\top z : z=f(x), ~ x\in\mathcal{X} \},
\end{aligned} \label{eq: robustness_certification_problem}
\end{equation}
and recall that we seek to certify that $\hl{\phi^*}(\mathcal{X})\le 0$.

Since the function $f$ is in general nonconvex, the nonlinear equality constraint $z=f(x)$ makes the optimization \eqref{eq: robustness_certification_problem} a nonconvex problem and the set $f(\mathcal{X})$ a nonconvex set. Furthermore, since the intermediate activations and preactivations of the network generally have a large dimension in practice, the problem \eqref{eq: robustness_certification_problem} is typically a high-dimensional problem. Therefore, computing an exact robustness certificate, as formulated in \eqref{eq: robustness_certification_problem}, is computationally intractable. Instead of directly maximizing $c^\top z$ over the nonconvex output set $f(\mathcal{X})$, one can overcome these hurdles by optimizing over a convex outer approximation $\hat{f}(\mathcal{X})\supseteq f(\mathcal{X})$. Indeed, this new problem is a convex relaxation of the original problem, so it is generally easier and more efficient to solve. If the convex outer approximation is shown to be safe, i.e., $\hat{f}(\mathcal{X})\subseteq\mathcal{S}$, then the true nonconvex set $f(\mathcal{X})$ is also known to be safe, implying that the robustness of the network is certified. Figure \ref{fig: outer_approximation} illustrates this idea.

\begin{figure}[tbh]
    \centering
    \includegraphics[width=0.65\linewidth]{figures/outer_approximation}
    \caption{The set $\hat{f}(\mathcal{X})$ is a convex outer approximation of the nonconvex set $f(\mathcal{X})$. If the outer approximation is safe, i.e., $\hat{f}(\mathcal{X})\subseteq \mathcal{S}$, then so is $f(\mathcal{X})$.}
    \label{fig: outer_approximation}
\end{figure}

A fundamental drawback to the convex relaxation approach to robustness certification is as follows: if the outer approximation $\hat{f}(\mathcal{X})$ is too loose, then it may intersect with the unsafe region of the output space, meaning $\hat{f}(\mathcal{X})\nsubseteq \mathcal{S}$, even in the case where the true output set is safe. In this scenario, the convex relaxation fails to issue a certificate of robustness, since the optimal value to the convex relaxation is positive, which incorrectly suggests the presence of unsafe network inputs within $\mathcal{X}$. This situation is illustrated in Figure \ref{fig: convex_relaxations}.

\begin{figure}[tbh]
    \centering
    \includegraphics[width=0.5\linewidth]{figures/convex_relaxations}
    \caption{This scenario shows that if the convex outer approximation $\hat{f}(\mathcal{X})$ is too large, meaning the relaxation is too loose, then the convex approach fails to issue a certificate of robustness.}
    \label{fig: convex_relaxations}
\end{figure}

The fundamental goal of this paper is to develop convex relaxation methods for robustness certification such that the outer approximation tightly fits $f(\mathcal{X})$, in effect granting strong certificates while maintaining the computational and theoretical advantages of convex optimization. We focus on two popular types of convex relaxations, namely, LP \citep{wong2018provable} and SDP \citep{raghunathan2018semidefinite} relaxations. It has been shown that the SDP relaxation for ReLU networks is tighter than the LP relaxation, with the tradeoff of being computationally more demanding. Our general approach for both the LP and the SDP relaxations is based on partitioning the input uncertainty set and solving a convex relaxation on each input part separately. Throughout our theoretical development and experiments, we will show that this approach presents a valid, efficient, and effective way to tighten the relaxations of both LP and SDP certifications, and \hl{we derive principled branching strategies for both methods that minimize upper bounds on the relaxation errors}. We now turn to mathematically formulating these relaxations.

\subsection{LP Relaxation of the Network Constraints}

We now introduce the LP relaxation. First, we remark that since $\mathcal{X}$ is bounded, the preactivations at each layer are bounded as well. That is, for every $k\in\{1,2,\dots,K\}$, there exist preactivation bounds $l^{[k]},u^{[k]}\in\mathbb{R}^{n_k}$ such that $l^{[k]}\le \hat{z}^{[k]}\le u^{[k]}$, where $\hat{z}^{[k]}$ is the $k^\text{th}$ layer's preactivation in \eqref{eq: network_description}, for all $x\in\mathcal{X}$. We assume without loss of generality that $l^{[k]}\le 0 \le u^{[k]}$ for all $k$, \hl{since, if $l^{[k]}_i>0$ for some $i$, the preactivation bound $l^{[k]}_i \le \hat{z}^{[k]}_i$ can be replaced by $0 \le \hat{z}^{[k]}_i$, and similarly for the case where $u^{[k]}_i < 0$}. Although there exist various methods in the literature for efficiently approximating these preactivation bounds, we consider the bounds to be tight, i.e., $\hat{z}^{[k]} = l^{[k]}$ for some $x\in\mathcal{X}$, and similarly for the upper bound $u^{[k]}$. \citet{tjeng2019evaluating} provides efficient methods for computing these bounds. Now, following the approach initially introduced in \citet{wong2018provable}, we can relax the $k^\text{th}$ ReLU constraint in \eqref{eq: network_description} to its convex upper envelope between the preactivation bounds. This is graphically shown in Figure \ref{fig: relaxed_relu_constraint}.

\begin{figure}[tbh]
    \centering
    \includegraphics[width=0.5\linewidth]{figures/relaxed_relu_constraint}
    \caption{Relaxed ReLU constraint set $\mathcal{N}_\textup{LP}^{[k]}$ at a single neuron $i$ in layer $k$ of the network.}
    \label{fig: relaxed_relu_constraint}
\end{figure}

We call the convex upper envelope associated with layer $k$ the \emph{relaxed ReLU constraint set}, and its mathematical definition is given by three linear inequalities:
\begin{equation}
\begin{aligned}
\mathcal{N}_\textup{LP}^{[k]} &= \{(x^{[k-1]},x^{[k]})\in\mathbb{R}^{n_{k-1}}\times\mathbb{R}^{n_k} : x^{[k]} \le u^{[k]} \odot (\hat{z}^{[k]}-l^{[k]}) \oslash (u^{[k]}-l^{[k]}), \\
&\qquad x^{[k]} \ge 0, ~ x^{[k]} \ge \hat{z}^{[k]}, ~ \hat{z}^{[k]}=W^{[k-1]}x^{[k-1]}\}.
\end{aligned}
\label{eq: relu_constraint_set}
\end{equation}
Next, we define the \emph{relaxed network constraint set} to be
\begin{equation}
\mathcal{N}_\textup{LP} = \{(x,z)\in\mathbb{R}^{n_x}\times \mathbb{R}^{n_z} : (x,x^{[1]})\in\mathcal{N}_\textup{LP}^{[1]}, ~ (x^{[1]},x^{[2]})\in\mathcal{N}_\textup{LP}^{[2]},\dots,(x^{[K-1]},z)\in\mathcal{N}_\textup{LP}^{[K]}\}. \label{eq: network_constraint_set}
\end{equation}
In essence, $\mathcal{N}_\textup{LP}$ is the set of all input-output pairs of the network that satisfy the relaxed ReLU constraints at every layer. Note that, since the bounds $l^{[k]}$ and $u^{[k]}$ are determined by the input uncertainty set $\mathcal{X}$, the set $\mathcal{N}_\textup{LP}^{[k]}$ is also determined by $\mathcal{X}$ for all layers $k$.

\begin{remark}
For networks with one hidden layer (i.e., $K=1$), the single relaxed ReLU constraint set coincides with the relaxed network constraint set: $\mathcal{N}_\textup{LP}^{[1]}=\mathcal{N}_\textup{LP}$. Therefore, for $K=1$ we drop the $k$-notation and simply write $z$, $\hat{z}$, $x$, $W$, $l$, $u$, and $\mathcal{N}_\textup{LP}$.
\end{remark}

Finally, we introduce the LP relaxation of \eqref{eq: robustness_certification_problem}:
\begin{equation}
\begin{aligned}
\hl{\hat{\phi}_{\textup{LP}}^*}(\mathcal{X}) ={}& \sup\{c^\top z : (x,z)\in\mathcal{N}_\textup{LP}, ~ x\in\mathcal{X} \}.
\end{aligned} \label{eq: lp_relaxation}
\end{equation}
Notice that, if $x\in\mathcal{X}$ and $z=f(x)$, then $(x,z)\in\mathcal{N}_\textup{LP}$ by the definition of $\mathcal{N}_\textup{LP}$. Furthermore, since $\mathcal{X}$ is a bounded convex polytope and $\mathcal{N}_\textup{LP}$ is defined by a system of linear constraints, we confirm that \eqref{eq: lp_relaxation} is a linear program. Therefore, \eqref{eq: lp_relaxation} is indeed an LP relaxation of \eqref{eq: robustness_certification_problem}, so it holds that
\begin{equation}
\begin{aligned}
\hl{\phi^*}(\mathcal{X})\le{}& \hl{\hat{\phi}_{\textup{LP}}^*}(\mathcal{X}).
\end{aligned} \label{eq: relaxation_bound}
\end{equation}
This analytically shows what Figures \ref{fig: outer_approximation} and \ref{fig: convex_relaxations} illustrate: the condition that $\hl{\hat{\phi}_{\textup{LP}}^*}(\mathcal{X}) \le 0$ is sufficient to conclude that the network is certifiably robust, but if $\hl{\hat{\phi}_{\textup{LP}}^*}(\mathcal{X})>0$, the relaxation fails to certify whether or not the network is robust, since it may still hold that $\hl{\phi^*}(\mathcal{X})\le 0$.

\subsection{SDP Relaxation of the Network Constraints}

An alternative convex relaxation of the robustness certification problem can be formulated as an SDP. This method was first introduced in \citet{raghunathan2018semidefinite}. Here, we will introduce the SDP relaxation for a network with a single hidden layer for notational convenience. The extension to multi-layer networks is straightforward and presented in \citet{raghunathan2018semidefinite}. In this formulation, the optimization variable $(x,z)\in\mathbb{R}^{n_x+n_z}$ is lifted to a symmetric matrix
\begin{equation*}
    P = \begin{bmatrix}
    1 \\ x \\ z
    \end{bmatrix}\begin{bmatrix}
    1 & x^\top & z^\top
    \end{bmatrix} \in \mathbb{S}^{n_x+n_z+1}.
\end{equation*}
We use the following block-indexing for $P$:
\begin{equation*}
	P = \begin{bmatrix}
	P_1 & P_x^\top & P_z^\top \\
	P_x & P_{xx} & P_{xz} \\
	P_z & P_{zx} & P_{zz}
	\end{bmatrix},
\end{equation*}
where $P_1\in\mathbb{R}$, $P_{x}\in\mathbb{R}^{n_x}$, $P_z \in\mathbb{R}^{n_z}$, $P_{xx}\in\mathbb{S}^{n_x}$, $P_{zz}\in\mathbb{S}^{n_z}$, $P_{xz}\in\mathbb{R}^{n_x\times n_z}$, and $P_{zx} = P_{xz}^\top$. This lifting procedure results in the optimization problem
\begin{equation*}
    \begin{aligned}
    & \underset{P\in\mathbb{S}^{n_x+n_z+1}}{\text{maximize}} && c^\top P_z \\
    & \text{subject to} && P_z \ge 0, \\
    &&& P_z \ge WP_x, \\
    &&& \diag(P_{xx}) \le (l+u)\odot P_x - l\odot u, \\
    &&& \diag(P_{zz}) = \diag(WP_{xz}), \\
    &&& P_1 = 1, \\
    &&& P \succeq 0, \\
    &&& \rank(P) = 1.
    \end{aligned}
\end{equation*}
Here, there are no preactivation bounds, unlike the LP relaxation. The vectors $l,u\in\mathbb{R}^{n_x}$ are lower and upper bounds on the input, which are determined by the input uncertainty set. For example, if $\mathcal{X} = \{x\in\mathbb{R}^{n_x}:\|x-\bar{x}\|_\infty \le \epsilon\}$, then $l=\bar{x}-\epsilon\mathbf{1}_{n_x}$ and $u = \bar{x} + \epsilon\mathbf{1}_{n_x}$.

We remark that the above problem is equivalent to the original robustness certification problem; no relaxation has been made yet. The only nonconvex constraint in this formulation is the rank-1 constraint on $P$. Dropping this rank constraint, we obtain the SDP relaxed network constraint set:
\begin{equation}
\begin{aligned}
\mathcal{N}_\textup{SDP} &= \{P\in\mathbb{S}^{n_x+n_z+1}: P_z \ge 0, ~ P_z \ge WP_x, ~ \diag(P_{zz}) = \diag(WP_{xz}), \\
&\qquad \diag(P_{xx}) \le (l+u) \odot P_x - l \odot u, ~ P_1 = 1, ~ P \succeq 0 \}.
\end{aligned}
\label{eq: network_constraint_set_sdp}
\end{equation}
Using this relaxed network constraint set as the feasible set for the optimization, we arrive at the SDP relaxation
\begin{equation}
\hl{\hat{\phi}^*_\textup{SDP}}(\mathcal{X}) = \sup\{c^\top P_z : P \in\mathcal{N}_\textup{SDP}, ~ P_x \in\mathcal{X} \}. \label{eq: sdp_relaxation}
\end{equation}

It is clear that by dropping the rank constraint, we have enlarged the feasible set, so again we obtain a viable relaxation of the original problem \eqref{eq: robustness_certification_problem}: $\hl{\phi^*}(\mathcal{X})\le \hl{\hat{\phi}^*_\textup{SDP}}(\mathcal{X})$. In the case the solution $P^*$ to \eqref{eq: sdp_relaxation} is rank-1, we can factorize it as
\begin{equation*}
    P^* = \begin{bmatrix}
    1 \\ x^* \\ z^*
    \end{bmatrix}
    \begin{bmatrix}
    1 & x^{*\top} & z^{*\top}
    \end{bmatrix}
\end{equation*}
and conclude that $(x^*,z^*)$ solves the original nonconvex problem. However, it is generally the case that the SDP solution will be of higher rank, leading to relaxation error and the possibility of a void robustness certificate, similar to the LP relaxation. We now turn to building upon the LP and SDP convex relaxations via input partitioning in order to tighten their relaxations.

\section{Partitioned LP Relaxation}
\label{sec: partitioned_lp_relaxation}

\subsection{Properties of Partitioned Relaxation}
\label{sec: properties_of_partitioned_relaxation}

In this section, we investigate the properties and effectiveness of partitioning the input uncertainty set when solving the LP relaxation for robustness certification. We start by validating the approach, namely, by showing that solving the LP relaxation separately on each input part maintains a theoretically guaranteed upper bound on the optimal value of the unrelaxed problem \eqref{eq: robustness_certification_problem}. Afterwards, the approach is proven to yield a tighter upper bound than solving the LP relaxation without partitioning.

\subsubsection{Partitioning Gives Valid Relaxation}

\begin{proposition}[Partitioned relaxation bound]
	\label{prop: partitioned_relaxation_bound}
	Let $\{\mathcal{X}^{(j)} \subseteq \mathcal{X} : j\in\{1,2,\dots,p\}\}$ be a partition of $\mathcal{X}$. Then, it holds that
	\begin{equation}
	\hl{\phi^*}(\mathcal{X}) \le \max_{j\in\{1,2,\dots,p\}}\hl{\hat{\phi}_{\textup{LP}}^*}(\mathcal{X}^{(j)}). \label{eq: partitioned_relaxation_bound}
	\end{equation}
\end{proposition}
\begin{proof}
See Appendix \ref{sec: proof_of_partitioned_relaxation_bound}.
\end{proof}
Despite the fact that Proposition \ref{prop: partitioned_relaxation_bound} asserts an intuitively expected bound, we remark the importance for its formal statement and proof. In particular, the inequality \eqref{eq: partitioned_relaxation_bound} serves as the fundamental reason for why the partitioned LP relaxation can be used to certify that all inputs in $\mathcal{X}$ map to safe outputs in the safe set $\mathcal{S}$. Knowing that the partitioning approach is valid for robustness certification, we move on to studying the effectiveness of partitioning.

\subsubsection{Tightening of the Relaxation}

We now show that the bound \eqref{eq: relaxation_bound} can always be tightened by partitioning the input uncertainty set. The result is given for networks with one hidden layer for simplicity. However, the conclusion naturally generalizes to multi-layer ReLU networks.

\begin{proposition}[Improving the LP relaxation bound]
	\label{prop: improving_the_lp_relaxation_bound}
	Consider a feedforward ReLU neural network with one hidden layer. Let $\{\mathcal{X}^{(j)} \subseteq \mathcal{X} : j\in\{1,2,\dots,p\}\}$ be a partition of $\mathcal{X}$. For the $j^{\text{th}}$ input part $\mathcal{X}^{(j)}$, denote the corresponding preactivation bounds by $l^{(j)}$ and $u^{(j)}$, where $l \le l^{(j)} \le Wx \le u^{(j)} \le u$ for all $x\in\mathcal{X}^{(j)}$. Then, it holds that
	\begin{equation}
	\max_{j\in\{1,2,\dots,p\}} \hl{\hat{\phi}_{\textup{LP}}^*}(\mathcal{X}^{(j)}) \le \hl{\hat{\phi}_{\textup{LP}}^*}(\mathcal{X}). \label{eq: improving_the_lp_relaxation_bound}
	\end{equation}
\end{proposition}

\begin{proof}
    See Appendix \ref{sec: proof_of_partition_improvement_lp}.
\end{proof}
\hspace*{\fill}

Propositions~\ref{prop: partitioned_relaxation_bound} and \ref{prop: improving_the_lp_relaxation_bound} together show that $\hl{\phi^*}(\mathcal{X})\le \max_{j\in\{1,2,\dots,p\}} \hl{\hat{\phi}_{\textup{LP}}^*}(\mathcal{X}^{(j)}) \le \hl{\hat{\phi}_{\textup{LP}}^*}(\mathcal{X})$, i.e., that the partitioned LP relaxation is theoretically guaranteed to perform at least as \hl{well} as the unpartitioned LP when solving the robustness certification problem. The improvement in the partitioned LP relaxation is captured by the difference
\begin{equation*}
\hl{\hat{\phi}_{\textup{LP}}^*}(\mathcal{X}) - \max_{j\in\{1,2,\dots,p\}}\hl{\hat{\phi}_{\textup{LP}}^*}(\mathcal{X}^{(j)}),
\end{equation*}
which is always nonnegative. We remark that it is possible for the improvement to be null in the sense that $\max_{j\in\{1,2,\dots,p\}}\hl{\hat{\phi}_{\textup{LP}}^*}(\mathcal{X}^{(j)}) = \hl{\hat{\phi}_{\textup{LP}}^*}(\mathcal{X})$. This may occur when the partition used is poorly chosen. An example of such a poor choice may be if one were to partition along a direction in the input space that, informally speaking, corresponds to directions near-orthogonal to the cost vector $c$ in the output space. In this case, one would expect all improvements to be nullified, and for the partitioned relaxation to give the same optimal value as the unpartitioned relaxation. Consequently, the following important question arises: \emph{what constitutes a good partition so that the improvement $\hl{\hat{\phi}_{\textup{LP}}^*}(\mathcal{X}) - \max_{j\in\{1,2,\dots,p\}}\hl{\hat{\phi}_{\textup{LP}}^*}(\mathcal{X}^{(j)})$ is strictly greater than zero and maximal?} We address this question in Sections \ref{sec: motivating_partition_lp} and \ref{sec: partitioning_scheme_lp}.

\subsection{Motivating Partition}
\label{sec: motivating_partition_lp}

In this section, we begin to answer our earlier inquiry, namely, how to choose a partition in order to maximize the improvement $\hl{\hat{\phi}_{\textup{LP}}^*}(\mathcal{X}) - \max_{j\in\{1,2,\dots,p\}}\hl{\hat{\phi}_{\textup{LP}}^*}(\mathcal{X}^{(j)})$ in the partitioned LP relaxation. Recall that this is equivalent to minimizing the relaxation error relative to the original unrelaxed problem, since $\hl{\phi^*}(\mathcal{X})\le \max_{j\in\{1,2,\dots,p\}} \hl{\hat{\phi}_{\textup{LP}}^*}(\mathcal{X}^{(j)}) \le \hl{\hat{\phi}_{\textup{LP}}^*}(\mathcal{X})$. To this end, we construct a partition with finitely many parts, based on the parameters of the network, which is shown to exactly recover the optimal value of the original unrelaxed problem \eqref{eq: robustness_certification_problem}. For simplicity, we present the result for a single hidden layer, but the basic idea of partitioning at the ``kinks'' of the ReLUs in order to collapse the ReLU upper envelope onto the ReLU curve and eliminate relaxation error can be generalized to multi-layer settings. At this point, let us remark that in Proposition \ref{prop: motivating_partition} below, we use a slight difference in notation for the partition. Namely, we use the set of all $n_z$-vectors with binary elements, $\mathcal{J}\coloneqq \{0,1\}^{n_z} = \{0,1\}\times\{0,1\}\times \cdots \times \{0,1\}$, to index the parts of the partition. Under this setting, the partition is composed of $p=2^{n_z}$ parts, so that $\mathcal{X}^{(j)}$ is the part of the partition corresponding to the binary vector $j$, which is an element of the index set $\mathcal{J}$. This temporary change in notation is chosen to simplify the proof of Proposition \ref{prop: motivating_partition}.

\begin{proposition}[Motivating partition]
	\label{prop: motivating_partition}
	Consider a feedforward ReLU neural network with one hidden layer and denote the $i^{\text{th}}$ row of $W$ by $w_i^\top\in\mathbb{R}^{1\times n_x}$ for all $i\in\{1,2,\dots,n_z\}$. Define $\mathcal{J} = \{0,1\}^{n_z}$ and take the partition of $\mathcal{X}$ to be indexed by $\mathcal{J}$, meaning that $\{\mathcal{X}^{(j)} \subseteq \mathcal{X} : j\in\mathcal{J}\}$, where for a given $j\in\mathcal{J}$ we define
	\begin{equation}
	\mathcal{X}^{(j)} = \{x\in\mathcal{X} : w_i^\top x \ge 0 ~ \emph{for all $i$ such that $j_i=1$}, ~ w_i^\top x<0 ~ \emph{for all $i$ such that $j_i=0$}\}. \label{eq: motivating_partition_input_part}
	\end{equation}
	Then, the partitioned relaxation is exact, i.e.,
	\begin{equation}
	\hl{\phi^*}(\mathcal{X}) = \max_{j\in\mathcal{J}} \hl{\hat{\phi}_{\textup{LP}}^*}(\mathcal{X}^{(j)}). \label{eq: motivating_partition_exact}
	\end{equation}
\end{proposition}

\begin{proof}
	We first show that $\{\mathcal{X}^{(j)} \subseteq \mathcal{X} : j\in\mathcal{J}\}$ is a valid partition. Since $\mathcal{X}^{(j)} \subseteq \mathcal{X}$ for all $j\in\mathcal{J}$, the relation $\cup_{j\in\mathcal{J}}\mathcal{X}^{(j)} \subseteq \mathcal{X}$ is satisfied. Now, suppose that $x\in\mathcal{X}$. Then, for all $i\in\{1,2,\dots,n_z\}$, either $w_i^\top x \ge 0$ or $w_i^\top x<0$ holds. Define $j\in\{0,1\}^{n_z}$ as follows:
	\begin{equation*}
	j_i = \begin{aligned}
	\begin{cases}
	1 & \text{if $w_i^\top x \ge 0$}, \\
	0 & \text{if $w_i^\top x < 0$},
	\end{cases}
	\end{aligned}
	\end{equation*}
	for all $i\in\{1,2,\dots,n_z\}$. Then, by the definition of $\mathcal{X}^{(j)}$ in \eqref{eq: motivating_partition_input_part}, it holds that $x\in\mathcal{X}^{(j)}$. Therefore, the relation $x\in\mathcal{X}$ implies that $x\in\mathcal{X}^{(j)}$ for some $j\in\{0,1\}^{n_z} = \mathcal{J}$. Hence, $\mathcal{X}\subseteq \cup_{j\in\mathcal{J}}\mathcal{X}^{(j)}$, and therefore $\cup_{j\in\mathcal{J}}\mathcal{X}^{(j)} = \mathcal{X}$.
	
	We now show that $\mathcal{X}^{(j)}\cap\mathcal{X}^{(k)}=\emptyset$ for all $j\ne k$. Let $j,k\in\mathcal{J}$ with the property that $j\ne k$. Then, there exists $i\in\{1,2,\dots,n_z\}$ such that $j_i\ne k_i$. Let $x\in\mathcal{X}^{(j)}$. In the case that $w_i^\top x\ge 0$, it holds that $j_i = 1$ and therefore $k_i = 0$. Hence, for all $y\in\mathcal{X}^{(k)}$, it holds that $w_i^\top y < 0$, and therefore $x\notin\mathcal{X}^{(k)}$. An analogous reasoning shows that $x\notin\mathcal{X}^{(k)}$ when $w_i^\top x < 0$. Therefore, one concludes that $x\in\mathcal{X}^{(j)}$ and $j\ne k$ implies that $x\notin\mathcal{X}^{(k)}$, i.e., that $\mathcal{X}^{(j)}\cap\mathcal{X}^{(k)} = \emptyset$. Hence, $\{\mathcal{X}^{(j)} \subseteq \mathcal{X} : j\in\mathcal{J}\}$ is a valid partition.
	
	We now prove \eqref{eq: motivating_partition_exact}. Let $j\in\mathcal{J}$. Since $w_i^\top x \ge 0$ for all $i$ such that $j_i=1$, the preactivation lower bound becomes $l_i^{(j)} = 0$ for all such $i$. On the other hand, since $w_i^\top x < 0$ for all $i$ such that $j_i=0$, the preactivation upper bound becomes $u_i^{(j)} = 0$ for all such $i$. Therefore, the relaxed network constraint set \eqref{eq: network_constraint_set} for the $j^{\text{th}}$ input part reduces to
	\begin{equation*}
	\begin{aligned}
	\mathcal{N}_\textup{LP}^{(j)} &= \{(x,z)\in\mathbb{R}^{n_x}\times\mathbb{R}^{n_z} :  z_i = 0 ~ \text{for all $i$ such that $j_i = 0$}, \\
	&\qquad z_i = w_i^\top x = (Wx)_i ~ \text{for all $i$ such that $j_i = 1$}\}.
	\end{aligned}
	\end{equation*}
	That is, the relaxed ReLU constraint envelope collapses to the exact ReLU constraint through the prior knowledge of each preactivation coordinate's sign.
%
%
	Therefore, we find that for all $x\in\mathcal{X}^{(j)}$ it holds that $(x,z)\in\mathcal{N}_\textup{LP}^{(j)}$ if and only if $z=\relu(Wx)$. Hence, the LP over the $j^{\text{th}}$ input part yields that
	\begin{align*}
	\hl{\hat{\phi}_{\textup{LP}}^*}(\mathcal{X}^{(j)}) &= \sup\{c^\top z : (x,z)\in\mathcal{N}_\textup{LP}^{(j)}, ~ x\in\mathcal{X}^{(j)}\} \\
    &= \sup\{c^\top z : z = \relu(Wx), ~ x\in\mathcal{X}^{(j)}\} \\
	&\le \sup\{c^\top z : z = \relu(Wx), ~ x\in\mathcal{X}\} \\
    &= \hl{\phi^*}(\mathcal{X}).
	\end{align*}
	Since $j$ was chosen arbitrarily, it holds that
	\begin{equation*}
	\max_{j\in\mathcal{J}}\hl{\hat{\phi}_{\textup{LP}}^*}(\mathcal{X}^{(j)}) \le \hl{\phi^*}(\mathcal{X}).
	\end{equation*}
	Since $\hl{\phi^*}(\mathcal{X})\le \max_{j\in\mathcal{J}}\hl{\hat{\phi}_{\textup{LP}}^*}(\mathcal{X}^{(j)})$ by the relaxation bound \eqref{eq: partitioned_relaxation_bound}, the equality \eqref{eq: motivating_partition_exact} holds, as desired.
\end{proof}
\hspace*{\fill}

Although the partition proposed in Proposition \ref{prop: motivating_partition} completely eliminates relaxation error of the LP, using it in practice may be computationally intractable, as it requires solving $2^{n_z}$ separate linear programs. Despite this limitation, the result provides two major theoretical implications. First, our input partitioning approach is fundamentally shown to be a simple, yet very powerful method, as the robustness certification problem can be solved exactly via a finite number of linear program subproblems. Second, the partition proposed in Proposition \ref{prop: motivating_partition} shows us the structure of an optimal partition, namely that the parts of the partition are defined by the half-spaces generated by the rows of $W$ (see Figure \ref{fig: weight-based_partition}). This result paves the way to develop a tractable branching scheme that incorporates the reduction in relaxation error endowed by the structure of this motivating partition. In the next section, we explore this idea further, and seek to answer the following question: \emph{if we only partition along a single row of the weight matrix, which one is the best to choose?}

\begin{figure}[tbh]
    \centering
    \includegraphics[width=0.9\linewidth]{figures/weight-based_partition}
    \caption{Partitioning based on row $w_i^\top$ of the weight matrix. This partition results in an exact ReLU constraint in coordinate $i$ over the two resulting input parts $\mathcal{X}_i^{(1)} = \{x\in\mathcal{X} : w_i^\top x \ge 0\}$ and $\mathcal{X}_i^{(2)}=\mathcal{X}\setminus\mathcal{X}_i^{(1)}$.}
    \label{fig: weight-based_partition}
\end{figure}

\subsection{LP Branching Scheme}
\label{sec: partitioning_scheme_lp}

In this section, we propose an explicit, computationally tractable, and effective LP branching scheme. The branching scheme is developed based on analyses for a single hidden layer. However, the resulting branching scheme is still applicable to multi-layer networks, and indeed will be shown to remain effective on two-layer networks in the experiments of Section \ref{sec: simulation_results}. The development of the partition boils down to two main ideas. First, we restrict our attention to two-part partitions defined by rows of the weight matrix $W$, specifically, $\mathcal{X}^{(1)}_i = \{x\in\mathcal{X}:w_i^\top x \ge 0\}$ and $\mathcal{X}_i^{(2)} = \mathcal{X}\setminus\mathcal{X}^{(1)}_i$, as motivated in the previous section. Second, we seek which index $i\in\{1,2,\dots,n_z\}$ gives the best partition, in the sense that the relaxation error of the resulting partitioned LP is minimized. As will be shown in Section \ref{sec: optimal_partitioning_is_np-hard}, this second aspect is NP-hard to discern in general. Therefore, to find the optimal row to partition along, we instead seek to minimize a more tractable upper bound on the relaxation error. \hl{Since the actual relaxation error cannot exceed this upper bound even in the worst case, we refer to such a bound-minimizing partition as ``worst-case optimal.''}

\subsubsection{Worst-Case Relaxation Bound}
We begin by bounding the relaxation error below.

\begin{theorem}[Worst-case relaxation bound]
	\label{thm: generalized_worst-case_relaxation_bound}
	Consider a feedforward ReLU neural network with one hidden layer, with the input uncertainty set $\mathcal{X}$ and preactivation bounds $l,u\in\mathbb{R}^{n_z}$. Consider also the relaxation error $\Delta \hl{\phi^*}(\mathcal{X}) \coloneqq \hl{\hat{\phi}_{\textup{LP}}^*}(\mathcal{X})-\hl{\phi^*}(\mathcal{X})$. Let $(\tilde{x}^*,\tilde{z}^*)$ and $(x^*,z^*)$ be optimal solutions for the relaxation $\hl{\hat{\phi}_{\textup{LP}}^*}(\mathcal{X})$ and the unrelaxed problem $\hl{\phi^*}(\mathcal{X})$, respectively. Given an arbitrary norm $\|\cdot\|$ on $\mathbb{R}^{n_x}$, it holds that
	\begin{equation}
	\begin{aligned}
	\Delta \hl{\phi^*}(\mathcal{X}) &\le \sum_{i=1}^{n_z} \bigg( \relu(c_i)\frac{u_i }{u_i-l_i}(\min\{\|w_i\|_* d_{\|\cdot\|}(\mathcal{X}),u_i\} - l_i) \\
	&\qquad + \relu(-c_i)\min\{\|w_i\|_* d_{\|\cdot\|}(\mathcal{X}),u_i\}\bigg),
	\end{aligned}
	 \label{eq: generalized_worst-case_relaxation_bound}
	\end{equation}
	where $\|\cdot\|_*$ is the dual norm of $\|\cdot\|$.
\end{theorem}

\begin{proof}
	First, recall that $\mathcal{X}\subseteq\mathbb{R}^{n_x}$ is assumed to be compact, and is therefore bounded, and hence $d_{\|\cdot\|}(\mathcal{X}) < \infty$. The definitions of $(\tilde{x}^*,\tilde{z}^*)$ and $(x^*,z^*)$ give that
	\begin{equation}
	\Delta \hl{\phi^*}(\mathcal{X}) = \sum_{i=1}^{n_z}c_i(\tilde{z}^*_i-z^*_i) \le \sum_{i=1}^{n_z}\Delta \hl{\phi^*_i}, \label{eq: generalized_worst-case_relaxation_bound_intermediate_1}
	\end{equation}
	where
	\begin{equation*}
	\begin{aligned}
	\Delta \hl{\phi^*_i} &= \sup\bigg\{c_i(\tilde{z}_i-z_i) : z_i = \relu(w_i^\top x), ~ \tilde{z}_i\ge 0, ~ \tilde{z}_i \ge w_i^\top \tilde{x}, \\
	&\qquad \tilde{z}_i \le \frac{u_i}{u_i-l_i}(w_i^\top \tilde{x} - l_i), ~ x,\tilde{x}\in\mathcal{X} \bigg\}
	\end{aligned}
	\end{equation*}
	for all $i\in\{1,2,\dots,n_z\}$. Note that
	\begin{equation*}
	\begin{aligned}
	\Delta \phi_i^* &= \sup\bigg\{c_i(\tilde{z}_i-z_i) : z_i = \relu(\hat{z}_i), ~ \tilde{z}_i\ge 0, ~ \tilde{z}_i \ge \hat{\tilde{z}}_i, ~ \tilde{z}_i \le \frac{u_i}{u_i-l_i}(\hat{\tilde{z}}_i - l_i), \\
	&\qquad \hat{z} = W x, ~ \hat{\tilde{z}} = W\tilde{x}, ~ x,\tilde{x}\in\mathcal{X} \bigg\}.
	\end{aligned}
	\end{equation*}
	If $x,\tilde{x}\in\mathcal{X}$ and $\hat{z},\hat{\tilde{z}}$ satisfy $\hat{z}=Wx$, $\hat{\tilde{z}}=W\tilde{x}$, then they satisfy $l\le \hat{z},\hat{\tilde{z}}\le u$ and $|\hat{\tilde{z}}_k-\hat{z}_k| = |w_k^\top(\tilde{x}-x)| \le \|w_k\|_*\|\tilde{x}-x\|\le \|w_k\|_* d_{\|\cdot\|}(\mathcal{X})$ for all $k\in\{1,2,\dots,n_z\}$ by the Cauchy-Schwarz inequality for dual norms. Therefore,
	\begin{align*}
	\Delta \phi_i^* \le{}& \sup\bigg\{c_i(\tilde{z}_i-z_i) : z_i = \relu(\hat{z}_i), ~ \tilde{z}_i\ge 0, ~ \tilde{z}_i \ge \hat{\tilde{z}}_i, ~ \tilde{z}_i \le \frac{u_i}{u_i-l_i}(\hat{\tilde{z}}_i - l_i), \\
	&{} l\le \hat{z},\hat{\tilde{z}}\le u, ~ |\hat{\tilde{z}}_k-\hat{z}_k|\le \|w_k\|_* d_{\|\cdot\|}(\mathcal{X}) ~ \text{for all $k\in\{1,2,\dots,n_z\}$}, ~ \hat{z},\hat{\tilde{z}}\in\mathbb{R}^{n_z} \bigg\} \\
	={}& \sup\bigg\{c_i(\tilde{z}_i-z_i) : z_i = \relu(\hat{z}_i), ~ \tilde{z}_i\ge 0, ~ \tilde{z}_i \ge \hat{\tilde{z}}_i, ~ \tilde{z}_i \le \frac{u_i}{u_i-l_i}(\hat{\tilde{z}}_i - l_i), \\
	&{} l_i\le \hat{z}_i,\hat{\tilde{z}}_i\le u_i, ~ |\hat{\tilde{z}}_i-\hat{z}_i|\le \|w_i\|_* d_{\|\cdot\|}(\mathcal{X}), ~ \hat{z}_i,\hat{\tilde{z}}_i\in\mathbb{R} \bigg\}.
	\end{align*}
	For $c_i\ge 0$, the above inequality yields that
	\begin{equation*}
	\begin{aligned}
	\Delta \phi_i^* &\le c_i\sup\bigg\{\tilde{z}_i-z_i : z_i = \relu(\hat{z}_i), ~ \tilde{z}_i\ge 0, ~ \tilde{z}_i \ge \hat{\tilde{z}}_i, ~ \tilde{z}_i \le \frac{u_i}{u_i-l_i}(\hat{\tilde{z}}_i - l_i), \\
	&\qquad l_i\le \hat{z}_i,\hat{\tilde{z}}_i\le u_i, ~ |\hat{\tilde{z}}_i-\hat{z}_i|\le \|w_i\|_* d_{\|\cdot\|}(\mathcal{X}), ~ \hat{z}_i,\hat{\tilde{z}}_i\in\mathbb{R} \bigg\}.
	\end{aligned}
	\end{equation*}
	The optimal solution to the above supremum is readily found by comparing the line $\tilde{z}_i = \frac{u_i}{u_i-l_i}(\hat{\tilde{z}}_i-l_i)$ to the function $z_i=\relu(\hat{z}_i)$ over $\hat{\tilde{z}}_i,\hat{z}_i\in[l_i,u_i]$. In particular, the maximum distance between $\tilde{z}_i$ and $z_i$ on the above feasible set occurs when $z_i=\hat{z}_i = 0$, $\hat{\tilde{z}}_i =  \|w_i\|_* d_{\|\cdot\|}(\mathcal{X})$, and $\tilde{z}_i = \frac{u_i}{u_i-l_i}( \|w_i\|_* d_{\|\cdot\|}(\mathcal{X}) - l_i)$. Therefore, we find that
	\begin{equation}
	\Delta \phi_i^* \le c_i \frac{u_i}{u_i-l_i}( \|w_i\|_* d_{\|\cdot\|}(\mathcal{X}) - l_i), \label{eq: generalized_worst-case_relaxation_bound_intermediate_2}
	\end{equation}
	for all $i\in\{1,2,\dots,n_z\}$ such that $c_i\ge 0$. We also note the trivial bound that $\tilde{z}_i - z_i \le u_i$ on the feasible set of the above supremum, so that
	\begin{equation}
	\Delta \phi_i^* \le c_i u_i = c_i \frac{u_i}{u_i-l_i}(u_i-l_i). \label{eq: generalized_worst-case_relaxation_bound_intermediate_3}
	\end{equation}
	The inequalities \eqref{eq: generalized_worst-case_relaxation_bound_intermediate_2} and \eqref{eq: generalized_worst-case_relaxation_bound_intermediate_3} together imply that
	\begin{equation}
	\Delta \phi_i^* \le c_i \frac{u_i}{u_i - l_i} (\min\{\|w_i\|_* d_{\|\cdot\|}(\mathcal{X}),u_i\}-l_i) \label{eq: generalized_worst-case_relaxation_bound_intermediate_4}
	\end{equation}
	for all $i\in\{1,2,\dots,n_z\}$ such that $c_i\ge 0$.
	
	On the other hand, for all $i\in\{1,2,\dots,n_z\}$ such that $c_i<0$, we have that
	\begin{equation*}
	\begin{aligned}
	\Delta \phi_i^* &\le c_i\inf\bigg\{\tilde{z}_i-z_i : z_i = \relu(\hat{z}_i), ~ \tilde{z}_i\ge 0, ~ \tilde{z}_i \ge \hat{\tilde{z}}_i, ~ \tilde{z}_i \le \frac{u_i}{u_i-l_i}(\hat{\tilde{z}}_i - l_i), \\
	&\qquad l_i\le \hat{z}_i,\hat{\tilde{z}}_i\le u_i, ~ |\hat{\tilde{z}}_i-\hat{z}_i|\le \|w_i\|_* d_{\|\cdot\|}(\mathcal{X}), ~ \hat{z}_i,\hat{\tilde{z}}_i\in\mathbb{R} \bigg\}.
	\end{aligned}
	\end{equation*}
	The optimal solution to the above infimum is readily found by comparing the line $\tilde{z}_i=0$ to the function $z_i=\relu(\hat{z}_i)$ over $\hat{\tilde{z}}_i,\hat{z}_i\in[l_i,u_i]$. In particular, the minimum value of $\tilde{z}_i-z_i$ on the above feasible set occurs when $\tilde{z}_i=\hat{\tilde{z}}_i=0$ and $z_i=\hat{z}_i=\|w_i\|_* d_{\|\cdot\|}(\mathcal{X})$. Therefore, we find that
	\begin{equation}
		\Delta \phi_i^* \le -c_i\|w_i\|_* d_{\|\cdot\|}(\mathcal{X}), \label{eq: generalized_worst-case_relaxation_bound_intermediate_5}
	\end{equation}
	for all $i\in\{1,2,\dots,n_z\}$ such that $c_i<0$. We also note the trivial bound that $\tilde{z}_i-z_i\ge -u_i$ on the feasible set of the above infimum, so that
	\begin{equation}
		\Delta \phi_i^* \le -c_i u_i. \label{eq: generalized_worst-case_relaxation_bound_intermediate_6}
	\end{equation}
	The inequalities \eqref{eq: generalized_worst-case_relaxation_bound_intermediate_5} and \eqref{eq: generalized_worst-case_relaxation_bound_intermediate_6} together imply that
	\begin{equation}
		\Delta \phi_i^* \le -c_i\min\{\|w_i\|_* d_{\|\cdot\|}(\mathcal{X}),u_i\} \label{eq: generalized_worst-case_relaxation_bound_intermediate_7}
	\end{equation}
	for all $i\in\{1,2,\dots,n_z\}$ such that $c_i<0$. Substituting \eqref{eq: generalized_worst-case_relaxation_bound_intermediate_4} and \eqref{eq: generalized_worst-case_relaxation_bound_intermediate_7} into \eqref{eq: generalized_worst-case_relaxation_bound_intermediate_1} gives the desired bound \eqref{eq: generalized_worst-case_relaxation_bound}.
\end{proof}
\hspace*{\fill}

The value $\Delta \hl{\phi^*_i}$ in the proof of Theorem \ref{thm: generalized_worst-case_relaxation_bound} can be interpreted as the worst-case relaxation error in coordinate $i$. From this perspective, Theorem \ref{thm: generalized_worst-case_relaxation_bound} gives an upper bound on the worst-case relaxation error of the overall network. Notice that, since $\min\{\|w_i\|_* d_{\|\cdot\|}(\mathcal{X}),u_i\} \le u_i$, the bound \eqref{eq: generalized_worst-case_relaxation_bound} immediately gives rise to the simple bound that
\begin{equation*}
    \Delta \hl{\phi^*}(\mathcal{X}) \le \sum_{i=1}^{n_z}\left(\relu(c_i) + \relu(-c_i)\right) u_i = \sum_{i=1}^{n_z}|c_i| u_i,
\end{equation*}
the right-hand side of which equals the relaxation error incurred when, at every neuron, the activations of the relaxation solution and the original nonconvex solution are at opposite corners of the relaxed $\relu$ constraint set, i.e., the convex upper envelope illustrated in Figure \ref{fig: relaxed_relu_constraint}. On the other hand, when $d_{\|\cdot\|}(\mathcal{X})$ is small, i.e., the input uncertainty set is small, one would expect the number of ``kinks'' in the graph of $f$ over $\mathcal{X}$ to be small, and as a consequence the relaxation error to decrease. This intuition is captured by the bound \eqref{eq: generalized_worst-case_relaxation_bound}, since, in this case, we find that
\begin{align*}
    \Delta \hl{\phi^*}(\mathcal{X}) &\le \sum_{i=1}^{n_z} \left( \relu(c_i) \frac{u_i}{u_i-l_i}(\|w_i\|_* d_{\|\cdot\|}(\mathcal{X}) - l_i) + \relu(-c_i) \|w_i\|_* d_{\|\cdot\|}(\mathcal{X}) \right) \\
    &\approx -\sum_{i=1}^{n_z} \relu(c_i) \frac{u_i l_i}{u_i-l_i} \\
    &\le \sum_{i=1}^{n_z} |c_i| u_i.
\end{align*}

To continue our development of a branching scheme that is optimal with respect to the worst-case relaxation error, we use \eqref{eq: generalized_worst-case_relaxation_bound} to bound the relaxation error of the partitioned LP in terms of the row $w_i^\top$ that is used to define the partition. This bound is given in the following lemma.

\begin{lemma}[Two-part bound]
    \label{lem: two-part_partition_bound}
    Let $i\in\{1,2,\dots,n_z\}$ and consider a two-part partition of $\mathcal{X}$ given by $\{\mathcal{X}_i^{(1)},\mathcal{X}_i^{(2)}\}$, where $\mathcal{X}_i^{(1)} = \{x\in\mathcal{X} : w_i^\top x \ge 0\}$ and $\mathcal{X}_i^{(2)} = \mathcal{X}\setminus \mathcal{X}_i^{(1)}$. Consider also the partitioned relaxation error $\Delta \hl{\phi^*}(\{\mathcal{X}_i^{(1)},\mathcal{X}_i^{(2)}\}) \coloneqq \max_{j\in\{1,2\}} \hl{\hat{\phi}_{\textup{LP}}^*}(\mathcal{X}_i^{(j)}) - \hl{\phi^*}(\mathcal{X})$. It holds that
    \begin{equation}
    \begin{aligned}
        \Delta \hl{\phi^*}(\{\mathcal{X}_i^{(1)},\mathcal{X}_i^{(2)}\}) & \le |c_i| \min\{\|w_i\|_* d_{\|\cdot\|}(\mathcal{X}),u_i\} \\
        &\qquad \begin{aligned}
        & +\sum_{\substack{k=1 \\ k\ne i}}^{n_z} \bigg(\relu(c_k) \frac{u_k}{u_k-l_k} (\min\{\|w_k\|_*d_{\|\cdot\|}(\mathcal{X}),u_k\}-l_k) \\
        &\qquad + \relu(-c_k)\min\{\|w_k\|_*d_{\|\cdot\|}(\mathcal{X}),u_k\} \bigg).
        \end{aligned}
        \end{aligned}
        \label{eq: two-part_partition_bound}
    \end{equation}
\end{lemma}

\begin{proof}
    Consider the relaxation solved over the first input part, $\mathcal{X}_i^{(1)}$, and denote by $l^{(1)},u^{(1)}\in\mathbb{R}^{n_z}$ the corresponding preactivation bounds. Since $w_i^\top x \ge 0$ on this input part, the preactivation bounds for the first subproblem $\hl{\hat{\phi}_{\textup{LP}}^*}(\mathcal{X}_i^{(1)})$ can be taken as
    \begin{equation*}
    	l^{(1)}=(l_1,l_2,\dots,l_{i-1},0,l_{i+1},\dots,l_{n_z})
    \end{equation*}
    and $u^{(1)}=u$. Thus, from \eqref{eq: generalized_worst-case_relaxation_bound} and the fact that $d_{\|\cdot\|}(\mathcal{X}_i^{(j)}) \le d_{\|\cdot\|}(\mathcal{X})$ for $j\in\{1,2\}$, it follows that
    \begin{equation}
    \begin{aligned}
    \Delta \hl{\phi^*}(\mathcal{X}_i^{(1)}) &\le |c_i| \min\{\|w_i\|_* d_{\|\cdot\|}(\mathcal{X}),u_i\} \\
    &\qquad \begin{aligned}
    & +\sum_{\substack{k=1 \\ k\ne i}}^{n_z} \bigg(\relu(c_k) \frac{u_k}{u_k-l_k} (\min\{\|w_k\|_*d_{\|\cdot\|}(\mathcal{X}),u_k\}-l_k) \\
    &\qquad + \relu(-c_k)\min\{\|w_k\|_*d_{\|\cdot\|}(\mathcal{X}),u_k\} \bigg).
    \end{aligned}
    \end{aligned}
    \label{eq: two-part_partition_bound_intermediate_1}
    \end{equation}
    Similarly, over the second input part, $\mathcal{X}_i^{(2)}$, we have that $w_i^\top x < 0$, and so the preactivation bounds for the second subproblem $\hl{\hat{\phi}_{\textup{LP}}^*}(\mathcal{X}_i^{(2)})$ can be taken as $l^{(2)}=l$ and
    \begin{equation*}
    	u^{(2)} = (u_1,u_2,\dots,u_{i-1},0,u_{i+1},\dots,u_{n_z}),
    \end{equation*}
    resulting in a similar bound to \eqref{eq: two-part_partition_bound_intermediate_1}:
    \begin{equation}
    \begin{aligned}
    	\Delta \hl{\phi^*}(\mathcal{X}_i^{(2)}) &\le \sum_{\substack{k=1 \\ k\ne i}}^{n_z} \bigg(\relu(c_k) \frac{u_k}{u_k-l_k} (\min\{\|w_k\|_*d_{\|\cdot\|}(\mathcal{X}),u_k\}-l_k) \\
    	&\qquad + \relu(-c_k)\min\{\|w_k\|_*d_{\|\cdot\|}(\mathcal{X}),u_k\} \bigg).
    	\end{aligned}
    	\label{eq: two-part_partition_bound_intermediate_2}
    \end{equation}
    Putting the two bounds \eqref{eq: two-part_partition_bound_intermediate_1} and \eqref{eq: two-part_partition_bound_intermediate_2} together and using the fact that $\hl{\phi^*}(\mathcal{X}_i^{(j)})\le \hl{\phi^*}(\mathcal{X})$ for all $j\in\{1,2\}$, we find that
    \begin{align*}
        \Delta \hl{\phi^*}(\{\mathcal{X}_i^{(1)},\mathcal{X}_i^{(2)}\}) &= \max_{j\in\{1,2\}} \left( \hl{\hat{\phi}_{\textup{LP}}^*}(\mathcal{X}_i^{(j)}) - \hl{\phi^*}(\mathcal{X}) \right)
        \\
        &\le \max_{j\in\{1,2\}}\left( \hl{\hat{\phi}_{\textup{LP}}^*}(\mathcal{X}_i^{(j)}) - \hl{\phi^*}(\mathcal{X}_i^{(j)}) \right) \\
        &= \max_{j\in\{1,2\}} \Delta \hl{\phi^*}(\mathcal{X}_i^{(j)}) \\
        & \le |c_i| \min\{\|w_i\|_* d_{\|\cdot\|}(\mathcal{X}),u_i\} \\
        &\qquad \begin{aligned}
        & +\sum_{\substack{k=1 \\ k\ne i}}^{n_z} \bigg(\relu(c_k) \frac{u_k}{u_k-l_k} (\min\{\|w_k\|_*d_{\|\cdot\|}(\mathcal{X}),u_k\}-l_k) \\
        &\qquad + \relu(-c_k)\min\{\|w_k\|_*d_{\|\cdot\|}(\mathcal{X}),u_k\} \bigg),
        \end{aligned}
    \end{align*}
    as desired.
\end{proof}
\hspace*{\fill}

\subsubsection{Proposed Branching Scheme}

Lemma \ref{lem: two-part_partition_bound} bounds the worst-case relaxation error for each possible row-based partition. Therefore, our final step in the development of our branching scheme is to find which row minimizes the upper bound \eqref{eq: two-part_partition_bound}. This worst-case optimal branching scheme is now presented.

\begin{theorem}[Worst-case optimal LP branching]
	\label{thm: optimal_partition}
	Consider the two-part partitions defined by the rows of $W$: $\{\mathcal{X}_i^{(1)},\mathcal{X}_i^{(2)}\}$, where $\mathcal{X}_i^{(1)} = \{ x\in\mathcal{X} : w_i^\top x \ge 0 \}$ and $\mathcal{X}_i^{(2)} = \mathcal{X}\setminus \mathcal{X}_i^{(1)}$ for all $i\in \{1,2,\dots,n_z\} \eqqcolon \mathcal{I}$. The optimal partition that minimizes the worst-case relaxation error bound in \eqref{eq: two-part_partition_bound} is given by
	\begin{equation}
	i^* \in \argmin_{i\in\mathcal{I}} \relu(c_i) \frac{l_i}{u_i-l_i} \left( u_i-\min\{\|w_i\|_* d_{\|\cdot\|}(\mathcal{X}),u_i\} \right). \label{eq: optimal_partition}
	\end{equation}
\end{theorem}

\begin{proof}
    Denote the bound in \eqref{eq: two-part_partition_bound} of Lemma \ref{lem: two-part_partition_bound} by
    \begin{equation*}
        \begin{aligned}
    B(i) & \coloneqq |c_i| \min\{\|w_i\|_* d_{\|\cdot\|}(\mathcal{X}),u_i\} \\
        &\qquad \begin{aligned}
        & +\sum_{\substack{k=1 \\ k\ne i}}^{n_z} \bigg(\relu(c_k) \frac{u_k}{u_k-l_k} (\min\{\|w_k\|_*d_{\|\cdot\|}(\mathcal{X}),u_k\}-l_k) \\
        &\qquad + \relu(-c_k)\min\{\|w_k\|_*d_{\|\cdot\|}(\mathcal{X}),u_k\} \bigg),
        \end{aligned}
        \end{aligned}
        \end{equation*}
        which is the quantity to be minimized over $i\in\mathcal{I}$. We may rewrite this as
        \begin{equation*}
        \begin{aligned}
        B(i) &=\sum_{k=1}^{n_z} \bigg(\relu(c_k) \frac{u_k}{u_k-l_k} (\min\{\|w_k\|_*d_{\|\cdot\|}(\mathcal{X}),u_k\}-l_k) \\
        &\qquad \qquad + \relu(-c_k)\min\{\|w_k\|_*d_{\|\cdot\|}(\mathcal{X}),u_k\} \bigg) \\
        &\qquad + |c_i| \min\{\|w_i\|_* d_{\|\cdot\|}(\mathcal{X}),u_i\} \\
        &\qquad \begin{aligned}
        & - \bigg(\relu(c_i) \frac{u_i}{u_i-l_i} (\min\{\|w_i\|_*d_{\|\cdot\|}(\mathcal{X}),u_i\}-l_i) \\
        &\qquad + \relu(-c_i)\min\{\|w_i\|_*d_{\|\cdot\|}(\mathcal{X}),u_i\} \bigg).
        \end{aligned}
        \end{aligned}
        \end{equation*}
        Hence, using the fact that $|c_i| = \relu(c_i) + \relu(-c_i)$, we find that
        \begin{equation*}
        \begin{aligned}
            \min_{i\in\mathcal{I}} B(i) &= \min_{i\in\mathcal{I}} \Bigg( |c_i| \min\{\|w_i\|_* d_{\|\cdot\|}(\mathcal{X}),u_i\} \\
        &\qquad \begin{aligned}
        & - \bigg(\relu(c_i) \frac{u_i}{u_i-l_i} (\min\{\|w_i\|_*d_{\|\cdot\|}(\mathcal{X}),u_i\}-l_i) \\
        &\qquad + \relu(-c_i)\min\{\|w_i\|_*d_{\|\cdot\|}(\mathcal{X}),u_i\} \bigg) \Bigg)
        \end{aligned} \\
        &\begin{aligned}
        & = \min_{i\in\mathcal{I}} \Bigg(  \min\{\|w_i\|_* d_{\|\cdot\|}(\mathcal{X}),u_i\} \bigg( |c_i| - \relu(c_i) \frac{u_i}{u_i-l_i} - \relu(-c_i) \bigg) \\
         &\qquad + \relu(c_i) \frac{u_i l_i}{u_i-l_i}\Bigg)
         \end{aligned} \\
         &\begin{aligned}
         & = \min_{i\in\mathcal{I}} \Bigg(  \min\{\|w_i\|_* d_{\|\cdot\|}(\mathcal{X}),u_i\} \bigg( \relu(c_i) - \relu(c_i) \frac{u_i}{u_i-l_i} \bigg) \\
         &\qquad + \relu(c_i) \frac{u_i l_i}{u_i-l_i}\Bigg)
         \end{aligned} \\
         &= \min_{i\in\mathcal{I}} \relu(c_i) \frac{l_i}{u_i-l_i} \left( u_i-\min\{\|w_i\|_* d_{\|\cdot\|}(\mathcal{X}),u_i\} \right),
        \end{aligned}
        \end{equation*}
    as desired.
\end{proof}
\hspace*{\fill}

Theorem \ref{thm: optimal_partition} provides the branching scheme that optimally reduces the worst-case relaxation error that we seek. We remark its simplicity: to decide which row to partition along, it suffices to enumerate the values $\relu(c_i)\frac{l_i}{u_i-l_i}(u_i - \min\{\|w_i\|_*d_{\|\cdot\|}(\mathcal{X}),u_i\})$ for $i\in\{1,2,\dots,n_z\}$, then choose the row corresponding to the minimum amongst these values. In practice (especially when using $\|\cdot\|=\|\cdot\|_\infty$), $d_{\|\cdot\|}(\mathcal{X})$ tends to be relatively small, making these values approximately equal to $\relu(c_i) \frac{u_il_i}{u_i-l_i}$. In our experiments that follow, we find that using these simplified values to select the partition does not degrade performance. Note that this optimization over $i$ scales linearly with the dimension $n_z$, and the resulting LP subproblems on each input part only require the addition of one extra linear constraint, meaning that this partitioning scheme is highly efficient.

\hl{We emphasize that our relaxation error bounds in Theorem~\ref{thm: generalized_worst-case_relaxation_bound} and Lemma~\ref{lem: two-part_partition_bound}, as well as our resulting worst-case optimal LP branching scheme of Theorem~\ref{thm: optimal_partition}, are proven specifically for the case of single-hidden layer networks. Deriving worst-case optimal LP branching schemes for deep networks remains a challenging open theoretical problem. However, one may easily extend our theoretically principled single-layer LP branching scheme to a multi-layer branching heuristic as follows. First, consider a layer $k\in\{1,\dots,K\}$. We generate a surrogate ``$c$''-vector of size $n_k\times 1$ so that the branching ``scores'' in \eqref{eq: optimal_partition} can be computed using this surrogate cost vector. To do this, treat the activation vector $x^{[k]}$ as the output and determine which coordinate $i_k \in\mathbb{R}^{n_k}$ of the nominal activation $\bar{x}^{[k]}$ is maximal. This means that $i_k$ would be the class assigned to $\bar{x}$ if the output were after the $k^\text{th}$ hidden layer. Then, we find the second-best coordinate $j_k\ne i_k$ so that $\bar{x}_{i_k}^{[k]} \ge \bar{x}_{j_k}^{[k]} \ge \bar{x}_i^{[k]}$ for all other $i$. Afterwards, the surrogate cost vector can be taken as $c^{[k]} \coloneqq e_{j_k} - e_{i_k}$, meaning that it serves as a measure of whether the classification after the $k^\text{th}$ hidden layer changes from $i_k$ to $j_k$. Then, the branching score $s_i^{[k]} \coloneqq \relu(c_i^{[k]}) \frac{u_i^{[k]} l_i^{[k]}}{u_i^{[k]} - l_i^{[k]}}$ from Theorem~\ref{thm: optimal_partition} corresponding to the $i^\text{th}$ neuron in layer $k$ can be computed. Of course, the surrogate vector $c^{[k]}$ is only used to compute the branching scores, and the full network and original cost vector $c$ are used in the resultant branched LP. This multi-layer heuristic is summarized in Algorithm~\ref{alg: multi-layer}.

\begin{algorithm}[H]
   \caption{Heuristic extension of LP branching to multi-layer networks}
   \label{alg: multi-layer}
   \hl{
   \textbf{Input:} $\bar{x}$, $K$, $l^{[1]},\dots,l^{[K]}$, $u^{[1]},\dots,u^{[K]}$
\begin{algorithmic}[1]
    \STATE \textbf{for} $k=1,\dots,K$
    \STATE \hspace*{\algorithmicindent} \textbf{compute} nominal activation vector $\bar{x}^{[k]}$
    \STATE \hspace*{\algorithmicindent} \textbf{compute} nominal class $i_k \in \argmax_{i\in\{1,\dots,n_k\}} \bar{x}_i^{[k]}$
    \STATE \hspace*{\algorithmicindent} \textbf{compute} runner-up class $j_k \in \argmax_{j\in\{1,\dots,n_k\} \setminus \{i_k\}} \bar{x}_j^{[k]}$
    \STATE \hspace*{\algorithmicindent} \textbf{assign} surrogate ``$c$''-vector $c^{[k]} \gets e_{j_k} - e_{i_k}$
    \STATE \hspace*{\algorithmicindent} \textbf{assign} branching scores $s_i^{[k]} \gets \relu(c_i^{[k]})\frac{u_i^{[k]} l_i^{[k]}}{ u_i^{[k]} - l_i^{[k]}}$, $i\in\{1,\dots,n_k\}$
    \STATE \textbf{return} branching scores $s_i^{[k]}$, $i\in\{1,\dots,n_k\}$, $k\in\{1,\dots,K\}$
\end{algorithmic}
}
\end{algorithm}
}

We also note that Theorem \ref{thm: optimal_partition} can be immediately extended to design multi-part partitions in two interesting ways. First, by ordering the values $\relu(c_i)\frac{u_il_i}{u_i-l_i}$, we are ordering the optimality of the rows $w_i^\top$ to partition along. Therefore, by partitioning along the $n_p>1$ rows corresponding to the smallest $n_p$ of these values, Theorem \ref{thm: optimal_partition} provides a strategy to design an effective $2^{n_p}$-part partition, in the case one prefers to perform more than just a two-part partition. Second, Theorem \ref{thm: optimal_partition} can be used directly in a branch-and-bound algorithm. See Section \ref{sec: bab} for implementation details. In our experiments that follow, we find that this technique works particularly well; our partition yields a branching method for branch-and-bound that outperforms the state-of-the-art per-neuron branching technique on small-to-moderately-sized single-hidden layer networks, and achieves comparable performance on large-scale deep models.

\subsubsection{Optimal Partitioning is NP-Hard}
\label{sec: optimal_partitioning_is_np-hard}

In this section, we show that finding a row-based partition that minimizes the actual LP relaxation error is an NP-hard problem. Recall that this approach is in contrast to our previous approach in the sense that our optimal partition in Theorem \ref{thm: optimal_partition} minimizes the worst-case relaxation error. Consequently, the results of this section show that the partition given by Theorem \ref{thm: optimal_partition} is in essence an optimal \emph{tractable} LP partitioning scheme.

To start, recall the robustness certification problem for a $K$-layer ReLU neural network:
\begin{equation}
\begin{aligned}
    & \text{maximize} && c^\top x^{[K]} & \\
    & \text{subject to} && x^{[0]}\in\mathcal{X}, & \\
    &&& x^{[k+1]} = \relu(W^{[k]}x^{[k]}), & k\in\{ 0,1,\dots,K-1 \},
\end{aligned} \label{eq: unrelaxed_certification_k_layer}
\end{equation}
where the optimal value of \eqref{eq: unrelaxed_certification_k_layer} is denoted by $\hl{\phi^*}(\mathcal{X})$. Moreover, recall the LP relaxation of \eqref{eq: unrelaxed_certification_k_layer}:
\begin{equation}
\begin{aligned}
    & \text{maximize} && c^\top x^{[K]} & \\
    & \text{subject to} && x^{[0]}\in\mathcal{X}, & \\
    &&& x^{[k+1]}\ge W^{[k]} x^{[k]}, & k\in\{ 0,1,\dots,K-1 \}, \\
    &&& x^{[k+1]}\ge 0, & k\in\{ 0,1,\dots,K-1 \}, \\
    &&& x^{[k+1]}\le u^{[k+1]}\odot (W^{[k]}x^{[k]}-l^{[k+1]})\oslash(u^{[k+1]}-l^{[k+1]}), & k\in\{ 0,1,\dots,K-1 \}.
\end{aligned} \label{eq: optimal_partitioning}
\end{equation}

As suggested by the motivating partition of Proposition \ref{prop: motivating_partition}, consider partitioning the input uncertainty set into $2^{n_p}$ parts based on $n_p$ preactivation decision boundaries corresponding to activation functions in the first layer. In particular, for each $j\in \mathcal{J}_{p}\coloneqq \{j_1,j_2,\dots,j_{n_p}\}\subseteq\{1,2,\dots,n_1\}$ we partition the input uncertainty set along the hyperplane $w^{[0]\top}_{j} x^{[0]} = 0$, giving rise to the partition $\{\mathcal{X}^{(1)},\mathcal{X}^{(2)},\dots,\mathcal{X}^{(2^{n_p})}\}$. Note that, for all $j\in\mathcal{J}_{p}$, the partition implies that the $j^\text{th}$ coordinate of the first layer's ReLU equality constraint becomes linear and exact on each part of the partition. Therefore, for all $j'\in\{1,2,\dots,2^{n_p}\}$, we may write the LP relaxation over part $\mathcal{X}^{(j')}$ as
\begin{equation}
    \begin{aligned}
    & \text{maximize} && c^\top x^{[K]} & \\
    & \text{subject to} && x^{[0]}\in\mathcal{X}^{(j')}, & \\
    &&& x^{[k+1]}\ge W^{[k]}x^{[k]}, & k\in\{ 0,1,\dots,K-1 \}, \\
    &&& x^{[k+1]}\ge 0, & k\in\{ 0,1,\dots,K-1 \}, \\
    &&& x^{[k+1]}\le u^{[k+1]}\odot (W^{[k]}x^{[k]}-l^{[k+1]})\oslash(u^{[k+1]}-l^{[k+1]}), & k\in\{ 0,1,\dots,K-1 \}, \\
    &&& x^{[1]}_j = \relu(w_j^{[0]\top}x^{[0]}), &  j\in\mathcal{J}_p,
    \end{aligned} \label{eq: def_partitioned_objective_subprob}
\end{equation}
with optimal objective value denoted by $\hl{\hat{\phi}_{\textup{LP}}^*}(\mathcal{X}^{(j')})$, and thus the partitioned LP relaxation becomes
\begin{equation}
    \hl{\phi^*_{\mathcal{J}_p}}(\mathcal{X}) \coloneqq \max_{j'\in\{1,2,\dots,2^{n_p}\}} \hl{\hat{\phi}_{\textup{LP}}^*}(\mathcal{X}^{(j')}). \label{eq: def_partitioned_objective}
\end{equation}
To reiterate, the final equality constraint in \eqref{eq: def_partitioned_objective_subprob} is linear over the restricted feasible set $\mathcal{X}^{(j')}$, which makes the problem \eqref{eq: def_partitioned_objective} a partitioned linear program.

If we now allow the indices used to define the partition, namely $\mathcal{J}_p$, to act as a variable, we can search for the optimal $n_p$ rows of the first layer that result in the tightest partitioned LP relaxation. To this end, the problem of optimal partitioning in the first layer is formulated as

\begin{equation}
    \begin{aligned}
    & \underset{\mathcal{J}_p\subseteq\{1,2,\dots,n_1\}}{\text{minimize}} && f_{\mathcal{J}_p}^*(\mathcal{X})\\
    & \text{subject to} && |\mathcal{J}_p|= n_p.
    \end{aligned} \label{eq: first_layer_partition_problem}
\end{equation}

In what follows, we prove the NP-hardness of the optimal partitioning problem \eqref{eq: first_layer_partition_problem}, thereby supporting the use of the worst-case sense optimal partition developed in Theorem \ref{thm: optimal_partition}. To show the hardness of \eqref{eq: first_layer_partition_problem}, we reduce an arbitrary instance of an NP-hard problem, the Min-$\mathcal{K}$-Union problem, to an instance of \eqref{eq: first_layer_partition_problem}. The reduction will show that the Min-$\mathcal{K}$-Union problem can be solved by solving an optimal partitioning problem. Before we proceed, we first recall the definition of the Min-$\mathcal{K}$-Union problem.

\begin{definition}[Min-$\mathcal{K}$-Union problem \citep{hochbaum1996approximating}]
Consider a collection of $n$ sets $\{\mathcal{S}_1,\mathcal{S}_2,\dots,\mathcal{S}_n\}$, where $\mathcal{S}_j$ is finite for all $j\in\{1,2,\dots,n\}$, and a positive integer $\mathcal{K}\le n$. Find $\mathcal{K}$ sets in the collection whose union has minimum cardinality, i.e., find a solution $\mathcal{J}^*$ of the following optimization problem:
\begin{equation}
    \begin{aligned}
    & \underset{\mathcal{J}\subseteq\{1,2,\dots,n\}}{\emph{minimize}} && \left| \bigcup_{j\in\mathcal{J}} \mathcal{S}_{j} \right| \\
    & \emph{subject to} && |\mathcal{J}| = \mathcal{K}.
    \end{aligned}
\end{equation}
\end{definition}

Remark the similarities between the optimal partitioning problem and the Min-$\mathcal{K}$-Union problem. In particular, if we think of the convex upper envelopes of the relaxed ReLU constraints as a collection of sets, then the goal of finding the optimal $n_p$ input coordinates to partition along is intuitively equivalent to searching for the $\mathcal{K}=n_1-n_p$ convex upper envelopes with minimum size, i.e., those with the least amount of relaxation. This perspective shows that the optimal partitioning problem is essentially a Min-$\mathcal{K}$-Union problem over the collection of relaxed ReLU constraint sets. Since the Min-$\mathcal{K}$-Union problem is NP-hard in general, it is not surprising that the optimal partitioning problem is also NP-hard. Indeed, this result is formalized in the following proposition.

\begin{proposition}[NP-hardness of optimal partition]
\label{prop: np-hardness_of_optimal_partition}
Consider the partitioned LP relaxation \eqref{eq: def_partitioned_objective} of the $K$-layer ReLU neural network certification problem. The optimal partitioning problem in the first-layer, as formulated in \eqref{eq: first_layer_partition_problem}, is NP-hard.
\end{proposition}

\begin{proof}
See Appendix \ref{sec: proof_of_np-hardness_optimal_partition}.
\end{proof}
\hspace*{\fill}

This concludes our development and analysis for partitioning the LP relaxation. In the next section, we follow a similar line of reasoning to develop a branching scheme for the other popular convex robustness certification technique, i.e., the SDP relaxation. Despite approaching this relaxation from the same partitioning perspective as the LP, the vastly different geometries of the LP and SDP feasible sets make the branching procedures quite distinct.

\section{Partitioned SDP Relaxation}
\label{sec: partitioned_sdp_relaxation}

\subsection{Tightening of the Relaxation}
\label{sec: tightening_of_the_relaxation_sdp}

As with the LP relaxation, we begin by showing that the SDP relaxation error is decreased when the input uncertainty set is partitioned. This proposition is formalized below.

\begin{proposition}[Improving the SDP relaxation bound]
	\label{prop: improving_the_sdp_relaxation_bound}
	Consider a neural network with one hidden ReLU layer. Let $\{\mathcal{X}^{(j)} : j\in\{1,2,\dots,p\}\}$ be a partition of $\mathcal{X}$. For the $j^\text{th}$ input part $\mathcal{X}^{(j)}$, denote the corresponding input bounds by $l \le l^{(j)} \le x \le u^{(j)} \le u$, where $x\in\mathcal{X}^{(j)}$. Then, it holds that
	\begin{equation}
	\max_{j\in\{1,2,\dots,p\}} \hl{\hat{\phi}_{\textup{SDP}}}(\mathcal{X}^{(j)}) \le \hl{\hat{\phi}_{\textup{SDP}}}(\mathcal{X}). \label{eq: improving_the_sdp_relaxation_bound}
	\end{equation}
\end{proposition}

\begin{proof}
See Appendix \ref{sec: proof_of_partition_improvement_sdp}.
\end{proof}
\hspace*{\fill}

Proposition \ref{prop: improving_the_sdp_relaxation_bound} guarantees that partitioning yields a tighter SDP relaxation. However, it is not immediately clear how to design the partition in order to maximally reduce the relaxation error. Indeed, a poorly designed partition may even yield an equality in the bound \eqref{eq: improving_the_sdp_relaxation_bound}. One notable challenge in designing the SDP partition relates to an inherent difference between the SDP relaxation and the LP relaxation. With the LP relaxation, the effect of partitioning can be visualized by how the geometry of the feasible set changes; see Figure \ref{fig: weight-based_partition}. However, with the SDP, the relaxation comes from dropping the nonconvex rank constraint, the geometry of which is more abstract and harder to exploit.

In the next section, we develop a bound measuring how far the SDP solution is from being rank-1, which corresponds to an exact relaxation, where the improvement in \eqref{eq: improving_the_sdp_relaxation_bound} is as good as possible. By studying the geometry of the SDP feasible set through this more tractable bound, we find that the partition design for the SDP naturally reduces to a uniform partition along the coordinate axes of the input set.

\subsection{Motivating Partition}
\label{sec: motivating_partition_sdp}

In this section, we seek the form of a partition that best reduces the SDP relaxation error. By restricting our focus to ReLU networks with one hidden layer, we develop a simple necessary condition for the SDP relaxation to be exact, i.e., for the matrix $P$ to be rank-1. We then work on the violation of this condition to define a measure of how close $P$ is to being rank-1 in the case it has higher rank. Next, we develop a tractable upper bound on this rank-1 gap. Finally, we formulate an optimization problem in which we search for a \hl{coordinate-wise} partition of the input uncertainty set that minimizes our upper bound. \hl{Since this partition robustly minimizes our rank-1 gap-based upper bound on the relaxation error, we refer to it as ``worst-case optimal,'' similarly to our previous worst-case optimal LP partition.} We show that such a \hl{coordinate-wise} worst-case optimal SDP partition takes the form of a uniform division of the input set. The result motivates the use of uniform partitions of the input uncertainty set, and in Section \ref{sec: partitioning_scheme_sdp}, we answer the question of which coordinate is best to uniformly partition along. Note that, despite the motivating partition being derived for networks with one hidden layer, the relaxation tightening in Proposition \ref{prop: improving_the_sdp_relaxation_bound} still holds for multi-layer networks. Indeed, the experiments in Section \ref{sec: simulation_results} will show that the resulting SDP branching scheme design maintains a relatively constant efficacy as the number of layers increases.
	
	\begin{proposition}[Necessary condition for exact SDP]
	\label{prop: rank-1_necessary_condition}
	Let $P^*\in\mathbb{S}^{n_x+n_z+1}$ denote a solution to the semidefinite programming relaxation \eqref{eq: sdp_relaxation}. If the relaxation is exact, meaning that $\rank(P^*)=1$, then the following conditions hold:
	\begin{equation}
	    \tr(P_{xx}^*) = \|P_x^*\|_2^2, \qquad \tr(P_{zz}^*) = \|P_z^*\|_2^2. \label{eq: rank-1_necessary_condition}
	\end{equation}
	\end{proposition}
	
	\begin{proof}
	    Since the SDP relaxation is exact, it holds that $\rank(P^*)=1$. Therefore, $P^*$ can be expressed as
	    \begin{equation*}
	        P^* = \begin{bmatrix} 1 \\ v \\ w \end{bmatrix}
	        \begin{bmatrix} 1 & v^\top & w^\top \end{bmatrix} = \begin{bmatrix}
	        1 & v^\top & w^\top \\
	        v & vv^\top & vw^\top \\
	        w & wv^\top & ww^\top
	        \end{bmatrix}
	    \end{equation*}
	    for some vectors $v\in\mathbb{R}^{n_x}$ and $w\in\mathbb{R}^{n_z}$. Recall the block decomposition of $P^*$:
	    \begin{equation*}
	        P^* = \begin{bmatrix}
	        P_1^* & P_x^{*\top} & P_z^{*\top} \\
	        P_x^* & P_{xx}^* & P_{xz}^* \\
	        P_z^* & P_{zx}^* & P_{zz}^*
	        \end{bmatrix}.
	    \end{equation*}
	    Equating coefficients, we find that $P_{xx}^* = vv^\top = P_x^* P_x^{*\top}$ and $P_{zz}^* = ww^\top = P_z^* P_z^{*\top}$. Therefore,
	    \begin{equation*}
	        \tr(P_{xx}^*) = \tr(P_x^* P_x^{*\top}) = \tr(P_x^{*\top} P_x^*) = \|P_x^*\|_2^2,
	    \end{equation*}
	    proving the first condition in \eqref{eq: rank-1_necessary_condition}. The second condition follows in the same way.
	\end{proof}
	\hspace*{\fill}

Enforcing the conditions \eqref{eq: rank-1_necessary_condition} as constraints in the SDP relaxation may assist in pushing the optimization variable $P$ towards a rank-1 solution. However, because the conditions in \eqref{eq: rank-1_necessary_condition} are nonlinear equality constraints in the variable $P$, we cannot impose them directly on the SDP without making the problem nonconvex. Instead, we will develop a convex method based on the rank-1 conditions \eqref{eq: rank-1_necessary_condition} that can be used to motivate the SDP solution to have a lower rank.
	
	In the general case that $\rank(P)=r\ge 1$, $P$ may be written as $P=VV^\top$, where
	\begin{equation*}
	    V = \begin{bmatrix}
	    e^\top \\ X \\ Z
	    \end{bmatrix}, ~ e\in\mathbb{R}^r, ~ X\in\mathbb{R}^{n_x\times r}, ~ Z\in\mathbb{R}^{n_z\times r},
	\end{equation*}
	and where the vector $e$ satisfies the equation $e^\top e = \|e\|_2^2 = 1$. The $i^\text{th}$ row of $X$ (respectively, $Z$) is denoted by $X_i^\top \in\mathbb{R}^{1\times r}$ (respectively, $Z_i^\top\in\mathbb{R}^{1\times r}$). Under this expansion, we find that $P_x = Xe$, $P_z = Ze$, $P_{xx} = XX^\top$, and $P_{zz} = ZZ^\top$. Therefore, the conditions \eqref{eq: rank-1_necessary_condition} can be written as
	\begin{equation*}
	    \tr(XX^\top) = \|Xe\|_2^2, \qquad \tr(ZZ^\top) = \|Ze\|_2^2.
	\end{equation*}
	
	To simplify the subsequent analysis, we will restrict our attention to the first of these two necessary conditions for $P$ to be rank-1. As the simulation results in Section \ref{sec: simulation_results} show, this restriction still yields significant reduction in relaxation error. Now, note that
	\begin{equation*}
	    \tr(XX^\top) = \sum_{i=1}^{n_x}(XX^\top)_{ii} = \sum_{i=1}^{n_x} \|X_i\|_2^2,
	\end{equation*}
	where $(XX^\top)_{ii}$ is the $(i,i)$ element of the matrix $XX^\top$, and also that
	\begin{equation*}
	    \|Xe\|_2^2 = \sum_{i=1}^{n_x} (Xe)_i^2 = \sum_{i=1}^{n_x} (X_i^\top e)^2,
	\end{equation*}
	where $(Xe)_i$ is the $i^\text{th}$ element of the vector $Xe$. Therefore, the rank-1 necessary condition is equivalently written as
	\begin{equation*}
	    g(P) \coloneqq \sum_{i=1}^{n_x} (\|X_i\|_2^2 - (X_i^\top e)^2) = 0,
	\end{equation*}
	where $g(P)$ serves as a measure of the rank-1 gap. Note that $g(P)$ is solely determined by $P=VV^\top$, even though it is written in terms of $X$ and $e$, which are blocks of $V$. In general, $g(P)\ge 0$ when $\rank(P)\ge 1$.
	
	\begin{lemma}[Rank-1 gap]
	\label{lem: rank-1_gap}
	Let $P\in\mathbb{S}^{n_x+n_z+1}$ be an arbitrary feasible point for the SDP relaxation \eqref{eq: sdp_relaxation}. The rank-1 gap $g(P)$ is nonnegative, and is zero if $P$ is rank-1.
	\end{lemma}
	\begin{proof}
	    By the Cauchy-Schwarz inequality, we have that $|X_i^\top e| \le \|X_i\|_2 \|e\|_2$ for all $i\in\{1,2,\dots,n_x\}$. Since $P$ is feasible for \eqref{eq: sdp_relaxation} we also have that $\|e\|_2 = 1$, so squaring both sides of the inequality gives that $(X_i^\top e)^2 \le \|X_i\|_2^2$. Summing these inequalities over $i$ gives
	    \begin{equation*}
	        g(P) = \sum_{i=1}^{n_x} (\|X_i\|_2^2 - (X_i^\top e)^2) \ge 0.
	    \end{equation*}
	    If $P$ is rank-1, then the dimension $r$ of the vectors $e$ and $X_i$ is equal to $1$. That is, $e,X_i\in\mathbb{R}$. Hence, $\|X_i\|_2 = |X_i|$ and $|e| = \|e\|_2 = 1$, yielding $\|X_i\|_2^2 - (X_i^\top e)^2 = X_i^2 - X_i^2 e^2 = 0$. Therefore, $g(P)=0$ in the case that $\rank(P)=1$.
	\end{proof}
	\hspace*{\fill}

	Since $g(P) = 0$ is necessary for $P$ to be rank-1 and $g(P) \ge 0$, it is desirable to make $g(P^*)$ as small as possible at the optimal solution $P^*$ of the partitioned SDP relaxation. Indeed, this is our partitioning motivation: we seek to partition the input uncertainty set to minimize $g(P^*)$, in order to influence $P^*$ to be of low rank. However, there is a notable hurdle with this approach. In particular, the optimal solution $P^*$ depends on the partition we choose, and finding a partition to minimize $g(P^*)$ in turn depends on $P^*$ itself. To overcome this cyclic dependence, we propose first bounding $g(P^*)$ by a worst-case upper bound, and then choosing an optimal partition to minimize the upper bound. This will make the partition design tractable, resulting in a closed-form solution.
	
	To derive the upper bound on the rank-1 gap at optimality, let $\{\mathcal{X}^{(j)} : j\in\{1,2,\dots,p\}\}$ denote the partition of $\mathcal{X}$. For the $j^{\text{th}}$ input part $\mathcal{X}^{(j)}$, denote the corresponding input bounds by $l^{(j)},u^{(j)}$. The upper bound is derived below.
	
	\begin{lemma}[Rank-1 gap upper bound]
	\label{lem: rank-1_gap_upper_bound}
	The rank-1 gap at the solution $P^*$ of the partitioned SDP satisfies
	\begin{equation}
        0 \le g(P^*) \le \frac{1}{4}\sum_{i=1}^{n_x} \max_{j\in\{1,2,\dots,p\}} (u_i^{(j)}-l_i^{(j)})^2. \label{eq: rank-1_gap_bound}
    \end{equation}
	\end{lemma}
	
	\begin{proof}
	    The left inequality is a direct result of Lemma \ref{lem: rank-1_gap}. For the right inequality, note that
	    \begin{equation}
	    g(P^*) \le \max_{j\in\{1,2,\dots,p\}}\sup_{P\in\mathcal{N}_\textup{SDP}^{(j)},~P_x\in\mathcal{X}^{(j)}} g(P) \le \sum_{i=1}^{n_x} \max_{j\in\{1,2,\dots,p\}}\sup_{P\in\mathcal{N}_\textup{SDP}^{(j)},~P_x\in\mathcal{X}^{(j)}} (\|X_i\|_2^2 - (X_i^\top e)^2). \label{eq: rank-1_gap_bound_intermediate}
	    \end{equation}
	Let us focus on the optimization over the $j^\text{th}$ part of the partition, namely,
	\begin{equation*}
	    \sup_{P\in\mathcal{N}_\textup{SDP}^{(j)},~P_x\in\mathcal{X}^{(j)}} (\|X_i\|_2^2 - (X_i^\top e)^2).
	\end{equation*}
	To bound this quantity, we analyze the geometry of the SDP relaxation over part $j$, following the methodology of \citet{raghunathan2018semidefinite}; see Figure \ref{fig: sdp_geometry}.
	
	\begin{figure}[tbh]
        \centering
        \includegraphics[width=0.6\linewidth]{figures/sdp_geometry.tikz}
        \caption{Geometry of the SDP relaxation in coordinate $i$ over part $j$ of the partition. The shaded region shows the feasible $X_i$ satisfying the input constraint \citep{raghunathan2018semidefinite}.}
        \label{fig: sdp_geometry}
    \end{figure}
	
	The shaded circle represents the set of feasible $X_i$ over part $j$ of the partition, namely, those satisfying the $i^\text{th}$ coordinate of the constraint $\diag(P_{xx}) \le (l^{(j)}+u^{(j)})\odot P_x - l^{(j)}\odot u^{(j)}$. To understand this, note that the constraint is equivalent to $\|X_i\|_2^2 \le (l_i^{(j)}+u_i^{(j)}) X_i^\top e - l_i^{(j)} u_i^{(j)}$, or, more geometrically written, that $\|X_i-\tfrac{1}{2}(u_i^{(j)}+l_i^{(j)})e\|_2^2 \le \left( \tfrac{1}{2}(u_i^{(j)}-l_i^{(j)}) \right)^2$. This shows that $X_i$ is constrained to a 2-norm ball of radius $r_i^{(j)}=\tfrac{1}{2}(u_i^{(j)}-l_i^{(j)})$ centered at $\tfrac{1}{2}(u_i^{(j)}+l_i^{(j)})e$, as shown in Figure \ref{fig: sdp_geometry}.
    
    The geometry of Figure \ref{fig: sdp_geometry} immediately shows that $\|X_i\|_2^2 = a^2 + (X_i^\top e)^2$ and $r_i^{(j)2} = (a+b)^2 + (X_i^\top e - \tfrac{1}{2}(u_i^{(j)}+l_i^{(j)}))^2$, and therefore
    \begin{equation*}
        \|X_i\|_2^2 - (X_i^\top e)^2 = a^2 = r_i^{(j)2} - (X_i^\top e - \tfrac{1}{2}(u_i^{(j)}+l_i^{(j)}))^2 - 2ab - b^2.
    \end{equation*}
    Since $a$ and $b$ are nonnegative,
    \begin{gather*}
        \sup_{P\in\mathcal{N}_\textup{SDP}^{(j)},~P_x\in\mathcal{X}^{(j)}} \|X_i\|_2^2 - (X_i^\top e)^2 \\
	= \sup_{P\in\mathcal{N}_\textup{SDP}^{(j)},~P_x\in\mathcal{X}^{(j)}} (r_i^{(j)2} - (X_i^\top e - \tfrac{1}{2}(u_i^{(j)}+l_i^{(j)}))^2 -2ab - b^2) \\
        \le \sup_{P\in\mathcal{N}_\textup{SDP}^{(j)},~P_x\in\mathcal{X}^{(j)}} (r_i^{(j)2} - (X_i^\top e - \tfrac{1}{2}(u_i^{(j)}+l_i^{(j)}))^2 ) \le r_i^{(j)2} = \frac{1}{4}(u_i^{(j)}-l_i^{(j)})^2.
    \end{gather*}
    Thus, \eqref{eq: rank-1_gap_bound_intermediate} gives
    \begin{equation*}
        g(P^*) \le \frac{1}{4}\sum_{i=1}^{n_x} \max_{j\in\{1,2,\dots,p\}} (u_i^{(j)}-l_i^{(j)})^2,
    \end{equation*}
    as desired.
	\end{proof}
	\hspace*{\fill}

\hl{
 \begin{remark}
     The upper bound in \eqref{eq: rank-1_gap_bound} may not be tight in general, since, if the solution $P^*$ to the partitioned SDP is rank-1---which occurs under some technical conditions (see \citet{zhang2020tightness})---then $g(P^*)=0$, whereas the upper bound will be strictly positive whenever the input bounds are such that $l_i^{(j)} < u_i^{(j)}$ for some $i,j$.
 \end{remark}
 }

	With Lemma \ref{lem: rank-1_gap_upper_bound} in place, we now have an upper bound on the rank-1 gap at optimality, in terms of the input bounds $\{l^{(j)},u^{(j)}\}_{j=1}^p$ associated with the partition. At this point, we turn to minimizing the upper bound over all valid choices of $p$-part partitions of the input uncertainty set along a given coordinate. Note that, in order for $\{l^{(j)},u^{(j)}\}_{j=1}^p$ to define valid input bounds for a $p$-part partition, it must be that the union of the input parts cover the input uncertainty set. In terms of the input bounds, this leads to the constraint that $[l,u] = \cup_{j=1}^p [l^{(j)},u^{(j)}]$, where $[l,u] \coloneqq [l_1,u_1]\times[l_2,u_2]\times\cdots\times[l_{n_x},u_{n_x}]$, and similarly for $[l^{(j)},u^{(j)}]$. Since we consider the partition along a single coordinate $k$, this constraint becomes equivalent to $\cup_{j=1}^p [l_k^{(j)},u_k^{(j)}] = [l_k,u_k]$, because all other coordinates $i\ne k$ satisfy $l_i^{(j)}=l_i$ and $u_i^{(j)}=u_i$ for all $j$ by assumption. We now give the \hl{coordinate-wise} worst-case optimal partitioning scheme for the SDP that minimizes the upper bound in Lemma \ref{lem: rank-1_gap_upper_bound}, which turns out to be a uniform split:

    \begin{theorem}[\hl{Uniform coordinate-wise SDP partition}]
    \label{thm: optimal_sdp_partition_rank-1_gap}
    Let $k\in\{1,\dots,n_x\}$ and define $\mathcal{I}_k=\{1,2,\dots,n_x\}\setminus\{k\}$. Consider the optimization problem of finding the partition to minimize the upper bound \eqref{eq: rank-1_gap_bound}, namely
    \begin{equation}
        \begin{aligned}
        & \underset{\mathcal{P}=\{l^{(j)},u^{(j)}\}_{j=1}^p\subseteq\mathbb{R}^{n_x}}{\emph{minimize}} && h(\mathcal{P}) = \sum_{i=1}^{n_x} \max_{j\in\{1,2,\dots,p\}}(u_i^{(j)} - l_i^{(j)})^2 \\
	& \emph{subject to} && \bigcup_{j=1}^p [l_k^{(j)},u_k^{(j)}] = [l_k,u_k], & i\in\mathcal{I}_k,~j\in\{1,2,\dots,p\}, & \\
	&&& l_i^{(j)} = l_i, & i\in\mathcal{I}_k,~j\in\{1,2,\dots,p\}, \\
	&&& u_i^{(j)} = u_i, & i\in\mathcal{I}_k,~j\in\{1,2,\dots,p\},
        \end{aligned} \label{eq: optimal_sdp_partition_optimization}
    \end{equation}
    Consider also the uniform partition defined by $\bar{\mathcal{P}} = \{\bar{l}^{(j)},\bar{u}^{(j)}\}_{j=1}^p\subseteq\mathbb{R}^{n_x}$, where
    \begin{align*}
        \bar{l}_i^{(j)} &= \begin{aligned}
        \begin{cases}
        \frac{j-1}{p}(u_i-l_i)+l_i & \emph{if $i=k$}, \\
        l_i & \emph{otherwise},
        \end{cases}
        \end{aligned} \\
        \bar{u}_i^{(j)} &=
        \begin{aligned}
        \begin{cases}
        \frac{j}{p}(u_i-l_i)+l_i & \emph{if $i=k$}, \\
        u_i & \emph{otherwise},
        \end{cases}
        \end{aligned}
    \end{align*}
    for all $j\in\{1,2,\dots,p\}$. It holds that $\bar{\mathcal{P}}$ is a solution to \eqref{eq: optimal_sdp_partition_optimization}.
    \end{theorem}
    
    \begin{proof}
        To prove the result, we show that the proposed $\bar{\mathcal{P}}$ is feasible for the optimization, and that $h(\bar{\mathcal{P}}) \le h(\mathcal{P})$ for all feasible $\mathcal{P}$. First, note that it is obvious by the definition of $\bar{\mathcal{P}}$ that $\bar{l}_i^{(j)}=l_i$ and $\bar{u}_i^{(j)}=u_i$ for all $i\in\{1,2,\dots,n_x\}\setminus \{k\}$ and all $j\in\{1,2,\dots,p\}$. Therefore, to prove that $\bar{\mathcal{P}}$ is feasible, it suffices to show that $\cup_{j=1}^p [\bar{l}_k^{(j)},\bar{u}_k^{(j)}] = [l_k,u_k]$. Indeed, since
        \begin{align*}
            \bar{u}_k^{(j)} = \frac{j}{p}(u_k-l_k)+l_k=\frac{(j+1)-1}{p}(u_k-l_k)+l_k=\bar{l}_k^{(j+1)}
        \end{align*}
        for all $j\in\{1,2,\dots,p-1\}$,
        \begin{equation*}
            \bar{l}_k^{(1)} = \frac{1-1}{p}(u_k-l_k)+l_k=l_k,
        \end{equation*}
        and
        \begin{equation*}
            \bar{u}_k^{(p)} = \frac{p}{p}(u_k-l_k)+l_k=u_k,
        \end{equation*}
        we have that
        \begin{equation*}
            \bigcup_{j=1}^p [\bar{l}^{(j)}_k,\bar{u}^{(j)}_k] = [\bar{l}_k^{(1)},\bar{u}_k^{(1)}]\cup[\bar{l}_k^{(2)},\bar{u}_k^{(2)}]\cup\cdots\cup[\bar{l}_k^{(p)},\bar{u}_k^{(p)}] = [\bar{l}_k^{(1)},\bar{u}_k^{(p)}] = [l_k,u_k].
        \end{equation*}
        Hence, $\bar{\mathcal{P}}=\{\bar{l}^{(j)},\bar{u}^{(j)}\}_{j=1}^p$ is feasible.
        
        The objective at the proposed feasible point can be computed as
        \begin{gather*}
            h(\bar{\mathcal{P}}) = \sum_{i=1}^{n_x} \max_{j\in\{1,2,\dots,p\}} (\bar{u}_i^{(j)} - \bar{l}_i^{(j)})^2 = \sum_{\substack{i=1\\i\ne k}}^{n_x} \max_{j\in\{1,2,\dots,p\}}(\bar{u}_i^{(j)}-\bar{l}_i^{(j)})^2 + \max_{j\in\{1,2,\dots,p\}}(\bar{u}_k^{(j)}-\bar{l}_k^{(j)})^2 \\
            = \sum_{\substack{i=1\\i\ne k}}^{n_x} \max_{j\in\{1,2,\dots,p\}} (u_i-l_i)^2 + \max_{j\in\{1,2,\dots,p\}} \left( \frac{j}{p}(u_k-l_k)+l_k - \frac{j-1}{p}(u_k-l_k)-l_k \right)^2 \\
            = \sum_{\substack{i=1 \\ i\ne k}}^{n_x}(u_i-l_i)^2 + \max_{j\in\{1,2,\dots,p\}}\left( \frac{1}{p}(u_k-l_k) \right)^2 = C + \frac{1}{p^2}(u_k-l_k)^2,
        \end{gather*}
        where $C\coloneqq \sum_{\substack{i=1 \\ i\ne k}}^{n_x}(u_i-l_i)^2$. Now, let $\mathcal{P}=\{l^{(j)},u^{(j)}\}_{j=1}^p$ be an arbitrary feasible point for the optimization \eqref{eq: optimal_sdp_partition_optimization}. Then by a similar analysis as above, the objective value at $\mathcal{P}$ satisfies
        \begin{gather*}
            h(\mathcal{P}) = \sum_{i=1}^{n_x} \max_{j\in\{1,2,\dots,p\}}(u_i^{(j)}-l_i^{(j)})^2 = \sum_{\substack{i=1 \\ i\ne k}}^{n_x} \max_{j\in\{1,2,\dots,p\}} (u_i^{(j)}-l_i^{(j)})^2 + \max_{j\in\{1,2,\dots,p\}}(u_k^{(j)}-l_k^{(j)})^2 \\
            = C + \max_{j\in\{1,2,\dots,p\}}(u_k^{(j)}-l_k^{(j)})^2 = C + \left(\max_{j\in\{1,2,\dots,p\}}(u_k^{(j)}-l_k^{(j)})\right)^2 \\
            \ge C + \left( \frac{1}{p}\sum_{j=1}^p (u_k^{(j)}-l_k^{(j)}) \right)^2 = C + \frac{1}{p^2}\left( \sum_{j=1}^p (u_k^{(j)}-l_k^{(j)})\right)^2.
        \end{gather*}
        Since $\mathcal{P}$ is feasible, it holds that $[l_k,u_k]= \bigcup_{j=1}^p [l_k^{(j)},u_k^{(j)}]$. Therefore, by subadditivity of Lebesgue measure $\mu$ on the Borel $\sigma$-algebra of $\mathbb{R}$, we have that
        \begin{equation*}
            u_k - l_k = \mu([l_k,u_k]) = \mu\left( \bigcup_{j=1}^p [l_k^{(j)},u_k^{(j)}] \right) \le \sum_{j=1}^p \mu([l_k^{(j)},u_k^{(j)}]) = \sum_{j=1}^p (u_k^{(j)} - l_k^{(j)}).
        \end{equation*}
        Substituting this into our above expressions, we conclude that
        \begin{align*}
            h(\bar{\mathcal{P}}) = C+\frac{1}{p^2}(u_k-l_k)^2 \le C+\frac{1}{p^2}\left( \sum_{j=1}^p (u_k^{(j)}-l_k^{(j)}) \right)^2 \le h(\mathcal{P}).
        \end{align*}
        Since $\mathcal{P}$ was an arbitrary feasible point for the optimization, this implies that $\bar{\mathcal{P}}$ is a solution to the optimization.
    \end{proof}
    \hspace*{\fill}

    \hl{It is important to note that, in general, an optimal partition of the input uncertainty set that minimizes the actual SDP relaxation error may not be along a coordinate axis. However, finding such an optimal general-form partition poses a significantly challenging open problem. Our Theorem~\ref{thm: optimal_sdp_partition_rank-1_gap} shows that, by restricting ourselves to coordinate-wise partitions, the analysis becomes tractable, and we find that the partition that uniformly splits the input set along the given coordinate is an optimal one (with respect to our worst-case upper bound \eqref{eq: rank-1_gap_bound}).} This gives a well-motivated, yet simple way to design a partition of the input uncertainty set in order to push the SDP relaxation towards being rank-1, thereby reducing relaxation error.

\subsection{SDP Branching Scheme}
\label{sec: partitioning_scheme_sdp}

With the motivating partition of Section \ref{sec: motivating_partition_sdp} now established, we turn our attention from the \textit{form} of a \hl{coordinate-wise} worst-case optimal SDP partition to the \textit{coordinate} of such a partition. In particular, we seek to find the best branching scheme to minimize relaxation error of the SDP. The results of Section \ref{sec: motivating_partition_sdp} suggest using a uniform partition, and in this section we seek to find which coordinate to apply the partitioning to. Similar to the LP relaxation, we derive an optimal branching scheme by first bounding the relaxation error in the worst-case sense.

\subsubsection{Worst-Case Relaxation Bound}
\label{sec: worst-case_relaxation_bound_sdp}

In the worst-case relaxation bound of Theorem \ref{thm: worst-case_relaxation_bound_sdp} below, and the subsequent \hl{coordinate-wise} worst-case optimal SDP branching scheme proposed in Theorem \ref{thm: optimal_sdp_partition}, we restrict our attention to a single hidden ReLU layer and make the following assumption on the weight matrix.

\begin{assumption}[Normalized rows]
\label{ass: normalized_rows}
The rows of the weight matrix are assumed to be normalized with respect to the $\ell_1$-norm, i.e., that $\|w_i\|_1 = 1$ for all $i\in\{1,2,\dots,n_z\}$.
\end{assumption}

We briefly remark that Assumption \ref{ass: normalized_rows} imposes no loss of generality, as it can be made to hold for any network by a simple rescaling. In particular, if the assumption does not hold, the network architecture can be rescaled as follows:
\begin{gather*}
    z = \relu(Wx) = \relu\left( \begin{bmatrix}
    w_1^\top \\ w_2^\top \\ \vdots \\ w_{n_z}^\top
    \end{bmatrix} x \right) \\
    = \relu\left( \diag(\|w_1\|_1,\|w_2\|_1,\dots,\|w_{n_z}\|_1) \begin{bmatrix}
    \frac{w_1^\top}{\|w_1\|_1} \\ \frac{w_2^\top}{\|w_2\|_1} \\ \vdots \\ \frac{w_{n_z}^\top}{\|w_{n_z}\|_1}
    \end{bmatrix} x \right) = W_\text{scale}\relu(W_\text{norm}x),
\end{gather*}
where $W_\text{scale}=\diag(\|w_1\|_1,\|w_2\|_1,\dots,\|w_{n_z}\|_1)\in\mathbb{R}^{n_z\times n_z}$ and
\begin{equation*}
    W_\text{norm} = \begin{bmatrix}
    \frac{w_1^\top}{\|w_1\|_1} \\ \frac{w_2^\top}{\|w_2\|_1} \\ \vdots \\ \frac{w_{n_z}^\top}{\|w_{n_z}\|_1}
    \end{bmatrix} \in\mathbb{R}^{n_z\times n_x}
\end{equation*}
are the scaling and normalized factors of the weight matrix $W$, respectively. The scaling factor can therefore be absorbed into the optimization cost vector $c$, yielding a problem with normalized rows as desired.

Before introducing the worst-case relaxation bound of Theorem \ref{thm: worst-case_relaxation_bound_sdp}, we state a short lemma that will be used in proving the relaxation bound.

\begin{lemma}[Bound on elements of PSD matrices]
\label{lem: bound_elements_psd_matrices}
Let $P\in\mathbb{S}^{n}$ be a positive semidefinite matrix. Then $|P_{ij}|\le \frac{1}{2}(P_{ii}+P_{jj})$ for all $i,j\in\{1,2,\dots,n\}$.
\end{lemma}
\begin{proof}
    See Appendix \ref{sec: proof_of_bound_elements_psd_matrices}.
\end{proof}
\hspace*{\fill}

\begin{theorem}[Worst-case relaxation bound for SDP]
\label{thm: worst-case_relaxation_bound_sdp}
Consider a feedforward ReLU neural network with one hidden layer, and with the input uncertainty set $\mathcal{X}$. Let the network have input bounds $l,u\in\mathbb{R}^{n_x}$ and preactivation bounds $\hat{l},\hat{u}\in\mathbb{R}^{n_z}$. Consider also the relaxation error $\Delta \hl{\phi^*_{\textup{SDP}}}(\mathcal{X}) \coloneqq \hl{\hat{\phi}_{\textup{SDP}}}(\mathcal{X}) - \hl{\phi^*}(\mathcal{X})$. Let $P^*$ and $(x^*,z^*)$ be optimal solutions for the relaxation $\hl{\hat{\phi}_{\textup{SDP}}}(\mathcal{X})$ and the unrelaxed problem $\hl{\phi^*}(\mathcal{X})$, respectively. Given an arbitrary norm $\|\cdot\|$ on $\mathbb{R}^{n_x}$, it holds that
\begin{equation}
    \Delta \hl{\phi^*_{\textup{SDP}}}(\mathcal{X}) \le \sum_{i=1}^{n_z} \bigg( \relu(c_i)q(l,u) + \relu(-c_i)\min\{\hat{u}_i , \|w_i\|_* d_{\|\cdot\|}(\mathcal{X})\} \bigg), \label{eqn: worst-case_relaxation_bound_sdp}
\end{equation}
where $\|\cdot\|_*$ is the dual norm of $\|\cdot\|$, and where
\begin{equation*}
    q(l,u) = \max_{k\in\{1,2,\dots,n_x\}} \max\{|l_k|,|u_k|\}.
\end{equation*}
\end{theorem}

\begin{proof}
    First, recall that $\mathcal{X}\subseteq\mathbb{R}^{n_x}$ is assumed to be compact, and is therefore bounded, and hence $d_{\|\cdot\|}(\mathcal{X}) < \infty$. The definitions of $P_x^*$ and $(x^*,z^*)$ give that
    \begin{equation}
        \Delta \hl{\phi^*_{\textup{SDP}}}(\mathcal{X}) = \sum_{i=1}^{n_z} c_i((P_{z}^*)_i - z_i^*) \le \sum_{i=1}^{n_z} \Delta \hl{\phi^*_i}, \label{eqn: worst-case_relaxation_bound_sdp_intermediate_1}
    \end{equation}
    where
    \begin{equation*}
    \begin{aligned}
        \Delta \hl{\phi^*_i} &= \sup \bigg\{ c_i((P_z)_i - z_i) : z_i = \relu(w_i^\top x), ~ P_z \ge 0, ~ P_z \ge WP_x, \\
        &\qquad \diag(P_{zz}) = \diag(WP_{xz}), ~ P_1 = 1, ~ P\succeq 0, ~ x,P_x\in\mathcal{X} \bigg\}
    \end{aligned}
    \end{equation*}
    for all $i\in\{1,2,\dots,n_z\}$. Defining the auxiliary variables $P_{\hat{z}} = WP_x$ and $\hat{z} = Wx$, this is equivalent to
    \begin{equation*}
    \begin{aligned}
        \Delta \hl{\phi^*_i} &= \sup \bigg\{ c_i((P_z)_i - z_i) : z_i = \relu(\hat{z}_i), ~ P_z \ge 0, ~ P_z \ge P_{\hat{z}}, ~ \diag(P_{zz}) = \diag(WP_{xz}), \\
        &\qquad P_1 = 1, ~ P\succeq 0, ~ P_{\hat{z}} = WP_x, ~ \hat{z}_i = w_i^\top x, ~ x,P_x\in\mathcal{X} \bigg\}.
    \end{aligned}
    \end{equation*}
    If $x,P_x\in\mathcal{X}$ and $\hat{z},P_{\hat{z}}$ satisfy $\hat{z}=Wx$ and $P_{\hat{z}}=WP_x$, then $|(P_{\hat{z}})_i-\hat{z}_i| = |w_i^\top(P_x-x)|\le \|w_i\|_*\|P_x-x\|\le\|w_i\|_* d_{\|\cdot\|}(\mathcal{X})$ for all $i\in\{1,2,\dots,n_z\}$ by the Cauchy-Schwarz inequality for dual norms. Therefore,
    \begin{equation*}
    \begin{aligned}
        \Delta \hl{\phi^*_i} &\le \sup\bigg\{ c_i((P_z)_i-z_i) : z_i = \relu(\hat{z}_i), ~ P_z\ge 0, ~ P_z \ge P_{\hat{z}}, \\
		&\qquad \diag(P_{zz}) = \diag(WP_{xz}), ~ P_1 = 1, ~ P\succeq 0, ~ \hat{l}\le \hat{z}, P_{\hat{z}} \le \hat{u}, \\
	   &\qquad |(P_{\hat{z}})_k-\hat{z}_k| \le \|w_k\|_* d_{\|\cdot\|}(\mathcal{X}) ~ \text{for all $k\in\{1,2,\dots,n_z\}$}, ~ \hat{z},P_{\hat{z}}\in\mathbb{R}^{n_z} \bigg\}.
    \end{aligned}
    \end{equation*}
    We now translate the optimization variables in the above problem from $\hat{z}\in\mathbb{R}^{n_z}$ and $P\in\mathbb{S}^{1+n_x+n_z}$ to the scalars $\hat{z}_i,(P_{\hat{z}})_i\in\mathbb{R}$. To this end, we note that if $P$ is feasible for the above supremum, then
    \begin{equation*}
        \diag(P_{zz})_i = \diag(WP_{xz})_i = w_i^\top (P_{xz})_i \le \|(P_{xz})_i\|_\infty \|w_i\|_1,
    \end{equation*}
    where $(P_{xz})_i$ is the $i^\text{th}$ column of the matrix $P_{xz}$, and the inequality again comes from Cauchy-Schwarz. By the weight matrix scaling assumption, this yields
    \begin{equation*}
        \diag(P_{zz})_i \le \|(P_{xz})_i\|_\infty.
    \end{equation*}
    Now, since $P$ is positive semidefinite, Lemma \ref{lem: bound_elements_psd_matrices} gives that
    \begin{gather*}
        \|(P_{xz})_i\|_\infty = \max_{k\in\{1,2,\dots,n_x\}}|(P_{xz})_i|_k = \max_{k\in\{1,2,\dots,n_x\}}|(P_{xz})_{ki}| \\
	\le \max_{k\in\{1,2,\dots,n_z\}}\frac{1}{2}\bigg((P_{xx})_{kk} + (P_{zz})_{ii}\bigg) = \frac{1}{2}(P_{zz})_{ii} + \frac{1}{2}\max_{k\in\{1,2,\dots,n_x\}}(P_{xx})_{kk}.
    \end{gather*}
    Noting that $(P_{zz})_{ii}=\diag(P_{zz})_i$, the bound of interest becomes
    \begin{equation*}
        \diag(P_{zz})_i \le \max_{k\in\{1,2,\dots,n_z\}}(P_{xx})_{kk}.
    \end{equation*}
    We now seek to bound $(P_{xx})_{kk}$. Recall that $(P_{xx})_{kk} = \diag(P_{xx})_k \le (l_k+u_k)(P_x)_k - l_ku_k$. If $(l_k+u_k)\ge 0$, then $(P_x)_k\le u_k$ implies that $(l_k+u_k)(P_x)_k \le (l_k+u_k)u_k$, and therefore $(P_{xx})_{kk} \le (l_k+u_k)u_k-l_ku_k=u_k^2$. On the other hand, if $(l_k+u_k)< 0$, then $(P_x)_k\ge l_k$ implies that $(l_k+u_k)(P_x)_k \le (l_k+u_k)l_k$, and therefore $(P_{xx})_{kk} \le (l_k+u_k)l_k-l_ku_k=l_k^2$. Hence, in all cases, it holds that
    \begin{equation*}
        (P_{xx})_{kk} \le \mathbb{I}(l_k+u_k\ge 0) u_k^2 + \mathbb{I}(l_k+u_k < 0) l_k^2.
    \end{equation*}
    We can further simplify this bound as follows. If $l_k+u_k\ge 0$, then $u_k\ge -l_k$ and $u_k\ge l_k$, implying $|l_k|\le u_k$, so $l_k^2\le u_k^2$ and therefore $u_k^2 = \max\{l_k^2,u_k^2\}$. On the other hand, if $l_k+u_k<0$, then an analogous argument shows that $l_k^2=\max\{l_k^2,u_k^2\}$. Hence, we conclude that the above bound on $(P_{xx})_{kk}$ can be rewritten as
    \begin{equation*}
        (P_{xx})_{kk}\le\max\{l_k^2,u_k^2\}.
    \end{equation*}
    Therefore, returning to the bound on $(P_{zz})_i$, we find that
    \begin{equation*}
        \diag(P_{zz})_i \le \max_{k\in\{1,2,\dots,n_x\}} \max\{l_k^2,u_k^2\},
    \end{equation*}
    for all $i\in\{1,2,\dots,n_z\}$. Now, note that since $P\succeq 0$, the Schur complement gives that
    \begin{equation*}
        \begin{bmatrix}
        P_{xx}-P_xP_x^\top & P_{xz} - P_xP_z^\top \\
        P_{xz}^\top - P_zP_x^\top & P_{zz}-P_zP_z^\top
        \end{bmatrix} \succeq 0,
    \end{equation*}
    which implies that
    \begin{equation*}
        \diag(P_{zz}) \ge \diag(P_zP_z^\top) = P_z\odot P_z.
    \end{equation*}
    Therefore, our upper bound on the diagonal elements of $P_{zz}$ yields that
    \begin{equation*}
        (P_z)_i \le \max_{k\in\{1,2,\dots,n_x\}} \max\{|l_k|,|u_k|\} = q(l,u).
    \end{equation*}
    Hence, we have derived a condition on the component $(P_z)_i$ that all feasible $P$ must satisfy. The supremum of interest may now be further upper bounded giving rise to
    \begin{equation}
    \begin{aligned}
        \Delta \phi_i^* &\le \sup\bigg\{ c_i((P_z)_i-z_i) : z_i = \relu(\hat{z}_i), ~ (P_z)_i \ge 0, ~ (P_z)_i \ge (P_{\hat{z}})_i, ~ (P_z)_i \le q(l,u), \\
        &\qquad \hat{l}_i \le \hat{z}_i,(P_{\hat{z}})_i\le \hat{u}_i, ~ |(P_{\hat{z}})_i - \hat{z}_i| \le  \|w_i\|_* d_{\|\cdot\|}(\mathcal{X}), ~ \hat{z}_i,(P_{\hat{z}})_i\in\mathbb{R}\bigg\}, \label{eq: worst-case_relaxation_bound_sdp_intermediate}
    \end{aligned}
    \end{equation}
    which is now in terms of the scalar optimization variables $\hat{z}_i$ and $(P_{\hat{z}})_i$, as we desired. This reformulation makes it tractable to compute the supremum in \eqref{eq: worst-case_relaxation_bound_sdp_intermediate} in closed-form, which we now turn to do.
    
    First, consider the case that $c_i\ge 0$. Then we seek to maximize the difference $(P_z)_i-z_i$ subject to the given constraints. Noting that $(P_z)_i\le q(l,u)$ and $z_i\ge 0$ on the above feasible set, we remark that the objective is upper bounded as $c_i((P_z)_i-z_i)\le c_iq(l,u)$. Indeed, this upper bound is attained at the feasible point defined by $z_i=\hat{z}_i=(P_{\hat{z}})_i=0$ and $(P_z)_i = q(l,u)$. Hence, we conclude that for all $i\in\{1,2,\dots,n_z\}$ such that $c_i\ge 0$, it holds that
    \begin{equation}
        \Delta \phi_i^* \le c_i q(l,u). \label{eqn: worst-case_relaxation_bound_sdp_intermediate_2}
    \end{equation}
    Now consider the case that $c_i<0$. Then we seek to minimize the difference $(P_z)_i-z_i$ subject to the given constraints. In this case, the optimal objective value depends on the relative sizes of $\hat{u}_i$ and $\|w_i\|_* d_{\|\cdot\|}(\mathcal{X})$. In particular, when $\hat{u}_i \le \|w_i\|_* d_{\|\cdot\|}(\mathcal{X})$, the constraint $\hat{z}_i\le \hat{u}_i$ becomes active at optimum, yielding a supremum value of $-c_iu_i$. Alternatively, when $\|w_i\|_* d_{\|\cdot\|}(\mathcal{X}) \le\hat{u}_i$, the constraint $|(P_{\hat{z}})_i-\hat{z}_i|\le \|w_i\|_* d_{\|\cdot\|}(\mathcal{X})$ becomes active at optimum, yielding the supremum value of $-c_i \|w_i\|_* d_{\|\cdot\|}(\mathcal{X})$. Therefore, we conclude that for all $i\in\{1,2,\dots,n_z\}$ such that $c_i<0$, it holds that
    \begin{equation}
        \Delta \phi_i^* \le -c_i\min\{\hat{u}_i,\|w_i\|_* d_{\|\cdot\|}(\mathcal{X})\}. \label{eqn: worst-case_relaxation_bound_sdp_intermediate_3}
    \end{equation}
    Substituting \eqref{eqn: worst-case_relaxation_bound_sdp_intermediate_2} and \eqref{eqn: worst-case_relaxation_bound_sdp_intermediate_3} into \eqref{eqn: worst-case_relaxation_bound_sdp_intermediate_1} gives the desired bound.
\end{proof}
\hspace*{\fill}

When the $x$-block $P^*_x$ of the SDP relaxation stays close to the true solution $x^*$, the bound \eqref{eqn: worst-case_relaxation_bound_sdp} shows that the worst-case relaxation error scales with the loosest input bound, i.e., the maximum value amongst the limits $|l_k|$ and $|u_k|$. This fact allows us to choose which coordinate to partition along in order to maximally reduce the relaxation bound on the individual parts of the partition. We state our proposed SDP branching scheme next.

\subsubsection{Proposed Branching Scheme}

We now focus on developing a \hl{coordinate-wise} worst-case optimal branching scheme based on the relaxation bound of Theorem \ref{thm: worst-case_relaxation_bound_sdp}. Similar to the partitioned LP relaxation, the diameter $d_{\|\cdot\|}(\mathcal{X})$ tends to be small in practical settings, making the terms being summed in \eqref{eqn: worst-case_relaxation_bound_sdp} approximately equal to $\relu(c_i)q(l,u)$. We restrict ourselves to this form in order to simplify the subsequent analysis. The objective to optimize therefore takes the form
\begin{equation}
    q(l,u)\sum_{i=1}^{n_z}\relu(c_i), \label{eq: sdp_useable_upper_bound}
\end{equation}
where
\begin{equation*}
    q(l,u) = \max_{k\in\{1,2,\dots,n_x\}} \max\{|l_k|,|u_k|\}.
\end{equation*}
Since the design of the partition amounts to choosing input bounds $l$ and $u$ for the input parts, the input bounds serve as our optimization variables in minimizing the above relaxation bound. By restricting the form of our partition to the uniform division motivated in Theorem \ref{thm: optimal_sdp_partition_rank-1_gap}, it follows from the form of $q$ that the best coordinate to partition along is that with the loosest input bound, i.e., along coordinate $i^* \in \argmax_{k\in\{1,2,\dots,n_x\}} \max\{|l_k|,|u_k|\}$. This observation is formalized below.

\begin{theorem}[\hl{Coordinate-wise} worst-case optimal SDP branching]
\label{thm: optimal_sdp_partition}
\begingroup
\setlength{\emergencystretch}{.1em}
Consider the two-part partitions defined by dividing $\mathcal{X}$ uniformly along the coordinate axes: $\{\mathcal{X}_i^{(1)},\mathcal{X}^{(2)}_i\}$, with $\mathcal{X}_i^{(1)} = \{x\in\mathcal{X} : l_i^{(1)}\le x \le u_i^{(1)} \}$ and $\mathcal{X}_i^{(2)} = \{x\in\mathcal{X} : l_i^{(2)}\le x \le u^{(2)}_i\}$, where $l_i^{(1)}=l$, $u_i^{(1)} = (u_1,u_2,\dots,u_{i-1},\tfrac{1}{2}(l_i+u_i),u_{i+1},\dots,u_{n_x})$, $l_i^{(2)}= (l_1,l_2,\dots,l_{i-1},\tfrac{1}{2}(l_i+u_i),l_{i+1},\dots,l_{n_x})$, and $u_i^{(2)}=u$, for all $i\in\{1,2,\dots,n_x\} \eqqcolon \mathcal{I}$. Let
\begin{equation}
    i^* \in \argmax_{k\in\{1,2,\dots,n_x\}}\max\{|l_k|,|u_k|\}, \label{eq: optimal_sdp_partition}
\end{equation}
and assume that $|l_{i^*}|\ne |u_{i^*}|$. Then the partition $\{\mathcal{X}_{i^*}^{(1)},\mathcal{X}_{i^*}^{(2)}\}$ is optimal in the sense that the upper bound factor $q(l_{i^*}^{(j)},u_{i^*}^{(j)})$ in \eqref{eq: sdp_useable_upper_bound} equals the unpartitioned upper bound $q(l,u)$ on one part $j$ of the partition, is strictly less than $q(l,u)$ on the other part, and $q(l_i^{(j)},u_i^{(j)})=q(l,u)$ for both $j\in\{1,2\}$ for all other $i\notin\argmax_{k\in\{1,2,\dots,n_x\}}\max\{|l_k|,|u_k|\}$.
\endgroup
\end{theorem}
\begin{proof}
First, consider partitioning along coordinate $i\notin\argmax_{k\in\{1,2,\dots,n_x\}}\max\{|l_k|,|u_k|\}$. Then
\begin{gather*}
    q(l_i^{(1)},u_i^{(1)}) = \max_{k\in\{1,2,\dots,n_x\}}\max\{|(l_i^{(1)})_k|,|(u_i^{(1)})_k|\} \\
    = \max\left\{|l_1|,\dots,|l_{n_x}|,|u_1|,\dots,\left|\frac{l_i+u_i}{2}\right|,\dots,|u_{n_x}|\right\} = \max\{|l_{i^*}|,|u_{i^*}|\} = q(l,u),
\end{gather*}
since $\left|\frac{l_i+u_i}{2}\right|\le \frac{|l_i|+|u_i|}{2}<\max\{|l_{i^*}|,|u_{i^*}|\}$ and $i\ne i^*$ implies that
\begin{equation*}
\max\{|l_{i^*}|,|u_{i^*}|\}\in \left\{|l_1|,\dots,|l_{n_x}|,|u_1|,\dots,\left|\frac{l_i+u_i}{2}\right|,\dots,|u_{n_x}|\right\}.
\end{equation*}
In an analogous fashion, it follows that
\begin{equation*}
    q(l_i^{(2)},u_i^{(2)}) = q(l,u).
\end{equation*}

Now, consider partitioning along coordinate $i^*$. Note that either
\begin{equation*}
	\max_{k\in\{1,2,\dots,n_x\}}\max\{|l_k|,|u_k|\} = |l_{i^*}|,~\text{or}~\max_{k\in\{1,2,\dots,n_x\}}\max\{|l_k|,|u_k|\} = |u_{i^*}|.
\end{equation*}
Suppose that the first case holds true. Then
\begin{gather*}
    q(l_{i^*}^{(1)},u_{i^*}^{(1)}) = \max_{k\in\{1,2,\dots,n_x\}}\max\{|(l_{i^*}^{(1)})_k|,|(u_{i^*}^{(1)})_k|\} \\
    = \max\left\{ |l_1|,\dots,|l_{n_x}|,|u_1|,\dots,\left|\frac{l_{i^*}+u_{i^*}}{2}\right|,\dots,|u_{n_x}| \right\}  = |l_{i^*}| = q(l,u),
\end{gather*}
since $|(l_{i^*}+u_{i^*})/2| \le (|l_i^*|+|u_i^*|)/2 < |l_i^*|$ and
\begin{equation*}
    |l_{i^*}| \in \max\left\{ |l_1|,\dots,|l_{n_x}|,|u_1|,\dots,\left|\frac{l_{i^*}+u_{i^*}}{2}\right|,\dots,|u_{n_x}| \right\}.
\end{equation*}
Over the second part of the partition,
\begin{gather*}
    q(l_{i^*}^{(2)},u_{i^*}^{(2)}) = \max_{k\in\{1,2,\dots,n_x\}}\max\{|(l_{i^*}^{(2)})_k|,|(u_{i^*}^{(2)})_k|\} \\
    = \max\left\{ |l_1|,\dots,\left|\frac{l_{i^*}+u_{i^*}}{2}\right|,\dots,|l_{n_x}|,|u_1|,\dots,|u_{n_x}| \right\} < |l_{i^*}| = q(l,u),
\end{gather*}
since $|(l_{i^*}+u_{i^*})/2|< |l_{i^*}|$ and
\begin{equation*}
    |l_i^*| \notin \left\{ |l_1|,\dots,\left|\frac{l_{i^*}+u_{i^*}}{2}\right|,\dots,|l_{n_x}|,|u_1|,\dots,|u_{n_x}| \right\}
\end{equation*}
since $|l_{i^*}|\ne |u_{i^*}|$. In the other case that $\max_{k\in\{1,2,\dots,n_x\}}\max\{|l_k|,|u_k|\} = |u_{i^*}|$, it follows via the same argument that $q(l_{i^*}^{(1)},u_{i^*}^{(1)}) < q(l,u)$ and $q(l_{i^*}^{(2)},u_{i^*}^{(2)}) = q(l,u)$. Since partitioning along any other coordinate $i\notin\argmax_{k\in\{1,2,\dots,n_x\}}\max\{|l_k|,|u_k|\}$ was shown to yield $q(l_i^{(1)},u_i^{(1)})=q(l_i^{(2)},u_i^{(2)})=q(l,u)$, the coordinate $i^*$ is optimal in the sense proposed.
\end{proof}
\hspace*{\fill}

Intuitively, the branching scheme defined in Theorem \ref{thm: optimal_sdp_partition} is \hl{the worst-case optimal coordinate-wise method}, because any other uniform partition along a coordinate axis cannot tighten the relaxation error bound \eqref{eq: sdp_useable_upper_bound}. On the other hand, Theorem \ref{thm: optimal_sdp_partition} guarantees that using the partition coordinate in \eqref{eq: optimal_sdp_partition} results in a strict tightening of the relaxation error upper bound factor $q$ on at least one part of the partition.

\hl{
As was the case with our LP branching scheme, we emphasize that our theoretical results for the SDP (namely, Theorems~\ref{thm: optimal_sdp_partition_rank-1_gap},~\ref{thm: worst-case_relaxation_bound_sdp}, and~\ref{thm: optimal_sdp_partition}) are proven under the assumption of a single hidden layer. However, we may extend our SDP branching scheme \eqref{eq: optimal_sdp_partition} into a branching heuristic for deep multi-layer networks, as we did with the LP branching scheme. Unlike the LP, where a surrogate ``$c$''-vector was defined in order to apply our LP branching score to neurons in layers other than the last, our SDP branching score values can be directly computed in the multi-layer setting at neuron $i$ in layer $k$ as $\max\{|l_i^{[k]}|,|u_i^{[k]}|\}$, since such scores only depend on the $k^\text{th}$-layer activation bounds $l^{[k]},u^{[k]}$ satisfying $l^{[k]} \le x^{[k]} \le u^{[k]}$ for all $x\in\mathcal{X}$.
}

\section{Implementing the Branching Schemes}
\label{sec: bab}

For the reader's convenience, we give a pseudocode in Algorithm \ref{alg: bab} to embed our branching schemes (Theorems \ref{thm: optimal_partition} and \ref{thm: optimal_sdp_partition}) into an overall branch-and-bound algorithm, which is what we use to compute the simulation results in Section \ref{sec: simulation_results} below. The pseudocode for the SDP branch-and-bound procedure is the same as in Algorithm \ref{alg: bab} albeit with lines $2$ and $3$ modified to use Theorem \ref{thm: optimal_sdp_partition} and \eqref{eq: sdp_relaxation}, respectively. \hl{Solving the two subproblems in line 3 can be done independently, and in a parallel fashion, to enhance computational efficiency and scalability.}

\begin{algorithm}[tb]
   \caption{LP branch-and-bound procedure for certification}
   \label{alg: bab}
   \textbf{Input:} $f$, $\mathcal{X}$, $n_\text{branches}$
\begin{algorithmic}[1]
    \STATE \textbf{for} $n=1,\dots,n_\text{branches}$
    \STATE \hspace*{\algorithmicindent} \textbf{compute} \hl{$\mathcal{X}_{i^*}^{(1)},\mathcal{X}_{i^*}^{(2)}$} according to Theorem \ref{thm: optimal_partition}
    \STATE \hspace*{\algorithmicindent} \textbf{compute} \hl{$\hl{\hat{\phi}_{\textup{LP}}}(\mathcal{X}_{i^*}^{(1)}),\hl{\hat{\phi}_{\textup{LP}}}(\mathcal{X}_{i^*}^{(2)})$} according to \eqref{eq: lp_relaxation}
    \STATE \hspace*{\algorithmicindent} \textbf{if} \hl{$\hl{\hat{\phi}_{\textup{LP}}^*}(\mathcal{X}_{i^*}^{(1)}) \ge \hl{\hat{\phi}_{\textup{LP}}^*}(\mathcal{X}_{i^*}^{(2)})$}
    \STATE \hspace*{2\algorithmicindent} \textbf{assign} $\mathcal{X} \gets \mathcal{X}_{i^*}^{(1)}$
    \STATE \hspace*{\algorithmicindent} \textbf{else}
    \STATE \hspace*{2\algorithmicindent} \textbf{assign} $\mathcal{X} \gets \mathcal{X}_{i^*}^{(2)}$
    \STATE \textbf{if} $\max_{j\in\{1,2\}} \hl{\hat{\phi}_{\textup{LP}}^*}(\mathcal{X}_{i^*}^{(j)}) \le 0$
    \STATE \hspace*{\algorithmicindent} \textbf{assign} \texttt{is\textunderscore certified} $\gets$ \texttt{True}
    \STATE \textbf{else}
    \STATE \hspace*{\algorithmicindent} \textbf{assign} \texttt{is\textunderscore certified} $\gets$ \texttt{False}
    \STATE \textbf{return} \texttt{is\textunderscore certified}
\end{algorithmic}
\end{algorithm}

\section{Simulation Results}
\label{sec: simulation_results}

In this section, we experimentally corroborate the effectiveness of our proposed certification methods. We first perform \hl{single-hidden layer} LP branch-and-bound on moderately-sized benchmark datasets and find that our derived branching scheme beats the current state-of-the-art. We then compute the SDP and branched SDP, and compare to the LP results. Next, we explore the effectiveness of branching on the LP and SDP as networks grow in size, namely, as the number of inputs and the number of layers independently increase. \hl{Finally, we compare our multi-layer LP branching heuristic to the state-of-the-art on large-scale certification problems from the $\alpha,\beta$-CROWN benchmarks and VNN competition benchmarks \citep{wang2021beta,bak2021second}.} The experiments in Sections~\ref{sec: partitioned_lp_results},~\ref{sec: partitioned_sdp_results}, and~\ref{sec: effectiveness_as_network_grows} are performed on a standard laptop computer using Tensorflow 2.5 in Python 3.9. Training of networks is done using the Adam optimizer \citep{kingma2014adam}, and certifications are performed using MOSEK in CVXPY \citep{diamond2016cvxpy}. \hl{The experiments in Section~\ref{sec: deep_nn_results} are conducted on a desktop computer equipped with an Nvidia 4080 GPU, using PyTorch 2.0.1 and Python 3.11.}

\subsection{\hl{Single-Hidden Layer} LP Results}
\label{sec: partitioned_lp_results}

In this experiment, we consider classification networks trained on three datasets: the Wisconsin breast cancer diagnosis dataset with $(n_x,n_z)=(30,2)$ \citep{dua_2017}, the MNIST handwritten digit dataset with $(n_x,n_z)=(784,10)$ \citep{lecun1998mnist}, and the CIFAR-10 image classification dataset with $(n_x,n_z)=(3072,10)$ \citep{krizhevsky2009learning}. The Wisconsin breast cancer dataset is characteristic of a real-world machine learning setting in which robustness guarantees are crucial for safety; a misdiagnosis of breast cancer may result in grave consequences for the patient under concern. Each neural network is composed of an affine layer followed by a $\relu$ hidden layer followed by another affine layer. For each network, we consider $15$ different nominal inputs $\bar{x}$ and corresponding uncertainty sets $\mathcal{X} = \{x\in\mathbb{R}^{n_x} : \|x-\bar{x}\|_\infty \le \epsilon\}$, where we choose a range of attack radii $\epsilon$. We also perform a sweep over the number of branching steps to use in the branch-and-bound certification scheme. Recall that a negative optimal objective value of the robustness certification problem proves that no perturbation of $\bar{x}$ within $\mathcal{X}$ results in misclassification.

\hl{Figures~\ref{fig: wisconsin},~\ref{fig: mnist}, and~\ref{fig: cifar} display the percent of test inputs certified using the LP relaxation with our branching method, as well as those certified using the state-of-the-art LP branching method, filtered smart branching (FSB) \citep{de2021improved}. FSB was used in the $\alpha,\beta$-CROWN branch-and-bound scheme to win the 2021, 2022, and 2023 VNN-COMP certification competitions \citep{wang2021beta,bak2021second,muller2022third,brix2023fourth}. We see that, for all three benchmarks and at all considered attack radii $\epsilon$ and number of branching steps, our LP branching method meets or exceeds the performance of FSB; our method attains higher certification percentages in fewer branching steps and at larger radii than FSB. At some radius-number of branches settings, we even see that our LP branching scheme substantially increases the percentage certified by approximately $10\%$ (Wisconsin), $20\%$ (MNIST), and $20\%$ (CIFAR-10).} 

\begin{figure}[tbh]
    \setlength\fwidth{0.49\linewidth}
    \centering
    \hl{
    \subfloat[Attack radius $\epsilon=0.2$.]{%
		\centering
		\includegraphics[width=\fwidth]{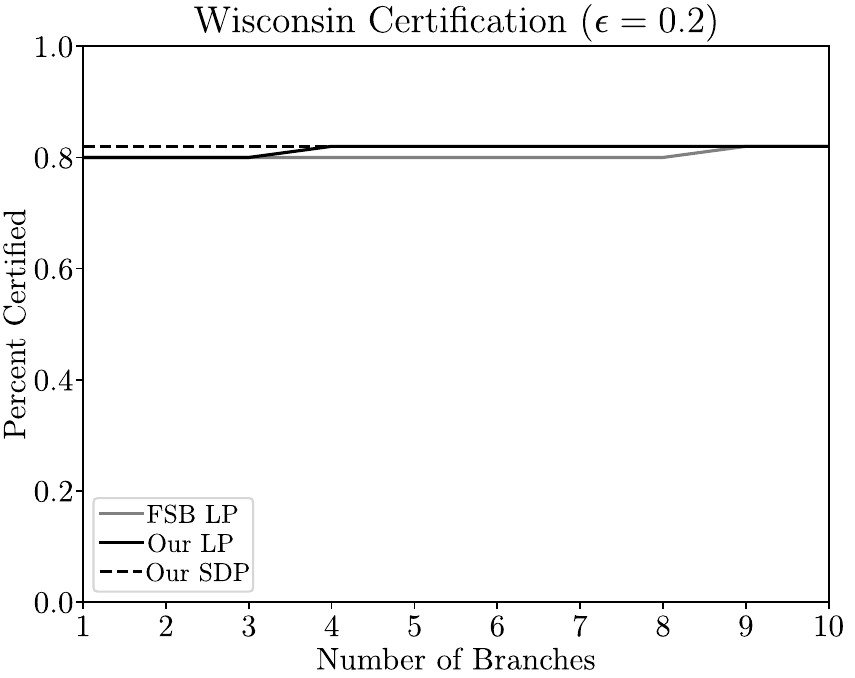}
		\label{fig: wisconsin-1}
	}
	\hfil
	\subfloat[Attack radius $\epsilon=0.4$.]{%
		\centering
		\includegraphics[width=\fwidth]{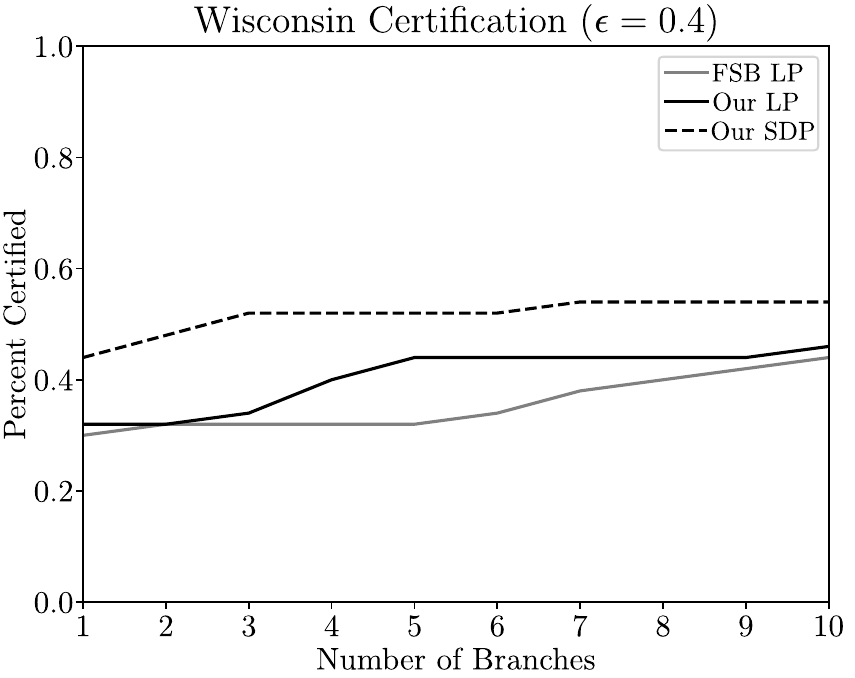}
		\label{fig: wisconsin-2}
	}\\
    \subfloat[Attack radius $\epsilon=0.6$.]{%
		\centering
		\includegraphics[width=\fwidth]{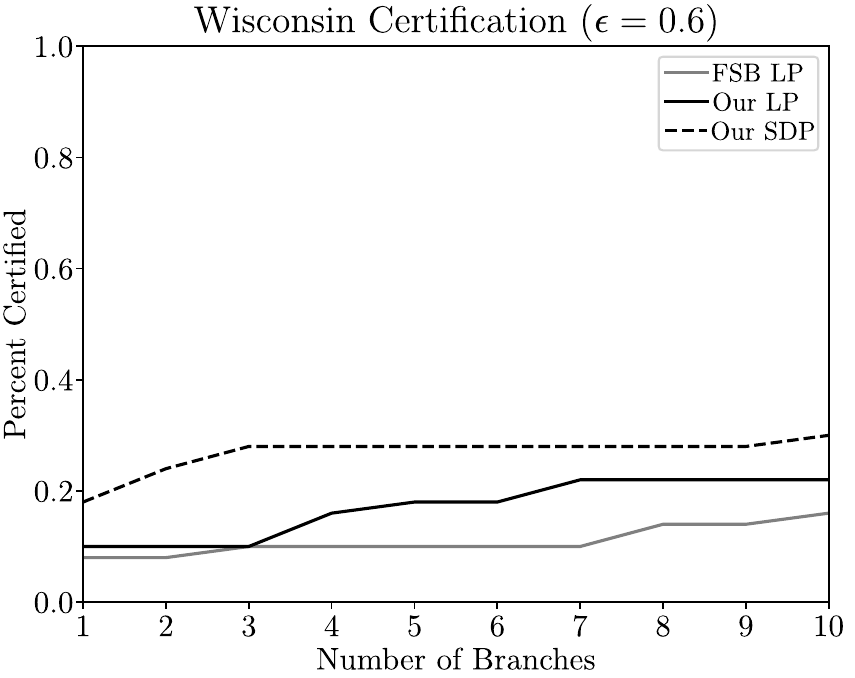}
		\label{fig: wisconsin-3}
	}
        \hfil
        \subfloat[Attack radius $\epsilon=0.8$.]{%
		\centering
		\includegraphics[width=\fwidth]{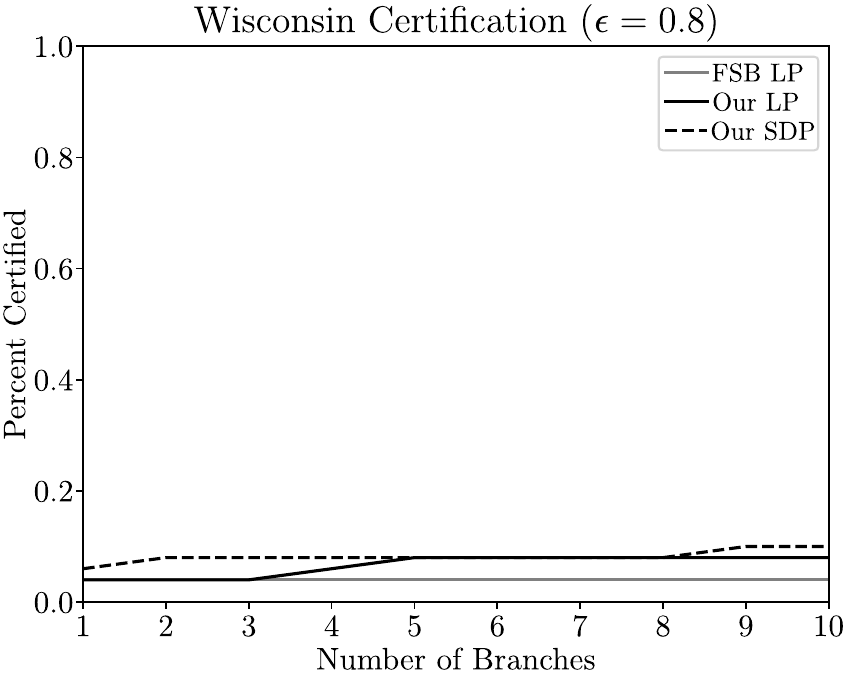}
		\label{fig: wisconsin-4}
	}
    \caption{Percent certified on Wisconsin breast cancer diagnosis dataset using branching on LP and SDP relaxations \hl{of single-hidden layer network}.}
    }
    \label{fig: wisconsin}
\end{figure}

\begin{figure}[tbh]
    \setlength\fwidth{0.49\linewidth}
    \centering
    \hl{
    \subfloat[Attack radius $\epsilon=0.015$.]{%
		\centering
		\includegraphics[width=\fwidth]{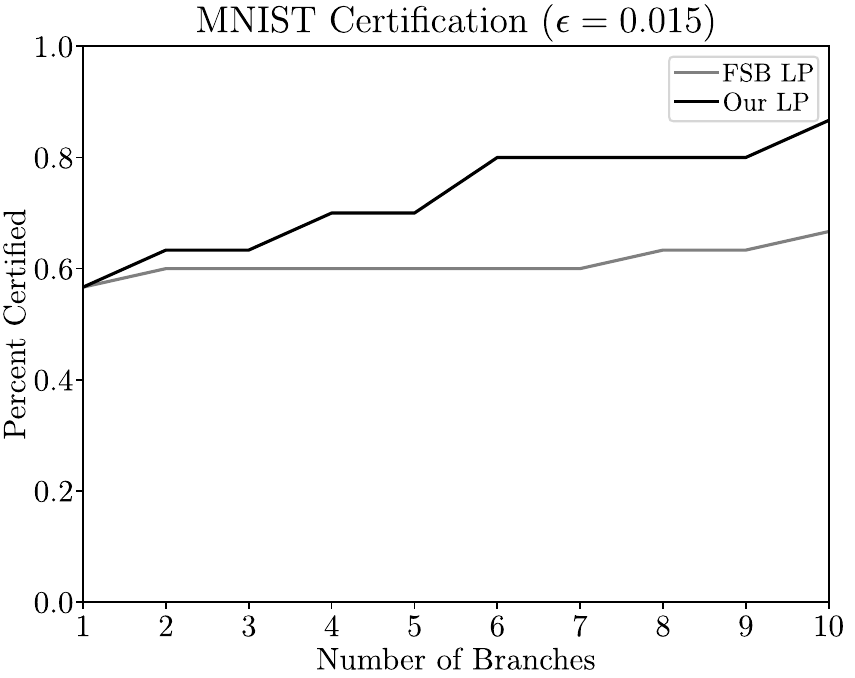}
		\label{fig: mnist-1}
	}
	\hfil
	\subfloat[Attack radius $\epsilon=0.02$.]{%
		\centering
		\includegraphics[width=\fwidth]{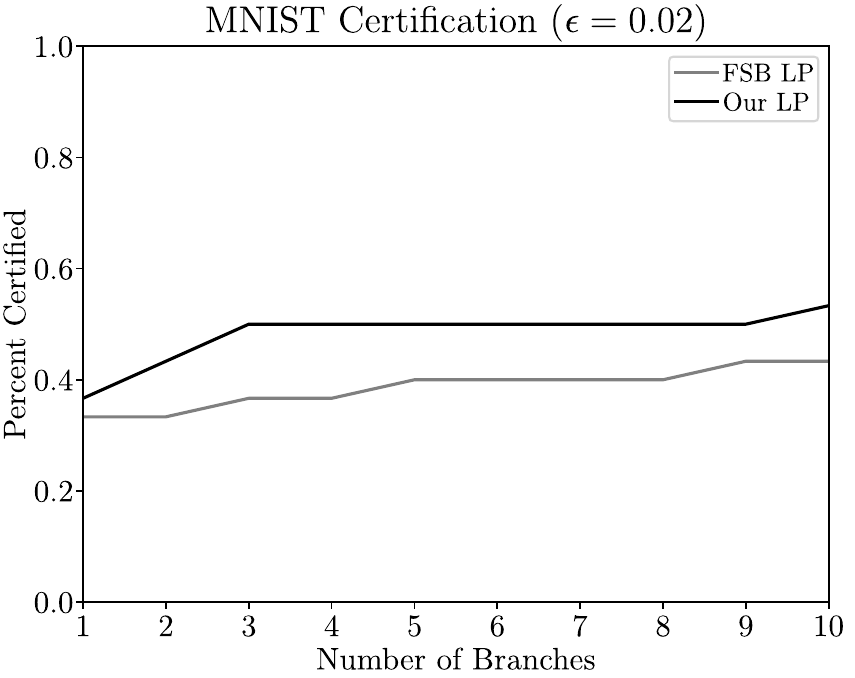}
		\label{fig: mnist-2}
	}\\
    \subfloat[Attack radius $\epsilon=0.025$.]{%
		\centering
		\includegraphics[width=\fwidth]{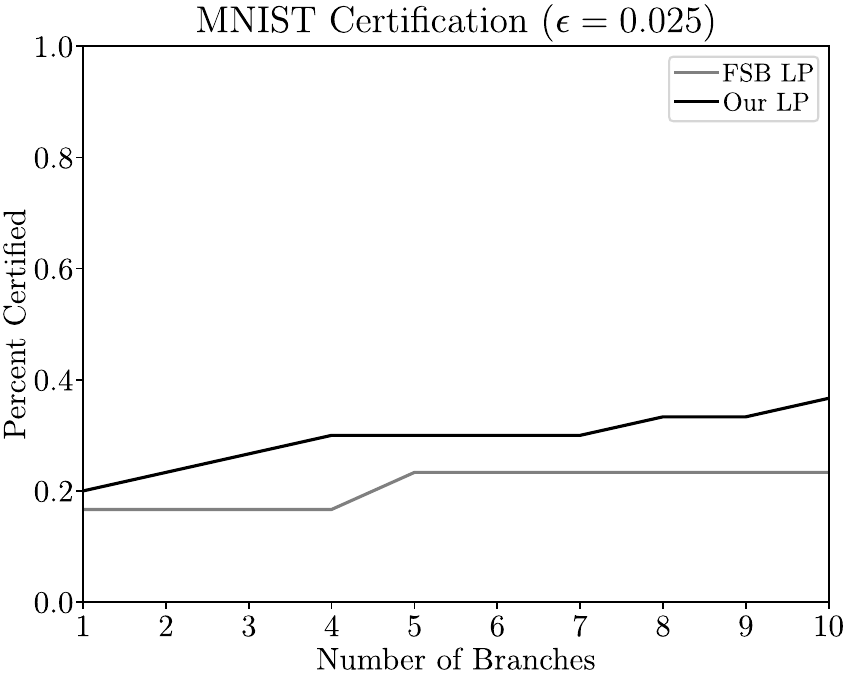}
		\label{fig: mnist-3}
	}
        \hfil
        \subfloat[Attack radius $\epsilon=0.03$.]{%
		\centering
		\includegraphics[width=\fwidth]{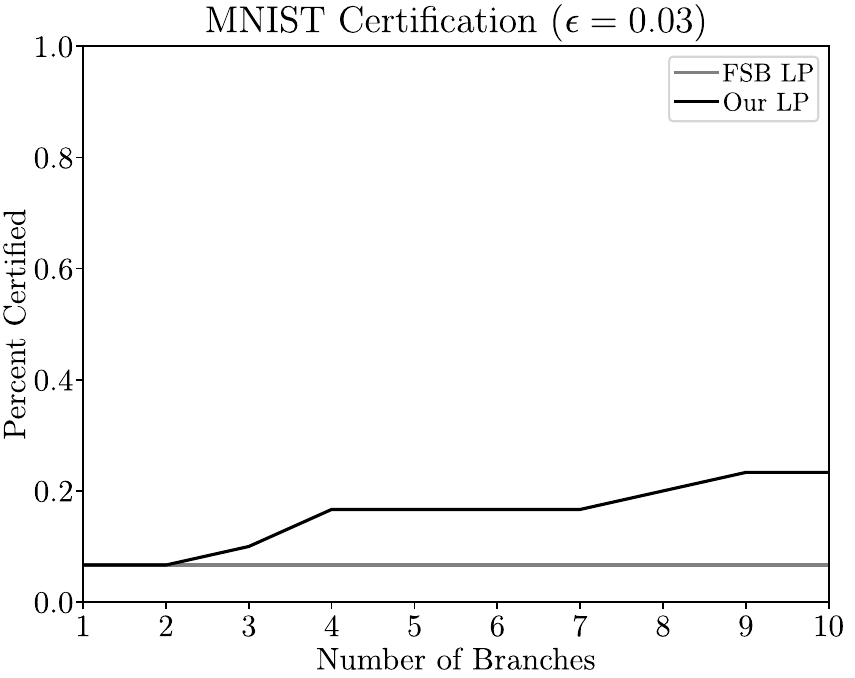}
		\label{fig: mnist-4}
	}
    \caption{Percent certified on MNIST dataset using branching on LP relaxation \hl{of single-hidden layer network}.}
    }
    \label{fig: mnist}
\end{figure}

\begin{figure}[tbh]
    \setlength\fwidth{0.49\linewidth}
    \centering
    \hl{
    \subfloat[Attack radius $\epsilon=0.015$.]{%
		\centering
		\includegraphics[width=\fwidth]{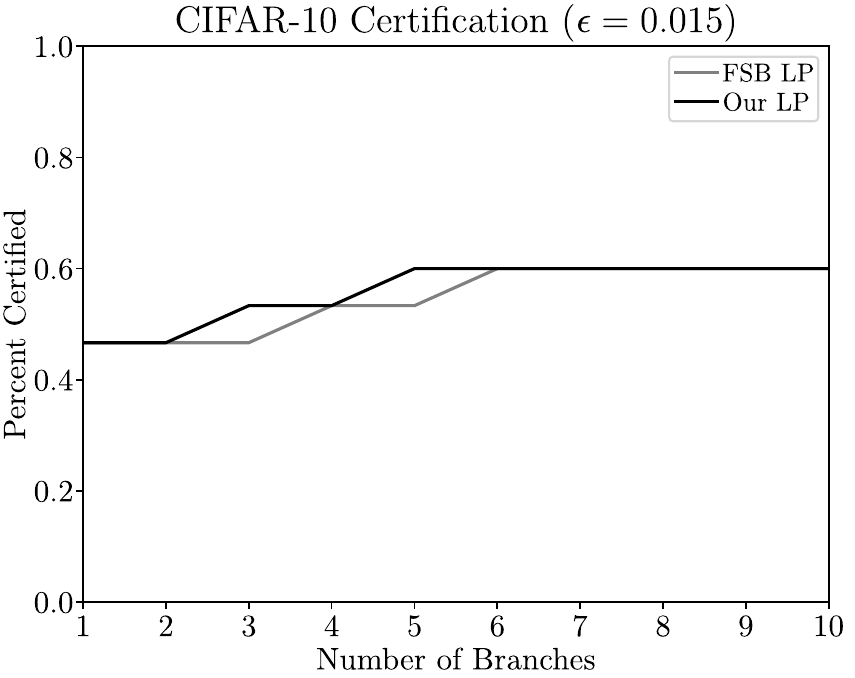}
		\label{fig: cifar-1}
	}
	\hfil
	\subfloat[Attack radius $\epsilon=0.0175$.]{%
		\centering
		\includegraphics[width=\fwidth]{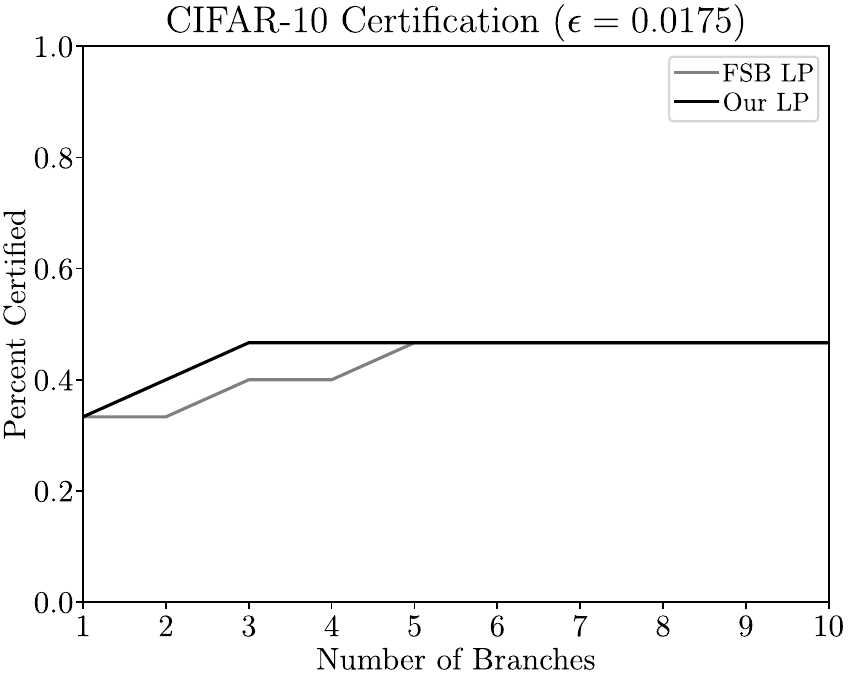}
		\label{fig: cifar-2}
	}\\
    \subfloat[Attack radius $\epsilon=0.02$.]{%
		\centering
		\includegraphics[width=\fwidth]{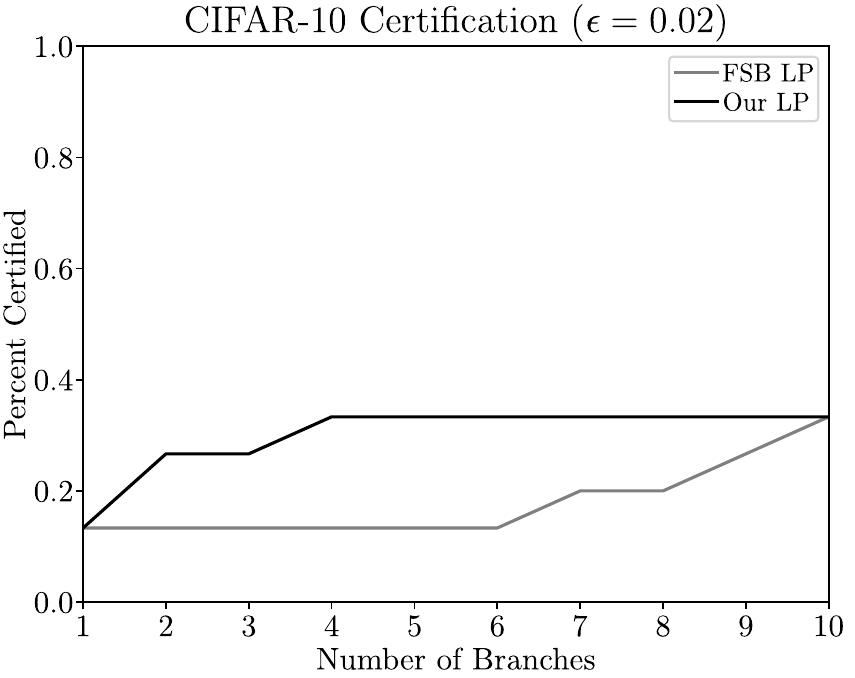}
		\label{fig: cifar-3}
	}
        \hfil
        \subfloat[Attack radius $\epsilon=0.0225$.]{%
		\centering
		\includegraphics[width=\fwidth]{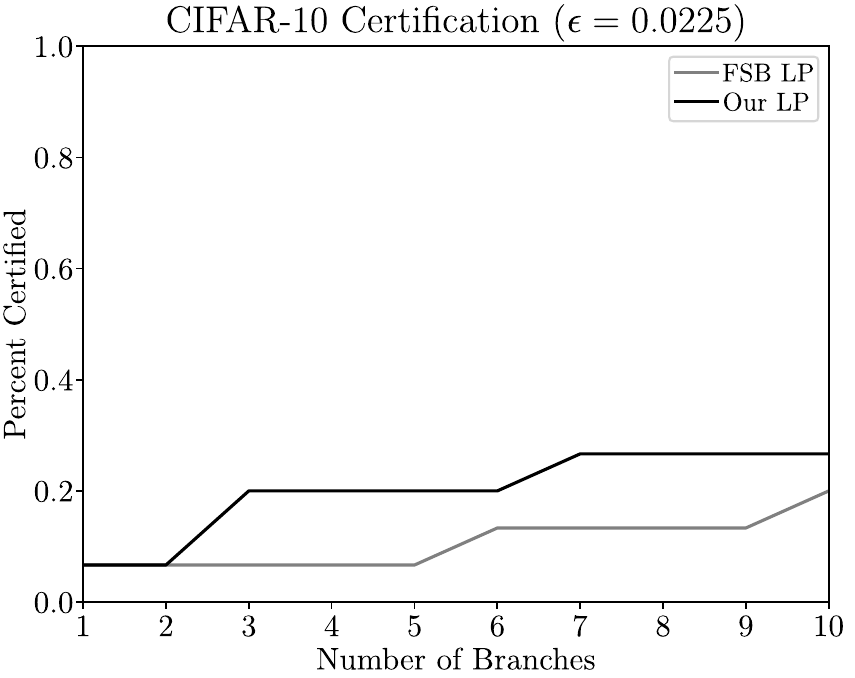}
		\label{fig: cifar-4}
	}
    \caption{Percent certified on CIFAR-10 dataset using branching on LP relaxation \hl{of single-hidden layer network}.}
    }
    \label{fig: cifar}
\end{figure}

\subsection{\hl{Single-Hidden Layer} SDP Results}
\label{sec: partitioned_sdp_results}
In this experiment, we consider the same single-hidden layer $\relu$ neural network trained on the Wisconsin breast cancer dataset as used in Section \ref{sec: partitioned_lp_results}. We solve the branched SDP using our proposed branching scheme from Theorem \ref{thm: optimal_sdp_partition}, where the same test data and simulation parameters are used as with the LP in Section \ref{sec: partitioned_lp_results}. The results are shown in Figure~\ref{fig: wisconsin}. In comparing the SDP results to the LP results in Figure~\ref{fig: wisconsin}, we see that our branched SDP achieves better certification percentages compared to both of the branched LP approaches, up to around $20\%$ better in some radius-number of branches settings. We remark that \citet{zhang2020tightness} provides theoretical guarantees for the tightness of the SDP relaxation under some technical conditions. However, since we find a strict increase in the certified percentages upon branching on the SDP, we conclude that such conditions for exactness may not be satisfied in general by practical networks, indicating that branching on the SDP may in fact be necessary in settings where high accuracy is the primary concern at hand. As we will see in the next experiment, this certification enhancement via SDP branching becomes even more substantial as the network depth increases. We now move to this experiment, and propose a general rule of thumb for when LP, SDP, and their branch-and-bound variants are best applied based on depth and width of the network.

\subsection{Effectiveness as Network Grows}
\label{sec: effectiveness_as_network_grows}

In this section, we perform two experiments to test the effectiveness of branching as the size and structure of the network changes. First, we consider two-layer networks of structure $n_x\times 100\times 5$, where $n_x$ is the input dimension. For each input size $n_x\in\{5,10,20,40,80,100\}$, we generate one network with standard normal random weights, and another network with uniformly distributed weights (where each element is distributed uniformly on the interval $[0,1]$). The weights are normalized according to Assumption \ref{ass: normalized_rows}. For each network being tested, we compute the LP, branched LP, SDP, and branched SDP relaxations at a fixed nominal input $\bar{x}$ using the input uncertainty set $\mathcal{X} = \{x\in\mathbb{R}^{n_x}:\|x-\bar{x}\|_\infty \le \epsilon\}$ with $\epsilon=0.5$. The optimal values, corresponding computation times, and percentage improvements induced by branching are reported in Table \ref{tab: vary_input_size}. The effectiveness of branching for the LP remains relatively constant between $5$ and $10$ percent improvement, whereas the branching appears to lose its efficacy on the SDP as the input size grows. As expected, the two-part partitioned convex relaxations take twice as long to solve as their unpartitioned counterparts. Note that, despite the fact that branching works better for the LP with wide networks, the actual optimal value of the SDP-based certificates are always lower (tighter) than the LP-based ones. This matches what is known in the literature: the SDP is a tighter relaxation technique than the LP \citep{raghunathan2018semidefinite}. However, the computation times of the SDP and branched SDP quickly increase as the network size increases, whereas the LP and branched LP computation times are seen to slowly increase. \hl{Indeed, SDP-based certification is known to suffer from scalability issues in high-dimensional settings \citep{chiu2023tight}. This limitation also applies to our branched SDP method as well.} All of this suggests the following: in the regime of shallow (i.e., one or two hidden layers) but very wide networks, the branched LP should be used, since the branching remains effective in tightening the relaxation, yet the method is scalable to large networks where the SDP cannot be feasibly applied.

\begin{table}[tbh]
	\centering
	\caption{Varying input size $n_x$ for $n_x\times 100 \times 5$ ReLU network. Optimal values and corresponding computation times reported. B-LP and B-SDP correspond to branched LP and branched SDP, respectively. \%-LP and \%-SDP represent the percentage tightening of the optimal values obtained from branching.}
	\captionsetup[subfloat]{position=top} 
	\subfloat[Normally distributed network weights.]{%
	\begin{tabular}{l r r r r r r}
	\toprule
	Input size & LP & B-LP & \%-LP & SDP & B-SDP & \%-SDP \\
	\midrule %
	\multirow{2}{*}{$5$} & $126.93$ & $117.92$ & $\mathbf{7.10\%}$ & $16.82$ & $14.83$ & $\mathbf{11.85\%}$ \\
    & $\SI{0.71}{\second}$ & $\SI{1.46}{\second}$ & $104.18\%$ & $\SI{1.66}{\second}$ & $\SI{3.33}{\second}$ & $101.33\%$ \\ %
	\multirow{2}{*}{$10$} & $187.57$ & $176.19$ & $\mathbf{6.07\%}$ & $33.62$ & $32.96$ & $\mathbf{1.98\%}$ \\
    & $\SI{0.77}{\second}$ & $\SI{1.36}{\second}$ & $76.13\%$ & $\SI{1.54}{\second}$ & $\SI{3.16}{\second}$ & $105.57\%$ \\ %
	\multirow{2}{*}{$20$} & $386.49$ & $364.53$ & $\mathbf{5.68\%}$ & $54.02$ & $54.01$ & $\mathbf{0.02\%}$ \\
    & $\SI{0.71}{\second}$ & $\SI{1.42}{\second}$ & $100.49\%$ & $\SI{1.85}{\second}$ & $\SI{4.31}{\second}$ & $132.94\%$ \\ %
	\multirow{2}{*}{$40$} & $874.70$ & $864.56$ & $\mathbf{1.16\%}$ & $104.90$ & $104.38$ & $\mathbf{0.49\%}$ \\
    & $\SI{1.27}{\second}$ & $\SI{2.68}{\second}$ & $110.93\%$ & $\SI{4.79}{\second}$ & $\SI{9.33}{\second}$ & $95.01\%$ \\ %
	\multirow{2}{*}{$80$} & $1591.41$ & $1496.23$ & $\mathbf{5.98\%}$ & $310.37$ & $310.31$ & $\mathbf{0.02\%}$ \\
    & $\SI{1.76}{\second}$ & $\SI{2.97}{\second}$ & $69.00\%$ & $\SI{9.81}{\second}$ & $\SI{17.87}{\second}$ & $82.11\%$ \\ %
	\multirow{2}{*}{$100$} & $2184.94$ & $2175.87$ & $\mathbf{0.42\%}$ & $383.63$ & $383.50$ & $\mathbf{0.03\%}$ \\
    & $\SI{0.78}{\second}$ & $\SI{1.84}{\second}$ & $136.93\%$ & $\SI{5.02}{\second}$ & $\SI{10.52}{\second}$ & $109.46\%$ \\ %
	\bottomrule
	\end{tabular}
	\label{tab: vary_input_size_normal}%
	} \\
	\vspace*{\baselineskip}%
	\subfloat[Uniformly distributed network weights.]{%
	\begin{tabular}{l r r r r r r}
	\toprule
	Input size & LP & B-LP & \%-LP & SDP & B-SDP & \%-SDP \\
	\midrule %
	\multirow{2}{*}{$5$} & $11.65$ & $10.69$ & $\mathbf{8.31\%}$ & $5.95$ & $5.74$ & $\mathbf{3.44\%}$ \\
    & $\SI{0.65}{\second}$ & $\SI{1.36}{\second}$ & $109.54\%$ & $\SI{1.39}{\second}$ & $\SI{2.20}{\second}$ & $58.32\%$ \\ %
	\multirow{2}{*}{$10$} & $34.13$ & $34.13$ & $\mathbf{0.00\%}$ & $12.61$ & $11.92$ & $\mathbf{5.47\%}$ \\
    & $\SI{0.68}{\second}$ & $\SI{1.36}{\second}$ & $101.35\%$ & $\SI{1.32}{\second}$ & $\SI{2.48}{\second}$ & $87.45\%$ \\ %
	\multirow{2}{*}{$20$} & $83.74$ & $83.02$ & $\mathbf{0.86\%}$ & $19.20$ & $19.00$ & $\mathbf{1.06\%}$ \\
    & $\SI{0.67}{\second}$ & $\SI{1.40}{\second}$ & $106.88\%$ & $\SI{1.31}{\second}$ & $\SI{2.85}{\second}$ & $118.39\%$ \\ %
	\multirow{2}{*}{$40$} & $141.37$ & $133.30$ & $\mathbf{5.71\%}$ & $25.89$ & $25.69$ & $\mathbf{0.74\%}$ \\
    & $\SI{0.69}{\second}$ & $\SI{1.43}{\second}$ & $106.67\%$ & $\SI{1.63}{\second}$ & $\SI{3.23}{\second}$ & $97.62\%$ \\ %
	\multirow{2}{*}{$80$} & $260.80$ & $242.19$ & $\mathbf{7.14\%}$ & $21.86$ & $21.68$ & $\mathbf{0.84\%}$ \\
    & $\SI{0.71}{\second}$ & $\SI{1.42}{\second}$ & $99.25\%$ & $\SI{2.82}{\second}$ & $\SI{5.44}{\second}$ & $92.97\%$ \\ %
	\multirow{2}{*}{$100$} & $400.73$ & $387.24$ & $\mathbf{3.37\%}$ & $102.87$ & $102.35$ & $\mathbf{0.51\%}$ \\
    & $\SI{0.74}{\second}$ & $\SI{1.56}{\second}$ & $111.10\%$ & $\SI{3.33}{\second}$ & $\SI{6.89}{\second}$ & $106.64\%$ \\ %
	\bottomrule
	\end{tabular}
	\label{tab: vary_input_size_uniform}%
	}
	\label{tab: vary_input_size}
\end{table}

In the second simulation of this section, we analyze the effectiveness of branching as the depth of the network increases. In particular, we consider networks with normal random weights having $5$ inputs and $5$ outputs, and each intermediate layer having $10$ neurons. We run the experiment on networks having $1$ through $6$ such intermediate layers. \hl{Recall that our LP branching scheme was derived for single-hidden layer networks. Therefore, we heuristically extend our LP branching scheme to multi-layer networks using the technique proposed in Algorithm~\ref{alg: multi-layer}. Unlike the branched LP, the SDP branching scheme given in Theorem \ref{thm: optimal_sdp_partition} can directly be applied to deep networks, without the need to use surrogate parameters or other intermediate steps to compute the partition.} We compute the LP, branched LP, SDP, and branched SDP on the networks at hand and report the objective values and computation times in Table \ref{tab: vary_num_layers}. In this simulation, we see a stark contrast to the results in Table \ref{tab: vary_input_size}. Specifically, the percentage improvement induced by branching on the LP reduces quickly to nearly zero percent for networks with 3 or more intermediate $10$-neuron hidden layers. Indeed, this is one fundamental drawback behind the LP relaxation: the convex upper envelope is used independently at every neuron, so the relaxation error quickly compounds as the network becomes deeper. On the other hand, the SDP relaxation takes into account the coupling between the layers of the network. This theoretical advantage is demonstrated empirically, as the percentage improvement gained by the branched SDP hovers around $10\%$ even for the deep networks tested here. Moreover, note how the SDP computation time remains relatively close to that of the LP, unlike the rapid increase in computation time seen when increasing the input size. This behavior suggests the following: in the regime of deep but relatively narrow networks, the branched SDP should be used, since the branching is effective in tightening the relaxation, yet the computational cost grows relatively slowly as more layers are added (compared to the case where more inputs are added).

\begin{table}[tbh]
	\centering
	\caption{Varying number of hidden layers for a $5\times 10\times10\times\cdots\times 10 \times 5$ ReLU network with normal random weights. Optimal values and corresponding computation times reported. B-LP and B-SDP correspond to branched LP and branched SDP, respectively. \%-LP and \%-SDP represent the percentage tightening of the optimal values obtained from branching.}
	\begin{tabular}{l r r r r r r}
	\toprule
	Layers & LP & B-LP & \%-LP & SDP & B-SDP & \%-SDP \\
	\midrule %
	\multirow{2}{*}{$1$} & $10.16$ & $7.03$ & $\mathbf{30.79\%}$ & $4.70$ & $4.65$ & $\mathbf{1.12\%}$ \\
& $\SI{0.59}{\second}$ & $\SI{1.21}{\second}$ & $105.06\%$ & $\SI{0.68}{\second}$ & $\SI{1.29}{\second}$ & $91.17\%$ \\ %
	\multirow{2}{*}{$2$} & $46.29$ & $44.89$ & $\mathbf{3.03\%}$ & $2.42$ & $1.94$ & $\mathbf{19.94\%}$ \\
& $\SI{0.62}{\second}$ & $\SI{1.21}{\second}$ & $93.07\%$ & $\SI{0.71}{\second}$ & $\SI{1.49}{\second}$ & $108.81\%$ \\ %
	\multirow{2}{*}{$3$} & $626.96$ & $626.96$ & $\mathbf{0.00\%}$ & $36.29$ & $34.36$ & $\mathbf{5.31\%}$ \\
& $\SI{0.61}{\second}$ & $\SI{1.29}{\second}$ & $110.86\%$ & $\SI{0.72}{\second}$ & $\SI{1.47}{\second}$ & $103.32\%$ \\ %
	\multirow{2}{*}{$4$} & $5229.32$ & $5229.32$ & $\mathbf{0.00\%}$ & $179.79$ & $167.34$ & $\mathbf{6.93\%}$ \\
& $\SI{0.65}{\second}$ & $\SI{1.29}{\second}$ & $97.47\%$ & $\SI{0.99}{\second}$ & $\SI{1.88}{\second}$ & $89.81\%$ \\ %
	\multirow{2}{*}{$5$} & $37625.91$ & $37625.86$ & $\mathbf{0.00\%}$ & $628.78$ & $561.60$ & $\mathbf{10.68\%}$ \\
& $\SI{0.69}{\second}$ & $\SI{1.34}{\second}$ & $94.59\%$ & $\SI{1.13}{\second}$ & $\SI{2.04}{\second}$ & $80.35\%$ \\ %
	\multirow{2}{*}{$6$} & $326743.55$ & $326743.34$ & $\mathbf{0.00\%}$ & $3245.41$ & $3050.69$ & $\mathbf{6.00\%}$ \\
& $\SI{0.75}{\second}$ & $\SI{1.35}{\second}$ & $79.44\%$ & $\SI{1.19}{\second}$ & $\SI{2.35}{\second}$ & $98.01\%$ \\ %
	\bottomrule
	\end{tabular}
	\label{tab: vary_num_layers}
\end{table}

\hl{
\subsection{Large-Scale, Multi-Layer Certification Benchmarks}
\label{sec: deep_nn_results}
In this section, we compare LP branch-and-bound with our multi-layer branching heuristic (Algorithm~\ref{alg: multi-layer}) to the state-of-the-art deep neural network verifier, $\alpha,\beta$-CROWN with $k$-FSB branching (a slight variant of filtered smart branching, introduced in \citet{de2021improved}). To implement our method, we embed our branching heuristic into the open-source $\alpha,\beta$-CROWN branch-and-bound verifier \citep{wang2021beta}. We consider two image classification certification benchmarks: 1) \texttt{cifar10\_resnet}, a benchmark from the $\alpha,\beta$-CROWN \citep{wang2021beta} repository that was used in VNN-COMP 2021 \citep{bak2021second}, and 2) the \texttt{cifar2020} benchmark from VNN-COMP 2022 \citep{muller2022third}. These large-scale problems involve deep neural networks with multiple layers and complex structures.

Table~\ref{tab: vnn} shows that our method yields performance comparable to that of the baseline $\alpha,\beta$-CROWN with $k$-FSB branching. We remark that the differences in computation time are negligible. Indeed, the implementation of either branching heuristic only requires the computation of the preactivation bounds $l^{[k]},u^{[k]}$ and an evaluation of a simple branching score function at each neuron (with $k$-FSB depending on one additional backwards pass through an auxiliary variable that our method does not require). Overall, our experiments suggest that LP branching scheme outperforms the state-of-the-art on small-to-moderately-sized single-hidden layer certification problems, and achieves comparable performance in large-scale deep neural network certification problems.
}

\begin{table}[tbh]
\label{tab: vnn}
    \centering
    \hl{
    \caption{Our multi-layer branching heuristic versus $k$-FSB branching in $\alpha,\beta$-CROWN branch-and-bound verifier. Benchmarks include \protect\subref{tab: cifar10_resnet} \texttt{cifar10\_resnet}, a residual neural network with 2 residual blocks,  5 convolutional layers, and 2 fully connected affine layers, and \protect\subref{tab: cifar2020} \texttt{cifar2020}, a 4-layer ReLU network with 2 convolutional layers and 2 fully connected affine layers. The reported ``Accuracy'' is the ratio between the number of test inputs that are successfully certified and the total number of test inputs. ``Time'' is the average certification time per test input.}   
    \captionsetup[subfloat]{position=top} 
    \subfloat[\texttt{cifar10\_resnet}]{
    \label{tab: cifar10_resnet}
    \begin{tabular}{l r r r}
         \toprule
         Method & Accuracy & Time (seconds) \\ 
         \midrule
         Our LP & 30/56 & 56.99 \\ 
         $k$-FSB & 30/56 & 53.12 \\ 
         \bottomrule
    \end{tabular}
    }
    \subfloat[\texttt{cifar2020}]{
    \label{tab: cifar2020}
    \begin{tabular}{l r r r}
        \toprule
         Method & Accuracy & Time (seconds) \\
         \midrule
         Our LP & 60/72 & 28.40 \\ 
         $k$-FSB & 61/72 & 33.11 \\ 
         \bottomrule
    \end{tabular}
    }
    }
\end{table}

\section{Conclusions}
\label{sec: conclusions}

In this paper, we propose intelligently designed closed-form branching schemes for linear programming (LP) and semidefinite programming (SDP) robustness certification methods of ReLU neural networks. The branching schemes are derived by minimizing the worst-case error induced by the corresponding convex relaxations \hl{on single-hidden layer networks}, which is theoretically justified by showing that minimizing the true relaxation error is NP-hard. The proposed techniques are experimentally substantiated on benchmark datasets by demonstrating significant reduction in relaxation error \hl{for moderately-sized single-hidden layer models, and performance on par with the state-of-the-art on large-scale, deep models}. 




\acks{This work was supported in part by U.S.\ Army Research Laboratory and the U.S.\ Army Research Office under Grant W911NF2010219, and in part by ONR and NSF.}



\appendix

\section{Proof of Proposition \ref{prop: partitioned_relaxation_bound}}
\label{sec: proof_of_partitioned_relaxation_bound}

\begin{proof}
	Assume that $\hl{\phi^*}(\mathcal{X}) > \max_{j\in\{1,2,\dots,p\}}\hl{\hat{\phi}_{\textup{LP}}^*}(\mathcal{X}^{(j)})$. Then,
	\begin{equation}
	\hl{\phi^*}(\mathcal{X})>\hl{\hat{\phi}_{\textup{LP}}^*}(\mathcal{X}^{(j)}) ~ \text{for all $j\in\{1,2,\dots,p\}$}. \label{eq: partitioned_relaxation_bound_assumption}
	\end{equation}
	Let $(x^*,z^*)$ denote an optimal solution to the unrelaxed problem \eqref{eq: robustness_certification_problem}, i.e., $x^*\in\mathcal{X}$, $z^*=f(x^*)$, and
	\begin{equation}
	c^\top z^* = \hl{\phi^*}(\mathcal{X}). \label{eq: partitioned_relaxation_bound_optimality}
	\end{equation}
	Since $\cup_{j=1}^p\mathcal{X}^{(j)} = \mathcal{X}$, there exists $j^*\in\{1,2,\dots,p\}$ such that $x^*\in\mathcal{X}^{(j^*)}$. Since $x^*\in\mathcal{X}^{(j^*)}$ and $z^*=f(x^*)$, it holds that $(x^*,z^*)\in\mathcal{N}_\textup{LP}^{(j^*)}$, where $\mathcal{N}_\textup{LP}^{(j^*)}$ is the relaxed network constraint set defined by $\mathcal{X}^{(j^*)}$. Therefore, 
	\begin{equation*}
	c^\top z^* \le \sup\{c^\top z : x\in\mathcal{X}^{(j^*)}, ~ (x,z)\in\mathcal{N}_\textup{LP}^{(j^*)}\} = \hl{\hat{\phi}_{\textup{LP}}^*}(\mathcal{X}^{(j^*)}) < \hl{\phi^*}(\mathcal{X}),
	\end{equation*}
	where the first inequality comes from the feasibility of $(x^*,z^*)$ over the $j^{*\text{th}}$ subproblem and the final inequality is due to \eqref{eq: partitioned_relaxation_bound_assumption}. This contradicts the optimality of $(x^*,z^*)$ given in \eqref{eq: partitioned_relaxation_bound_optimality}. Hence, \eqref{eq: partitioned_relaxation_bound} must hold.
\end{proof}
\hspace*{\fill}

\section{Proof of Proposition \ref{prop: improving_the_lp_relaxation_bound}}
\label{sec: proof_of_partition_improvement_lp}

\begin{proof}
	Let $j\in\{1,2,\dots,p\}$. It will be shown that $\mathcal{N}_\textup{LP}^{(j)} \subseteq \mathcal{N}_\textup{LP}$. Let $(x,z)\in\mathcal{N}_\textup{LP}^{(j)}$. Define $u'=u^{(j)}$, $l'=l^{(j)}$, and
	\begin{align*}
	g(x) ={}& u\odot(Wx-l)\oslash(u-l), \\
	g'(x) ={}& u'\odot (Wx-l')\oslash(u'-l').
	\end{align*}
	Then, by letting $\Delta g(x) = g(x) - g'(x) = a\odot (Wx) + b$, where
	\begin{align*}
	a ={}& u\oslash(u-l) - u'\oslash(u'-l'), \\
	b ={}& u'\odot l'\oslash(u'-l') - u\odot l\oslash(u-l),
	\end{align*}
	the following relations are derived for all $i\in\{1,2,\dots,n_z\}$:
	\begin{gather*}
	g^*_i \coloneqq \inf_{\{x : l'\le Wx\le u'\}} (\Delta g(x))_i	\ge \inf_{\{\hat{z} : l'\le \hat{z}\le u'\}} (a\odot \hat{z} + b)_i \\
	= \inf_{\{\hat{z}_i : l_i' \le \hat{z}_i \le u'_i\}} (a_i \hat{z}_i + b_i) =
	\begin{aligned}
	\begin{cases}
	a_i l_i' + b_i & \text{if $a_i\ge 0$}, \\
	a_i u_i' + b_i & \text{if $a_i< 0$}.
	\end{cases}
	\end{aligned}
	\end{gather*}
	In the case that $a_i \ge 0$, we have that
	\begin{align*}
	g_i^* \ge{}& a_i l_i' + b_i = \left(\frac{u_i}{u_i-l_i}-\frac{u_i'}{u_i'-l_i'}\right)l_i' + \left(\frac{u_i'l_i'}{u_i'-l_i'}-\frac{u_i l_i}{u_i-l_i}\right) = \frac{u_i}{u_i-l_i}(l_i'-l_i) \ge 0,
	\end{align*}
	where the final inequality comes from the fact that $u\ge 0$, $l'\ge l$, and $u> l$. On the other hand, if $a_i < 0$, it holds that
	\begin{gather*}
	g_i^* \ge a_i u_i' + b_i = \left(\frac{u_i}{u_i-l_i}-\frac{u_i'}{u_i'-l_i'}\right)u_i' + \left(\frac{u_i'l_i'}{u_i'-l_i'}-\frac{u_i l_i}{u_i-l_i}\right) \\
	= \frac{u_i}{u_i-l_i}(u_i'-l_i) - u_i' = \frac{u_i'-u_i}{u_i-l_i}l_i \ge 0,
	\end{gather*}
	where the final inequality comes from the fact that $u'\le u$, $l\le 0$, and $u>l$. Therefore,
	\begin{equation*}
	g^* = (g^*_1,g^*_2,\dots,g^*_{n_z})\ge 0,
	\end{equation*}
	which implies that $\Delta g(x) = g(x) - g'(x) \ge 0$ for all $x$ such that $l^{(j)}=l'\le Wx\le u'=u^{(j)}$. Hence, since $(x,z)\in\mathcal{N}_\textup{LP}^{(j)}$, it holds that $z\ge 0$, $z\ge Wx$, and
	\begin{equation*}
	z \le g'(x) \le g(x) = u\odot(Wx-l)\oslash(u-l).
	\end{equation*}
	Therefore, we have that $(x,z)\in\mathcal{N}_\textup{LP}$.
	
	Since $\mathcal{X}^{(j)}\subseteq \mathcal{X}$ (by definition) and $\mathcal{N}_\textup{LP}^{(j)}\subseteq \mathcal{N}_\textup{LP}$, it holds that the solution to the problem over the smaller feasible set gives a lower bound to the original solution: $\hl{\hat{\phi}_{\textup{LP}}^*}(\mathcal{X}^{(j)}) \le \hl{\hat{\phi}_{\textup{LP}}^*}(\mathcal{X})$. Finally, since $j$ was chosen arbitrarily, this implies the desired inequality \eqref{eq: improving_the_lp_relaxation_bound}.
\end{proof}
\hspace*{\fill}

\section{Proof of Proposition \ref{prop: np-hardness_of_optimal_partition}}
\label{sec: proof_of_np-hardness_optimal_partition}

\newcounter{para}
\newcommand\numpara[1]{\par\refstepcounter{para}\textbf{Step \thepara:\space#1.\space}}

\newcounter{subpara}
\newcommand\numsubpara[1]{\par\refstepcounter{subpara}\textbf{Step \thepara.\thesubpara:\space#1.\space}}

\begin{proof}
We prove the result by reducing an arbitrary instance of the Min-$\mathcal{K}$-Union problem to an instance of the optimal partitioning problem \eqref{eq: first_layer_partition_problem}. The proof is broken down into steps. In Step \ref{sec: np-hard_min-k-union_problem}, we introduce the Min-$\mathcal{K}$-Union problem. We then construct a specific neural network based on the parameters of the Min-$\mathcal{K}$-Union problem in Step \ref{sec: np-hard_neural_network_construction}. In Step \ref{sec: np-hard_dense_partition_lp}, we construct the solution to the partitioned LP relaxation for our neural network in the case that the partition is performed along all input coordinates. In Step \ref{sec: np-hard_sparse_partition_lp}, we construct the solution to the partitioned LP relaxation in the case that only a subset of the input coordinates are partitioned. Finally, in Step \ref{sec: np-hard_min-k-union_from_optimal_partition}, we show that the solution to the Min-$\mathcal{K}$-Union problem can be constructed from the solution to the optimal partitioning problem, i.e., by finding the best subset of coordinates to partition along in the fourth step. As a consequence, we show that optimal partitioning is NP-hard.

\numpara{Arbitrary Min-$\mathcal{K}$-Union Problem}
\label{sec: np-hard_min-k-union_problem}
Suppose that we are given an arbitrary instance of the Min-$\mathcal{K}$-Union problem, i.e., a finite number of finite sets $\mathcal{S}_1,\mathcal{S}_2,\dots,\mathcal{S}_n$ and a positive integer $\mathcal{K}\le n$. Since each set $\mathcal{S}_j$ is finite, the set $\bigcup_{j=1}^n\mathcal{S}_j$ is finite with cardinality $m\coloneqq \left|\bigcup_{i=j}^n\mathcal{S}_j\right|\in\mathbb{N}$. Therefore, there exists a bijection between the elements of $\bigcup_{j=1}^n\mathcal{S}_j$ and the set $\{1,2,\dots,m\}$. Hence, without loss of generality, we assume $\mathcal{S}_j\subseteq\mathbb{N}$ for all $j\in\{1,2,\dots,n\}$ such that $\bigcup_{j=1}^n\mathcal{S}_j=\{1,2,\dots,m\}$. In this Min-$\mathcal{K}$-Union problem, the objective is to find $\mathcal{K}$ sets $\mathcal{S}_{j_1},\mathcal{S}_{j_2},\dots,\mathcal{S}_{j_{\mathcal{K}}}$ among the collection of $n$ given sets such that $\left|\bigcup_{i=1}^{\mathcal{K}}\mathcal{S}_{j_i}\right|$ is minimized over all choices of $\mathcal{K}$ sets. In what follows, we show that the solution to this problem can be computed by solving a particular instance of the optimal partitioning problem \eqref{eq: first_layer_partition_problem}.

\numpara{Neural Network Construction}
\label{sec: np-hard_neural_network_construction}
Consider a $3$-layer ReLU network, where $x^{[0]},x^{[1]}\in\mathbb{R}^{n}$ and $x^{[2]},x^{[3]}\in\mathbb{R}^{m}$. Let the weight vector on the output be $c=\mathbf{1}_m$. Take the input uncertainty set to be $\mathcal{X}=[-1,1]^n$. Let $W^{[0]}=I_n$ and $W^{[2]}=I_m$. In addition, construct the weight matrix on the first layer to be $W^{[1]}\in\mathbb{R}^{m\times n}$ such that
\begin{equation*}
    W^{[1]}_{ij} = \begin{aligned}
    \begin{cases}
    1 & \text{if $i\in\mathcal{S}_j$}, \\
    0 & \text{otherwise}.
    \end{cases}
    \end{aligned}
\end{equation*}
We remark that, since all entries of $c=\mathbf{1}_m$, $W^{[0]}=I_n$, $W^{[1]}$, and $W^{[2]}=I_m$ are nonnegative, the optimal value of the unrelaxed certification problem \eqref{eq: unrelaxed_certification_k_layer} is $\hl{\phi^*}(\mathcal{X})=\mathbf{1}_m^\top W^{[1]} \mathbf{1}_n$.

To finish defining the network and its associated LP relaxations, we must specify the preactivation bounds at each layer. Since all weights of the neural network are nonnegative, the largest preactivation at each layer is attained when the input is $x^{[0]}=\mathbf{1}_n$, the element-wise maximum vector in $\mathcal{X}$. The preactivations corresponding to this input are $\hat{z}^{[1]}=\mathbf{1}_n$, $\hat{z}^{[2]}=W^{[1]}\mathbf{1}_n$, and $\hat{z}^{[3]}=W^{[1]}\mathbf{1}_n$. Therefore, setting
\begin{align*}
    u^{[1]}&=2\mathbf{1}_n, \\
    u^{[2]}&=\frac{3}{2}W^{[1]}\mathbf{1}_n, \\
    u^{[3]}&=\frac{5}{4}W^{[1]}\mathbf{1}_n+\frac{1}{8}\mathbf{1}_m,
\end{align*}
we obtain valid preactivation upper bounds. Similarly, taking
\begin{equation*}
    l^{[k]}=-u^{[k]}
\end{equation*}
for all $k\in\{1,2,3\}$ defines valid preactivation lower bounds.

\numpara{Densely Partitioned LP Relaxation}
\label{sec: np-hard_dense_partition_lp}
With the network parameters defined, we now consider the first variant of our partitioned LP relaxation. In particular, we consider the relaxation \eqref{eq: def_partitioned_objective} where all coordinates of the first layer are partitioned. We denote by $\bar{\phi}(\mathcal{X})$ the optimal objective value of this problem:
\begin{equation}
    \bar{\phi}(\mathcal{X}) = \max_{j'\in\{1,2,\dots,2^n\}} \hl{\hat{\phi}_{\textup{LP}}^*}(\mathcal{X}^{(j')}), \label{eq: baseline_in_proof_of_np-hard}
\end{equation}
where $\hl{\hat{\phi}_{\textup{LP}}^*}(\mathcal{X}^{(j')})$ denotes the optimal objective value of
\begin{equation}
    \begin{aligned}
        & \text{maximize} && c^\top x^{[3]} & \\
        & \text{subject to} && x^{[0]}\in\mathcal{X}^{(j')}, & \\
        &&& x^{[k+1]}\ge W^{[k]}x^{[k]}, & k\in\{0,1,2\}, \\
        &&& x^{[k+1]}\ge 0, & k\in\{0,1,2\}, \\
        &&& x^{[k+1]} \le u^{[k+1]}\odot (W^{[k]}x^{[k]}-l^{[k+1]})\oslash(u^{[k+1]}-l^{[k+1]}), & k\in\{0,1,2\}, \\
        &&& x^{[1]} = \relu(W^{[0]}x^{[0]}). &
    \end{aligned} \label{eq: baseline_in_proof_of_np-hard_subprob}
\end{equation}
This problem serves as a baseline; this is the tightest LP relaxation of the certification problem among all those with partitioning along the input coordinates. (Recall that the final equality constraint in \eqref{eq: baseline_in_proof_of_np-hard_subprob} is linear over the restricted feasible set $\mathcal{X}^{(j')}$.)

We will now show that an optimal solution $\bar{x}=(\bar{x}^{[0]},\bar{x}^{[1]},\bar{x}^{[2]},\bar{x}^{[3]})$ of the partitioned LP defined by \eqref{eq: baseline_in_proof_of_np-hard} and \eqref{eq: baseline_in_proof_of_np-hard_subprob} can be taken to satisfy
\begin{equation*}
    \bar{x}^{[3]} = \frac{5}{4}W^{[1]}\mathbf{1}_n + \frac{1}{16}\mathbf{1}_m. \label{eq: np-hard_densely_partitioned_solution}
\end{equation*}
To see this, note that since all weights of the network and optimization \eqref{eq: baseline_in_proof_of_np-hard} are nonnegative, the optimal activations $\bar{x}$ will be as large as possible in all coordinates and at all layers. Therefore, since the input is constrained to $\mathcal{X}=[-1,1]^n$, the optimal input for \eqref{eq: baseline_in_proof_of_np-hard} is $\bar{x}^{[0]}=\mathbf{1}_n$. Since the ReLU constraint in \eqref{eq: baseline_in_proof_of_np-hard_subprob} is exact, this implies that the optimal activation over this part at the first layer is
\begin{equation*}
\bar{x}^{[1]} = \relu(W^{[0]}\bar{x}^{[0]}) = \relu(\mathbf{1}_n) = \mathbf{1}_n.
\end{equation*}
Now, for the second layer, the activation attains its upper bound. Since $u^{[2]}=-l^{[2]}=\frac{3}{2}W^{[1]}\mathbf{1}_n$, this implies that
\begin{align*}
    \bar{x}^{[2]} &= u^{[2]}\odot(W^{[1]}\bar{x}^{[1]} - l^{[2]})\oslash(u^{[2]}-l^{[2]}) \\
    &= u^{[2]}\odot(W^{[1]}\bar{x}^{[1]} + u^{[2]})\oslash(2 u^{[2]}) \\
    &= \frac{1}{2}(W^{[1]}\bar{x}^{[1]} + u^{[2]}) \\
    &= \frac{1}{2}\left(W^{[1]}\mathbf{1}_n + \frac{3}{2}W^{[1]}\mathbf{1}_n\right) \\
    &= \frac{5}{4}W^{[1]}\mathbf{1}_n.
\end{align*}
Similarly, for the third layer, we find that the optimal activation attains its upper bound as well. Since $u^{[3]} = -l^{[3]} = \frac{5}{4}W^{[1]}\mathbf{1}_n+\frac{1}{8}\mathbf{1}_m$ and $W^{[2]}=I_m$ this gives that
\begin{align*}
    \bar{x}^{[3]} &= u^{[3]}\odot(W^{[2]}\bar{x}^{[2]} - l^{[3]})\oslash(u^{[3]}-l^{[3]}) \\
    &= u^{[3]}\odot(\bar{x}^{[2]} + u^{[3]})\oslash(2 u^{[3]}) \\
    &= \frac{1}{2}(\bar{x}^{[2]}+u^{[3]}) \\
    &= \frac{1}{2}\left( \frac{5}{4}W^{[1]}\mathbf{1}_n + \frac{5}{4}W^{[1]}\mathbf{1}_n + \frac{1}{8}\mathbf{1}_m \right) \\
    &= \frac{5}{4}W^{[1]}\mathbf{1}_n + \frac{1}{16}\mathbf{1}_m,
\end{align*}
as claimed in \eqref{eq: np-hard_densely_partitioned_solution}. It is easily verified that $\bar{x}$ as computed above satisfies all constraints of the problem \eqref{eq: baseline_in_proof_of_np-hard_subprob} over the part of the partition containing $\bar{x}^{[0]}=\mathbf{1}_n$.

\numpara{Sparsely Partitioned LP Relaxation}
\label{sec: np-hard_sparse_partition_lp}
We now introduce the second variant of the partitioned LP relaxation. In particular, let $\mathcal{J}_p\subseteq\{1,2,\dots,n\}$ be an index set such that $|\mathcal{J}_p|=n_p=n-\mathcal{K}$. Denote the complement of $\mathcal{J}_p$ by $\mathcal{J}_p^c = \{1,2,\dots,n\}\setminus\mathcal{J}_p$. We consider the partitioned LP defined in \eqref{eq: def_partitioned_objective}, which partitions along each coordinate in the index set $\mathcal{J}_p$. The optimal value of this problem is denoted by $\hl{\phi^*_{\mathcal{J}_p}}(\mathcal{X})$, and we denote an optimal solution by $\hat{x}=(\hat{x}^{[0]},\hat{x}^{[1]},\hat{x}^{[2]},\hat{x}^{[3]})$. We will compute $\hat{x}$ in three steps.

\numsubpara{Upper Bounding the Solution}
We start by upper bounding the final layer activation of the solution. In particular, we claim that the optimal solution $\hat{x}$ satisfies
\begin{equation}
    \hat{x}^{[3]} \le t \coloneqq u^{[3]} - \frac{1}{16}\mathbf{1}_{\mathcal{I}^c}. \label{eq: sparsely_partitioned_bound_og}
\end{equation}
where $\mathcal{I} = \bigcup_{j\in\mathcal{J}_p^c} S_j \subseteq \{1,2,\dots,m\}$ and $\mathcal{I}^c = \{1,2,\dots,m\}\setminus \mathcal{I}$. Since $\bar{x}^{[3]} = u^{[3]}-\frac{1}{16}\mathbf{1}_m$, the bound \eqref{eq: sparsely_partitioned_bound_og} is equivalent to
\begin{equation}
    \hat{x}^{[3]}_i \le t_i = \begin{aligned}
    \begin{cases}
    u_i^{[3]} & \text{if $i\in\mathcal{I}$}, \\
    \bar{x}^{[3]}_i & \text{if $i\in\mathcal{I}^c$},
    \end{cases}
    \end{aligned} \label{eq: sparsely_partitioned_bound}
\end{equation}
for all $i\in\{1,2,\dots,m\}$. We now prove the element-wise representation of the bound, \eqref{eq: sparsely_partitioned_bound}.

First, by the feasibility of $\hat{x}$ and the definitions of $u^{[3]},l^{[3]}$, it must hold for all $i\in\{1,2,\dots,m\}$ that
\begin{equation*}
    \hat{x}^{[3]}_i \le \frac{u^{[3]}_i}{u_i^{[3]} - l_i^{[3]}}(w_i^{[2]\top}\hat{x}^{[2]}-l_i^{[3]}) = \frac{1}{2}(w_i^{[2]\top}\hat{x}^{[2]}+u_i^{[3]}),
\end{equation*}
and also that
\begin{equation*}
    \hat{x}_i^{[3]} \ge w_i^{[2]\top}\hat{x}^{[2]}.
\end{equation*}
Combining these inequalities, we find that $\hat{x}^{[3]}_i \le \frac{1}{2}(\hat{x}^{[3]}_i+u_i^{[3]})$, or, equivalently, that
\begin{equation*}
    \hat{x}^{[3]}_i \le u_i^{[3]}.
\end{equation*}
This bound holds for all $i\in\{1,2,\dots,m\}$, and therefore it also holds for $i\in\mathcal{I}$. This proves the first case in the bound \eqref{eq: sparsely_partitioned_bound}.

We now prove the second case of the claimed upper bound. For this case, suppose $i\notin\mathcal{I}$. Then $i\notin S_j$ for all $j\in\mathcal{J}_p^c$, which implies that
\begin{equation*}
    W_{ij}^{[1]} = 0 ~ \text{for all $j\in\mathcal{J}_p^c$},
\end{equation*}
by the definition of $W^{[1]}$. Therefore,
\begin{equation*}
    w_i^{[1]\top}\hat{x}^{[1]} = \sum_{j=1}^n W_{ij}^{[1]}\hat{x}_j^{[1]} = \sum_{j\in\mathcal{J}_p^c}W_{ij}^{[1]}\hat{x}_j^{[1]} + \sum_{j\in\mathcal{J}_p}W_{ij}^{[1]}\hat{x}_j^{[1]} = \sum_{j\in\mathcal{J}_p}W_{ij}^{[1]}\hat{x}_j^{[1]}.
\end{equation*}
Now, note that for $j\in\mathcal{J}_p$, the $j^\text{th}$ coordinate of the input is being partitioned, and therefore the optimal solution must satisfy
\begin{equation*}
    \hat{x}^{[1]}_j = \relu(w_j^{[0]\top}\hat{x}^{[0]}) = \relu(e_j^\top \hat{x}^{[0]}) = \relu(\hat{x}_j^{[0]}) \le 1,
\end{equation*}
since $\hat{x}_j^{[0]}\in[-1,1]$. Therefore,
\begin{align*}
    w_i^{[1]\top}\hat{x}^{[1]} \le \sum_{j\in\mathcal{J}_p^c} W_{ij}^{[1]} \le \sum_{j=1}^n W_{ij}^{[1]} = w_i^{[1]\top}\mathbf{1}_n.
\end{align*}
It follows from the feasibility of $\hat{x}$ and the definitions of $u^{[2]},l^{[2]}$ that
\begin{gather*}
    \hat{x}_i^{[2]} \le \frac{u_i^{[2]}}{u_i^{[2]}-l_i^{[2]}}(w_i^{[1]\top}\hat{x}^{[1]} - l_i^{[2]}) = \frac{1}{2}(w_i^{[1]\top}\hat{x}^{[1]}+u_i^{[2]}) \\
    \le \frac{1}{2}\left(w_i^{[1]\top}\mathbf{1}_n + \frac{3}{2}w_i^{[1]\top}\mathbf{1}_n\right) = \frac{5}{4}w_i^{[1]\top}\mathbf{1}_n = \bar{x}^{[2]}_i,
\end{gather*}
where $\bar{x}$ is the solution computed for the densely partitioned LP relaxation in Step \ref{sec: np-hard_dense_partition_lp}. Therefore, we conclude that for all $i\notin\mathcal{I}$, it holds that
\begin{equation*}
    \hat{x}^{[3]}_i \le \frac{u_i^{[3]}}{u_i^{[3]}-l_i^{[3]}}(w_i^{[2]\top}\hat{x}^{[2]} - l_i^{[3]}) \le \frac{u_i^{[3]}}{u_i^{[3]}-l_i^{[3]}}(w_i^{[2]\top}\bar{x}^{[2]} - l_i^{[3]}) = \bar{x}^{[3]}_i,
\end{equation*}
by our previous construction of $\bar{x}^{[3]}$. Thus, we have proven the second case in \eqref{eq: sparsely_partitioned_bound} holds. Hence, the claimed bound \eqref{eq: sparsely_partitioned_bound_og} holds.

\numsubpara{Feasibility of Upper Bound}
Let us define $x=(x^{[0]},x^{[1]},x^{[2]},x^{[3]})$, a point in $\mathbb{R}^n\times \mathbb{R}^n\times \mathbb{R}^m\times \mathbb{R}^m$, by
\begin{equation*}
    x^{[0]} = \mathbf{1}_n, \qquad x^{[1]} = \mathbf{1}_{\mathcal{J}_p} + \frac{5}{4}\mathbf{1}_{\mathcal{J}_p^c}, \qquad x^{[2]} = u^{[3]} - \frac{1}{8}\mathbf{1}_{\mathcal{I}^c}, \qquad x^{[3]} = u^{[3]} - \frac{1}{16}\mathbf{1}_{\mathcal{I}^c}.
\end{equation*}
Note that $x^{[3]}$ equals the upper bound $t=(t_1,t_2,\dots,t_m)$. We now show that $x$ is feasible for the partitioned LP defined by \eqref{eq: def_partitioned_objective_subprob} and \eqref{eq: def_partitioned_objective}.

First, the input uncertainty constraint is satisfied, since $x^{[0]}=\mathbf{1}_n\in \mathcal{X}^{(j')} \subseteq \mathcal{X}$ for some part $\mathcal{X}^{(j')}$. Next, the relaxed ReLU constraints at the first layer are satisfied, since
\begin{gather*}
    x^{[1]} = \mathbf{1}_{\mathcal{J}_p} + \frac{5}{4}\mathbf{1}_{\mathcal{J}_p^c} \ge 0, \tag*{(Layer 1 lower bound.)} \\
    x^{[1]} - W^{[0]}x^{[0]} = \frac{1}{4}\mathbf{1}_{\mathcal{J}_p^c} \ge 0, \tag*{(Layer 1 lower bound.)} \\
    x^{[1]} - u^{[1]}\odot(W^{[0]}x^{[0]} - l^{[1]})\oslash(u^{[1]}-l^{[1]}) = -\frac{1}{2}\mathbf{1}_{\mathcal{J}_p} - \frac{1}{4}\mathbf{1}_{\mathcal{J}_p^c} \le 0. \tag*{(Layer 1 upper bound.)}
\end{gather*}
The relaxed ReLU constraints are also satisfied in the second layer, since
\begin{gather*}
    x^{[2]} = \frac{5}{4}W^{[1]}\mathbf{1}_n + \frac{1}{8}\mathbf{1}_{\mathcal{I}} \ge 0, \tag*{(Layer 2 lower bound.)} \\
    x^{[2]} - W^{[1]}x^{[1]} = \frac{1}{4}W^{[1]}\mathbf{1}_{\mathcal{J}_p} + \frac{1}{8}\mathbf{1}_{\mathcal{I}} \ge 0, \tag*{(Layer 2 lower bound.)} \\
    x^{[2]} - u^{[2]}\odot(W^{[1]}x^{[1]} - l^{[2]})\oslash(u^{[2]}-l^{[2]}) = \frac{1}{8}(\mathbf{1}_{\mathcal{I}} - W^{[1]}\mathbf{1}_{\mathcal{J}_p^c}) \le 0. \tag*{(Layer 2 upper bound.)}
\end{gather*}
The final inequality above follows from the fact that either $(\mathbf{1}_{\mathcal{I}})_i=0$ or $(\mathbf{1}_{\mathcal{I}})_i=1$. For coordinates $i$ such that $(\mathbf{1}_{\mathcal{I}})_i=0$, the inequality obviously holds. For coordinates $i$ such that $(\mathbf{1}_{\mathcal{I}})_i=1$, we know that $i\in\mathcal{I}$, implying that $i\in \mathcal{S}_j$ for some $j\in\mathcal{J}_p^c$. This in turn implies that $W_{ij}^{[1]}=1$ for some $j\in\mathcal{J}_p^c$, and therefore $w_i^{[1]\top}\mathbf{1}_{\mathcal{J}_p^c} = \sum_{j\in\mathcal{J}_p^c}W_{ij}^{[1]} \ge 1 = (\mathbf{1}_{\mathcal{I}})_i$.

Continuing to check feasibility of $x$, the relaxed ReLU constraints in the final layer are also satisfied:
\begin{gather*}
    x^{[3]} = \frac{5}{4}W^{[1]}\mathbf{1}_n + \frac{1}{16}\mathbf{1}_m + \frac{1}{16}\mathbf{1}_{\mathcal{I}} \ge 0, \tag*{(Layer 3 lower bound.)} \\
    x^{[3]} - W^{[2]}x^{[2]} = \frac{1}{16}\mathbf{1}_{\mathcal{I}^c} \ge 0, \tag*{(Layer 3 lower bound.)} \\
    x^{[3]} - u^{[3]}\odot(W^{[2]}x^{[2]} - l^{[3]})\oslash(u^{[3]}-l^{[3]}) = 0 \le 0. \tag*{(Layer 3 upper bound.)}
\end{gather*}
Hence, the relaxed ReLU constraints are satisfied at all layers. The only remaining constraints to verify are the exact ReLU constraints for the partitioned input indices $\mathcal{J}_p$. Indeed, for all $j\in\mathcal{J}_p$, we have that
\begin{align*}
    x^{[1]}_j - \relu(W^{[0]}x^{[0]})_j &= (\mathbf{1}_{\mathcal{J}_p})_j + \frac{5}{4}(\mathbf{1}_{\mathcal{J}_p^c})_j - \relu(\mathbf{1}_n)_j = 1+0-1 = 0.
\end{align*}
Hence, the ReLU equality constraint is satisfied for all input coordinates in $\mathcal{J}_p$. Therefore, our proposed point $x$ is feasible for \eqref{eq: def_partitioned_objective}.

\numsubpara{Solution to Sparsely Partitioned LP}
As shown in the previous step, the proposed point $x=(x^{[0]},x^{[1]},x^{[2]},x^{[3]})$ is feasible for the partitioned LP defined by \eqref{eq: def_partitioned_objective_subprob} and \eqref{eq: def_partitioned_objective}. Recall from the upper bound \eqref{eq: sparsely_partitioned_bound_og} that our solution $\hat{x}=(\hat{x}^{[0]},\hat{x}^{[1]},\hat{x}^{[2]},\hat{x}^{[3]})$ satisfies $\hat{x}^{[3]}\le t$. The objective value of the feasible point $x$ gives that
\begin{equation*}
    c^\top x^{[3]} = \sum_{i=1}^m x^{[3]} = \sum_{i=1}^m t_i \ge \sum_{i=1}^m \hat{x}^{[3]}_i = \hl{\phi^*_{\mathcal{J}_p}}(\mathcal{X}).
\end{equation*}
Since $\hl{\phi^*_{\mathcal{J}_p}}(\mathcal{X})$ is the maximum value of the objective for all feasible points, it must be that $c^\top x^{[3]} = \hl{\phi^*_{\mathcal{J}_p}}(\mathcal{X})$. Hence, the point $x$ is an optimal solution to \eqref{eq: def_partitioned_objective}. Therefore, we can write the final activation of our optimal solution $\hat{x}$ to \eqref{eq: def_partitioned_objective} as
\begin{equation}
    \hat{x}^{[3]} = x^{[3]} = t = u^{[3]} - \frac{1}{16}\mathbf{1}_{\mathcal{I}^c}. \label{eq: sparsely_partitioned_solution}
\end{equation}

\numpara{Min-$\mathcal{K}$-Union from Optimal Partition}
\label{sec: np-hard_min-k-union_from_optimal_partition}
With the solutions constructed in Steps \ref{sec: np-hard_dense_partition_lp} and \ref{sec: np-hard_sparse_partition_lp}, we compute the difference in the objective values between the two partitioned LP relaxations:
\begin{gather*}
    \hl{\phi^*_{\mathcal{J}_p}}(\mathcal{X}) - \bar{\phi}(\mathcal{X}) = c^\top \hat{x}^{[3]} - c^\top \bar{x}^{[3]} = c^\top \left( u^{[3]} - \frac{1}{16}\mathbf{1}_{\mathcal{I}^c} - \frac{5}{4}W^{[1]}\mathbf{1}_n - \frac{1}{16}\mathbf{1}_m \right) \\
    = c^\top \left( \frac{5}{4}W^{[1]}\mathbf{1}_n + \frac{1}{8}\mathbf{1}_m - \frac{1}{16}\mathbf{1}_{\mathcal{I}^c} - \frac{5}{4}W^{[1]}\mathbf{1}_n - \frac{1}{16}\mathbf{1}_m \right) = \frac{1}{16} c^\top (\mathbf{1}_m - \mathbf{1}_{\mathcal{I}^c}) \\
    = \frac{1}{16}c^\top\mathbf{1}_{\mathcal{I}} = \frac{1}{16}\sum_{i\in\mathcal{I}}1 = \frac{1}{16}\left| \mathcal{I} \right| = \frac{1}{16}\left| \bigcup_{j\in\mathcal{J}_p^c}S_j\right|.
\end{gather*}
Therefore,
\begin{equation}
    \left| \bigcup_{j\in\mathcal{J}_p^c}\mathcal{S}_j \right| = 16\left( \hl{\phi^*_{\mathcal{J}_p}}(\mathcal{X}) - \bar{\phi}(\mathcal{X}) \right), \label{eq: equivalence_between_optimal_partition_and_min-k-u}
\end{equation}
which holds for all partition index sets $\mathcal{J}_p\subseteq\{1,2,\dots,n\}$ such that $|\mathcal{J}_p|=n_p=n-\mathcal{K}$.

Now, let $\mathcal{J}_p^*$ be an optimal partition, i.e., a solution to \eqref{eq: first_layer_partition_problem} with our specified neural network parameters. Then, by \eqref{eq: equivalence_between_optimal_partition_and_min-k-u}, we have $\left|\bigcup_{j\in(\mathcal{J}^*_p)^c} \mathcal{S}_j\right| = 16\left( \hl{\phi^*_{\mathcal{J}_p^*}}(\mathcal{X}) - \bar{\phi}(\mathcal{X}) \right) \le 16\left( \hl{\phi^*_{\mathcal{J}_p}}(\mathcal{X}) - \bar{\phi}(\mathcal{X}) \right) = \left| \bigcup_{j\in\mathcal{J}_p^c}\mathcal{S}_j \right|$ for all $\mathcal{J}_p$ with $|\mathcal{J}_p| = n_p$. Since this holds for all $\mathcal{J}_p^c$ with $|\mathcal{J}_p^c|=n-n_p=\mathcal{K}$, this shows that the set $(\mathcal{J}_p^*)^c$ is an optimal solution to the Min-$\mathcal{K}$-Union problem specified at the beginning of the proof. Now, suppose the optimal partitioning problem in \eqref{eq: first_layer_partition_problem} could be solved for $\mathcal{J}^*_p$ in polynomial time. Then the optimal solution $(\mathcal{J}^*_p)^c$ to the Min-$\mathcal{K}$-Union problem is also computable in polynomial time. Since this holds for an arbitrary instance of the Min-$\mathcal{K}$-Union problem, this implies that the Min-$\mathcal{K}$-Union problem is polynomially solvable in general, which is a contradiction. Therefore, the problem \eqref{eq: first_layer_partition_problem} is NP-hard in general.
\end{proof}
\hspace*{\fill}

\section{Proof of Proposition \ref{prop: improving_the_sdp_relaxation_bound}}
\label{sec: proof_of_partition_improvement_sdp}

\begin{proof}
	Let $j\in\{1,2,\dots,p\}$. From the definition of the partition, it holds that $\mathcal{X}^{(j)}\subseteq \mathcal{X}$. What remains to be shown is that $\mathcal{N}_{\textup{SDP}}^{(j)} \subseteq \mathcal{N}_\textup{SDP}$.
	
	Let $P\in\mathcal{N}_{\textup{SDP}}^{(j)}$. Define $u' = u^{(j)}$ and $l' = l^{(j)}$. Since $P \in \mathcal{N}_{\textup{SDP}}^{(j)}$, it follows that
	\begin{align*}
	P_z\ge{}& 0, \\
	P_z\ge{}& WP_x, \\
	\diag(P_{zz}) ={}& \diag(WP_{xz}), \\
	\diag(P_{xx}) \leq{}& (l'+u') \odot P_x - l' \odot u', \\
	P_1 ={}& 1, \\
	P \succeq{}& 0.
	\end{align*}
	To show that $P\in\mathcal{N}_\textup{SDP}$, we should show that the above expressions imply that $\diag(P_{xx})\le (l+u)\odot P_x - l\odot u$. To do so, define $\Delta l_i\ge 0$ and $\Delta u_i \ge 0$ such that $l_i' = l_i + \Delta l_i$ and $u_i' = u_i - \Delta u_i$ for all $i \in \{1,2,\dots,n_x\}$. Then we find that 
	\begin{align*}
	((l'+u') \odot P_x - l' \odot u')_i ={}& (l_i'+u_i')(P_x)_i-l_i' u_i' \\
	={}& (l_i+u_i)(P_x)_i - l_i u_i + (\Delta l_i - \Delta u_i)(P_x)_i  \\
	   & - ( \Delta l_i u_i - \Delta u_i l_i - \Delta u_i \Delta l_i) \\
	={}& ((l+u) \odot P_x - l \odot u)_i + (\Delta l_i - \Delta u_i)(P_x)_i \\
	   & - ( \Delta l_i u_i - \Delta u_i l_i - \Delta u_i \Delta l_i) \\
	={}& ((l+u) \odot P_x - l \odot u)_i + \Delta_i,
	\end{align*}
	where $\Delta_i \coloneqq (\Delta l_i - \Delta u_i)(P_x)_i - ( \Delta l_i u_i - \Delta u_i l_i - \Delta u_i \Delta l_i)$. Therefore, it suffices to prove that $\Delta_i \leq 0$ for all $i$. Since $-l_i u_i \ge -l_i' u_i'$ by definition, it holds that $ \Delta l_i u_i - \Delta u_i l_i - \Delta u_i \Delta l_i \ge 0$. Thus, when $(\Delta l_i - \Delta u_i)(P_x)_i \le 0$, it holds that $\Delta_i \leq 0$, as desired. On the other hand, suppose that $(\Delta l_i - \Delta u_i)(P_x)_i \ge 0$. Then we find two cases:
	\begin{enumerate}
		\item $(\Delta l_i - \Delta u_i) \ge 0$ and $(P_x)_i \ge 0$. In this case, the maximum value of $(\Delta l_i - \Delta u_i)(P_x)_i$ is $(\Delta l_i - \Delta u_i)u'_i$. Therefore, the maximum value of $\Delta_i$ is
		\begin{align*}
		\Delta_i ={}& (\Delta l_i - \Delta u_i)u'_i - ( \Delta l_i u_i - \Delta u_i l_i - \Delta u_i \Delta l_i) \\
		={}& \Delta l_i(u'_i-u_i)-\Delta u_i u'_i + \Delta u_i l_i + \Delta u_i \Delta l_i \\
		={}& \Delta l_i (-\Delta u_i) + \Delta u_i \Delta l_i -\Delta u_i u'_i + \Delta u_i l_i \\
		={}& -\Delta u_i u'_i + \Delta u_i l_i.
		\end{align*}
		Both of the two final terms are nonpositive, and therefore $\Delta_i \le 0$.
		
		\item $(\Delta l_i - \Delta u_i) \le 0$ and $(P_x)_i \le 0$. In this case, the maximum value of $(\Delta l_i - \Delta u_i)(P_x)_i$ is $(\Delta l_i - \Delta u_i)l_i'$. Therefore, the maximum value of  $\Delta_i$ is
		\begin{align*}
		\Delta_i ={}& (\Delta l_i - \Delta u_i)l'_i - ( \Delta l_i u_i - \Delta u_i l_i - \Delta u_i \Delta l_i) \\
		={}& - \Delta u_i \Delta l_i + \Delta u_i \Delta l_i +\Delta l_i l'_i - \Delta l_i u_i \\
		={}& \Delta l_i l'_i - \Delta l_i u_i.
		\end{align*}
		Both of the two final terms are nonpositive, and therefore $\Delta_i \le 0$.
	\end{enumerate}
	Hence, we find that $(l'+u') \odot P_x - l' \odot u' \leq (l+u) \odot P_x - l \odot u$ for all $P\in\mathcal{N}_\textup{SDP}^{(j)}$, proving that $P\in\mathcal{N}_\textup{SDP}$, and therefore $\mathcal{N}_{\textup{SDP}}^{(j)} \subseteq \mathcal{N}_\textup{SDP}$.
	
	Since $\mathcal{X}^{(j)}\subseteq \mathcal{X}$ and $\mathcal{N}_{\textup{SDP}}^{(j)} \subseteq \mathcal{N}_\textup{SDP}$, it holds that the solution to the problem over the smaller feasible set lower bounds the original solution: $\hl{\hat{\phi}^*_\textup{SDP}}(\mathcal{X}^{(j)}) \le \hl{\hat{\phi}^*_\textup{SDP}}(\mathcal{X})$. Finally, since $j$ was chosen arbitrarily, this implies the desired inequality \eqref{eq: improving_the_sdp_relaxation_bound}.
\end{proof}
\hspace*{\fill}

\section{Proof of Lemma \ref{lem: bound_elements_psd_matrices}}
\label{sec: proof_of_bound_elements_psd_matrices}

\begin{proof}
    Let $i,j\in\{1,2,\dots,n\}$. Since $P$ is positive semidefinite, the $2^\text{nd}$-order principal minor $P_{ii}P_{jj} - P_{ij}^2$ is nonnegative, and therefore 
    \begin{equation}
        |P_{ij}| \le \sqrt{P_{ii}P_{jj}}. \label{eq: lemma_5_proof_intermediate}
    \end{equation}
    Furthermore, by the basic inequality that $2ab \le a^2+b^2$ for all $a,b\in\mathbb{R}$, we have that $\sqrt{P_{ii}P_{jj}}\le \frac{1}{2}(P_{ii} + P_{jj})$. Substituting this inequality into \eqref{eq: lemma_5_proof_intermediate} gives the desired bound.
\end{proof}
\hspace*{\fill}

\vskip 0.2in
\bibliography{references}


\end{document}